%% file: 0_main.tex
\pdfoutput=1
\documentclass{article}

\PassOptionsToPackage{numbers, compress}{natbib}

\usepackage[main, final]{neurips_2026}

\usepackage{annotate-equations}
\usepackage[utf8]{inputenc} 
\usepackage[T1]{fontenc}    
\usepackage{hyperref}       
\usepackage{url}            
\usepackage{booktabs}       
\usepackage{caption}        
\usepackage{amsfonts}       
\usepackage{nicefrac}       
\usepackage{microtype}      
\usepackage{xcolor}         
\usepackage{enumitem}
\usepackage{amsthm}
\usepackage{amsmath}
\usepackage{bm}
\usepackage{amssymb}
\usepackage{mathtools}
\usepackage[ruled, linesnumbered]{algorithm2e}
\usepackage{makecell}
\usepackage{multirow}
\usepackage{subfigure} 
\usepackage{booktabs}
\usepackage{wrapfig}
\usepackage{extarrows}
\usepackage{colortbl}
\usepackage{tikz}

\usepackage{CJKutf8}

\newcommand{\ie}{\textit{i.e.}}
\newcommand{\eg}{\textit{e.g.}}

\definecolor{mygray}{RGB}{230,230,230}
\definecolor{myblue}{RGB}{225,244,252}
\definecolor{lightgreen}{RGB}{76,175,80}

\theoremstyle{definition}

\newtheorem{theorem}{Theorem}[section]
\newtheorem{lemma}[theorem]{Lemma}

\newtheorem{definition}[theorem]{Definition}

\newtheorem{proposition}[theorem]{Proposition}

\newcommand{\ourmethod}{{EMC}\xspace}

\usepackage{etoc}
\etocdepthtag.toc{mtchapter}
\etocsettagdepth{mtchapter}{subsection}
\etocsettagdepth{mtappendix}{none}

\usepackage{pifont}

\definecolor{myblue}{RGB}{225,244,252}

\title{Efficient Memory Crystallization for Graph Learning under Non-Stationary Distribution Shifts}

\author{%
  {Yue Hou$^{1,2}$, Ruomei Liu$^{1}$, Yingke Su$^{2}$, Junran Wu$^{1}$\thanks{Corresponding author}, Ke Xu$^{1}$,}\\
  $^{1}$State Key Laboratory of Complex \& Critical Software Environment, Beihang University\\
  $^{2}$Shen Yuan Honors College, Beihang University\\
  \texttt{\{hou\_yue, rmliu, suyingke, wu\_junran, kexu\}@buaa.edu.cn}
}

\begin{document}
\begin{CJK}{UTF8}{gbsn}

\maketitle

\begin{abstract}
Deep graph learning models deployed in real-world systems often need to cope with non-stationary environments, where the underlying graph distribution drifts continually over time. Prevailing solutions rely on training auxiliary generative modules to synthesize memory graphs for cross-domain adaptation, which incurs substantial computational overhead and scales poorly under prolonged distribution shifts. We argue that a more economical path exists: rather than generating memory, one can crystallize it. To this end, we propose Efficient Memory Crystallization (EMC), a training-free test-time framework that distills each incoming graph domain into a compact, semantically faithful memory through a closed-form solution to a memory-oriented distribution-matching objective, thereby eliminating redundant domain information under continual covariate shifts. To preserve both generalizability and adaptability as the model traverses a long sequence of target domains, EMC further models inter-domain dependencies through state-evolving memories and admits a theoretically grounded, tighter generalization error bound than direct adaptation. Extensive experiments demonstrate the superior performance of EMC over state-of-the-art baselines on graphs under non-stationary distribution shifts, while reducing average runtime by 87.4\% and GPU memory consumption by 92.4\% relative to the recent competitor, making continual graph adaptation practical at scale.
\end{abstract}

\input{chapter/1-introduction}
\input{chapter/3-method}
\input{chapter/4-experiment}
\input{chapter/2-related}
\input{chapter/conclusion}

\section*{Acknowledgements}
\vspace{-2.0mm}
This work has been supported by CCSE project (CCSE-2024ZX-09), by NSFC (Grant No. U25B2070), and by the Academic Excellence Foundation of BUAA for PhD Students.

{\small{
\bibliographystyle{plain}
\bibliography{reference.bib}
}}


\newpage
\appendix



\etocdepthtag.toc{mtappendix}
	\etocsettagdepth{mtchapter}{none}
	\etocsettagdepth{mtappendix}{subsection}
	\renewcommand{\contentsname}{Technical Appendices and Supplementary Material}
	\tableofcontents

\input{appendix/appendix}

\clearpage

\newpage
\section*{NeurIPS Paper Checklist}

\begin{enumerate}

\item {\bf Claims}
    \item[] Question: Do the main claims made in the abstract and introduction accurately reflect the paper's contributions and scope?
    \item[] Answer: \answerYes{} 
    \item[] Justification: Abstract and introduction accurately reflect the paper's contributions and scope.

\item {\bf Limitations}
    \item[] Question: Does the paper discuss the limitations of the work performed by the authors?
    \item[] Answer: \answerYes{} 
    \item[] Justification: The limitations regarding the future research direction of the work are discussed in the conclusion section.

\item {\bf Theory assumptions and proofs}
    \item[] Question: For each theoretical result, does the paper provide the full set of assumptions and a complete (and correct) proof?
    \item[] Answer: \answerYes{} 
    \item[] Justification: The full set of assumptions and proofs is provided in the paper.

    \item {\bf Experimental result reproducibility}
    \item[] Question: Does the paper fully disclose all the information needed to reproduce the main experimental results of the paper to the extent that it affects the main claims and/or conclusions of the paper (regardless of whether the code and data are provided or not)?
    \item[] Answer: \answerYes{} 
    \item[] Justification: Our paper fully discloses all the information needed to reproduce the main experimental results.

\item {\bf Open access to data and code}
    \item[] Question: Does the paper provide open access to the data and code, with sufficient instructions to faithfully reproduce the main experimental results, as described in supplemental material?
    \item[] Answer: \answerYes{} 
    \item[] Justification: The adopted datasets are public datasets. The code of \ourmethod is available.

\item {\bf Experimental setting/details}
    \item[] Question: Does the paper specify all the training and test details (e.g., data splits, hyperparameters, how they were chosen, type of optimizer, etc.) necessary to understand the results?
    \item[] Answer: \answerYes{} 
    \item[] Justification: Our paper specifies all the training and testing details.

\item {\bf Experiment statistical significance}
    \item[] Question: Does the paper report error bars suitably and correctly defined or other appropriate information about the statistical significance of the experiments?
    \item[] Answer: \answerYes{} 
    \item[] Justification: Our paper shows average results over 5 runs with different random seeds and reports the variance for our metrics.

\item {\bf Experiments compute resources}
    \item[] Question: For each experiment, does the paper provide sufficient information on the computer resources (type of compute workers, memory, time of execution) needed to reproduce the experiments?
    \item[] Answer: \answerYes{} 
    \item[] Justification: Our paper provides sufficient information on the computer resources.
    
\item {\bf Code of ethics}
    \item[] Question: Does the research conducted in the paper conform, in every respect, with the NeurIPS Code of Ethics \url{https://neurips.cc/public/EthicsGuidelines}?
    \item[] Answer: \answerYes{} 
    \item[] Justification: Our research complies with the NeurIPS Code of Ethics.

\item {\bf Broader impacts}
    \item[] Question: Does the paper discuss both potential positive societal impacts and negative societal impacts of the work performed?
    \item[] Answer: \answerYes{} 
    \item[] Justification: The societal impacts of the work are discussed in Appendix \ref{Impacts}.
    
\item {\bf Safeguards}
    \item[] Question: Does the paper describe safeguards that have been put in place for responsible release of data or models that have a high risk for misuse (e.g., pretrained language models, image generators, or scraped datasets)?
    \item[] Answer: \answerNA{} 
    \item[] Justification: The paper poses no such risk.

\item {\bf Licenses for existing assets}
    \item[] Question: Are the creators or original owners of assets (e.g., code, data, models), used in the paper, properly credited and are the license and terms of use explicitly mentioned and properly respected?
    \item[] Answer: \answerYes{} 
    \item[] Justification: The adopted datasets and compared baselines are properly credited.

\item {\bf New assets}
    \item[] Question: Are new assets introduced in the paper well documented and is the documentation provided alongside the assets?
    \item[] Answer: \answerNA{} 
    \item[] Justification: Our paper does not release new assets.

\item {\bf Crowdsourcing and research with human subjects}
    \item[] Question: For crowdsourcing experiments and research with human subjects, does the paper include the full text of instructions given to participants and screenshots, if applicable, as well as details about compensation (if any)? 
    \item[] Answer: \answerNA{} 
    \item[] Justification: Our paper does not involve crowdsourcing nor research with human subjects.

\item {\bf Institutional review board (IRB) approvals or equivalent for research with human subjects}
    \item[] Question: Does the paper describe potential risks incurred by study participants, whether such risks were disclosed to the subjects, and whether Institutional Review Board (IRB) approvals (or an equivalent approval/review based on the requirements of your country or institution) were obtained?
    \item[] Answer: \answerNA{} 
    \item[] Justification: Our paper does not involve crowdsourcing nor research with human subjects.

\item {\bf Declaration of LLM usage}
    \item[] Question: Does the paper describe the usage of LLMs if it is an important, original, or non-standard component of the core methods in this research? Note that if the LLM is used only for writing, editing, or formatting purposes and does not impact the core methodology, scientific rigorousness, or originality of the research, declaration is not required.
    \item[] Answer: \answerNA{} 
    \item[] Justification: The core method development in this research does not involve LLMs as any important, original, or non-standard components.
 
\end{enumerate}

\end{CJK}
\end{document}

%% file: chapter/1-introduction.tex
\section{Introduction}

Graph-structured data is ubiquitous in real-world applications such as social network analysis, recommender systems, and epidemiological modeling.
To effectively learn from such data, Graph Neural Networks (GNNs)~\cite{gcn_kipf2017semi,gin_xu2019how} have emerged as powerful models that integrate node features and structural information through neighborhood aggregation.
While GNNs have demonstrated impressive performance under the assumption that training and testing data follow the same distribution, this assumption is often violated in practice due to distribution shift, that is, a mismatch between the training (source) and testing (target) distributions which poses a major challenge to the generalization ability of deep models.
To bridge this gap, Graph Domain Adaptation (GDA) methods aim to transfer knowledge from the source domain to the target domain by mitigating the distributional discrepancy.

In real-world scenarios, the source data is often inaccessible during test time due to privacy or business constraints, making the adaptation problem more challenging.
Moreover, the GNN model frequently operates in non-stationary environments, where the target graph distribution drifts continually over time.
Transferring knowledge from the source domain to entirely unseen target domains can introduce inductive bias, which exacerbates catastrophic forgetting and degrades the model's generalization ability to future domains, especially when distribution shifts arrive in a continual stream rather than a single jump, as illustrated in Figure~\ref{fig:toy}.
This necessitates addressing the problem of \emph{graph learning under non-stationary distribution shifts at test time}.

\begin{wrapfigure}{r}{0.5\textwidth}
    \centering
    \vspace{-1em}
    \includegraphics[width=0.49\textwidth]{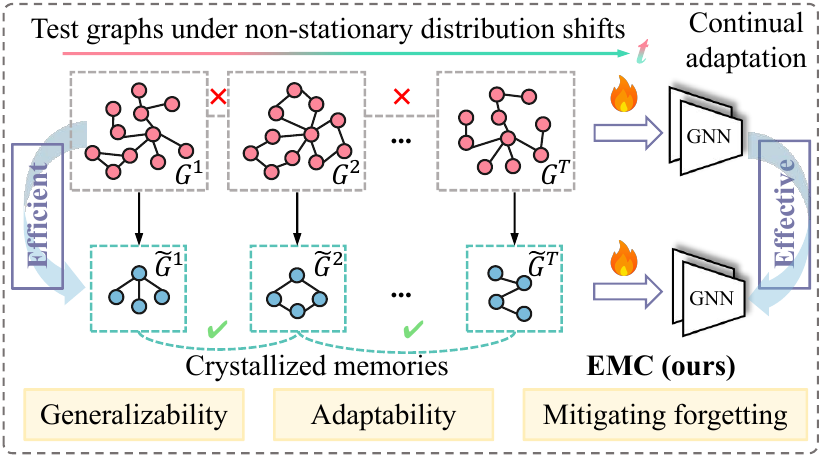}
    \vspace{-0.2cm}
    \caption{A general paradigm of test-time adaptation under non-stationary distribution shifts. Compared with the first row, which directly adapts the model on the test data without modeling \textbf{inter-domain relations} or \textbf{intra-domain redundancy}, the second row illustrates our \ourmethod, which crystallizes each domain into a compact memory and explicitly considers these factors to maintain both generalizability and adaptability, thereby ensuring efficiency and effectiveness.}
    \label{fig:toy}
    \vspace{-0.8em}
\end{wrapfigure}

Recently, there has been a growing research trend focusing on single-step domain adaptation, typically employing techniques such as Maximum Mean Discrepancy~\cite{dziugaite2015training} and adversarial learning~\cite{dan2024tfgda,qiao2023semi,zhang2018collaborative}.
However, these approaches inevitably fall short when the model is exposed to distribution shifts that arrive continually across time or domains.
As graph data evolves, models encounter new domains that require ongoing adaptation, while simultaneously losing the ability to perform well on previously seen graphs, which leads to catastrophic forgetting.
To mitigate catastrophic forgetting, {Continual Learning}~\cite{wang2024comprehensive,CoTTA,HPNs,zhang2022sparsified,zhang2023ricci} has emerged as a promising direction.
These methods typically adopt {replay mechanisms}, where the model periodically revisits selected or generated historical data to reinforce prior knowledge.
Such strategies help balance training cost with long-term knowledge retention.
However, most existing approaches rely on labeled data to guide memory selection and replay, making them ill-suited for the test-time setting where labels are unavailable.
Although GCAL~\cite{qiao2025gcal} attempts to improve cross-domain adaptation by training an auxiliary generative module to synthesize memory graphs, its dependence on a separately trained generator introduces substantial computational overhead and limits its practicality under prolonged drift.
Furthermore, as shown in Figure~\ref{fig:toy}, current methods under test-time settings with continual drift often ignore the inter-domain relationships and the redundancy within each domain that is irrelevant to the underlying semantic information.
We summarize that a desirable framework for adapting graph models under non-stationary distribution shifts at test time should pursue three core objectives:
ensuring the model's \textbf{generalizability}, maintaining its \textbf{adaptability} to novel domains, and \textbf{mitigating forgetting} throughout the continual adaptation process, all in a memory- and computation-efficient manner.

\begin{wrapfigure}{r}{0.5\textwidth}
    \centering
    \vspace{-1em}
    \includegraphics[width=0.49\textwidth]{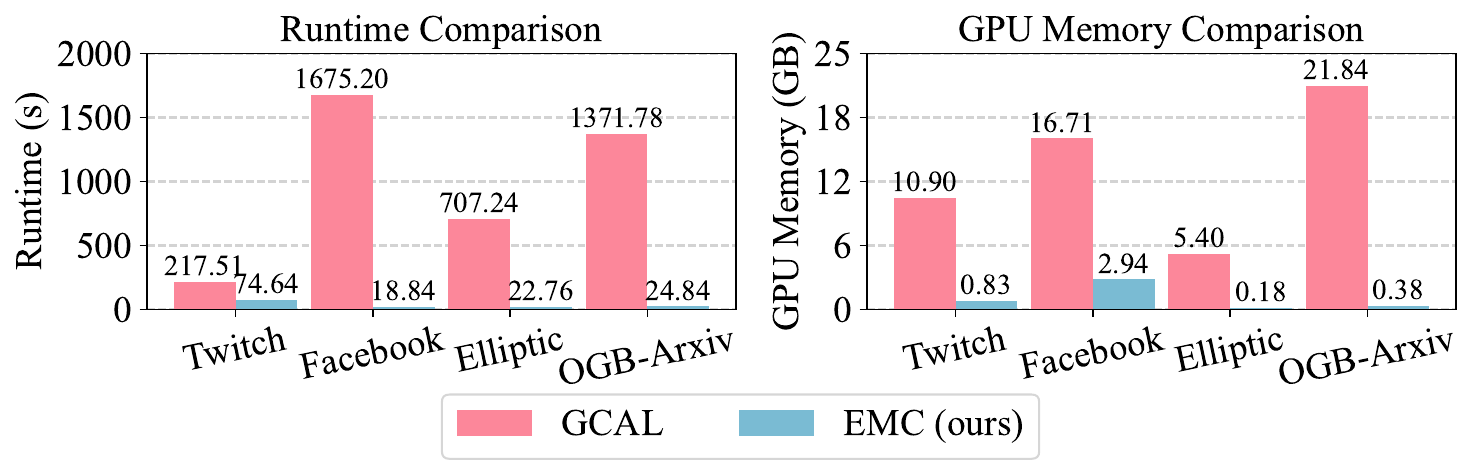}
    \vspace{-0.2cm}
    \caption{Runtime and GPU memory cost comparison between baseline GCAL and \ourmethod.}
    \label{fig:compare-bar}
    \vspace{-0.8em}
\end{wrapfigure}
In this paper, we argue that a more economical path exists: rather than \emph{generating} memory through a separately trained module, one can \emph{crystallize} it directly from each incoming graph in closed form. We instantiate this idea as \textbf{\ourmethod} (\textbf{E}fficient \textbf{M}emory \textbf{C}rystallization), a test-time framework for graph learning under non-stationary distribution shifts. Specifically, we introduce a memory- and computation-efficient procedure that crystallizes each incoming graph domain into a compact, semantically faithful memory through a closed-form solution to a memory-oriented distribution-matching objective, eliminating redundant domain information under continual covariate shifts.
To preserve both generalizability and adaptability as the model traverses a long sequence of target domains, we further model inter-domain dependencies through state-evolving memories.
We theoretically establish that the generalization error bounds with crystallized memories (from source to target domain) and with state-evolving memories (during continual adaptation on target domains) are both tighter than directly adapting the model on the target domain without such mechanisms, ensuring effectiveness.
Extensive experiments on real-world datasets demonstrate the superiority of \ourmethod over state-of-the-art (SOTA) baselines.
A comparison of runtime and GPU memory cost between \ourmethod and the recent baseline GCAL~\cite{qiao2025gcal} across four benchmarks under continual drift is shown in Figure~\ref{fig:compare-bar}.
Notably, \ourmethod achieves substantial efficiency gains, requiring on average \textbf{87.4\% less runtime} and \textbf{92.4\% less GPU memory} across all datasets.
The contributions of this work are as follows:

\begin{itemize} [leftmargin=*,noitemsep,topsep=1.5pt]
  \item We propose an efficient and effective framework, \ourmethod, for graph learning under non-stationary distribution shifts at test time. To the best of our knowledge, this is the first approach to integrate closed-form memory crystallization with state-evolving adaptation {without requiring any additional model training at test time}.

  \item We introduce a memory-oriented distribution-matching objective with a closed-form solution for memory construction. We further model inter-domain dependencies through state-evolving memories and derive tighter generalization error bounds.

  \item Extensive experiments validate the effectiveness and efficiency of \ourmethod, showing superior performance compared to existing SOTA methods, while reducing average runtime by 87.4\% and GPU memory consumption by 92.4\% relative to the recent competitor.
\end{itemize}

%% file: chapter/3-method.tex
\vspace{-3.0mm}
\section{Preliminaries} \label{sec:problem}
\vspace{-2.0mm}
\paragraph{Problem Setup.}
We consider the task of test-time adaptation on graphs under non-stationary distribution shifts.
Given a GNN model $f(G^{(s)};\mathbf{\Theta}_0) \rightarrow \mathcal{Y}$ that is pre-trained on one or multiple source graphs $G^{(s)}$ drawn from the source distribution $\mathbb{P}_\mathcal{S}(G, Y)$, where $\mathbf{\Theta}_0$ denotes the parameters of the pre-trained GNN encoder and $\{Y\}\in \mathcal{Y}$ is the label set.
The target data is represented as a sequence of graphs $\mathcal{T} = \{G^t\}_{t=1}^T$ containing $T$ test samples drawn from the marginalized target distribution. Each graph $G^{(t)} = \{\mathbf{A}^{(t)}, \mathbf{X}^{(t)}\} \in \mathbb{P}_\mathcal{T}(G)$ corresponds to the $t$-th target domain, where $\mathbf{A}^{(t)} \in \mathbb{R}^{N_t \times N_t}$ is the adjacency matrix with $N_t$ nodes, and $\mathbf{X}^{(t)} \in \mathbb{R}^{N_t \times d}$ is the node feature with $d$-dimensional attributes. We use $G^t$ to specifically denote the graph from the $t$-th target domain in the arrival order, and $G^{(t)}$ as a generic notation when the index $(t)$ may refer to any past or current domain in $\{1,\ldots,T\}$.
With the covariate shift assumption, i.e., $\mathbb{P}_\mathcal{S}(G) \neq \mathbb{P}_\mathcal{T}(G)$ and $\mathbb{P}_\mathcal{S}(Y | G) = \mathbb{P}_\mathcal{T}(Y | G)$, we aim to continually adapt the model parameters from $\mathbf{\Theta}_{t-1}$ to $\mathbf{\Theta}_t$ to maintain the generalizability and adaptability\footnote{A detailed discussion on the concepts and distinctions between generalizability and adaptability is in Appendix~\ref{app:setting}.} on target samples from $\mathbb{P}_\mathcal{T}(G)$, while mitigating the risk of catastrophic forgetting on previously encountered domains.
We summarize the frequently used notations in Appendix~\ref{sec:notation}.

\begin{figure*}
    \centering
    \includegraphics[width=\linewidth]{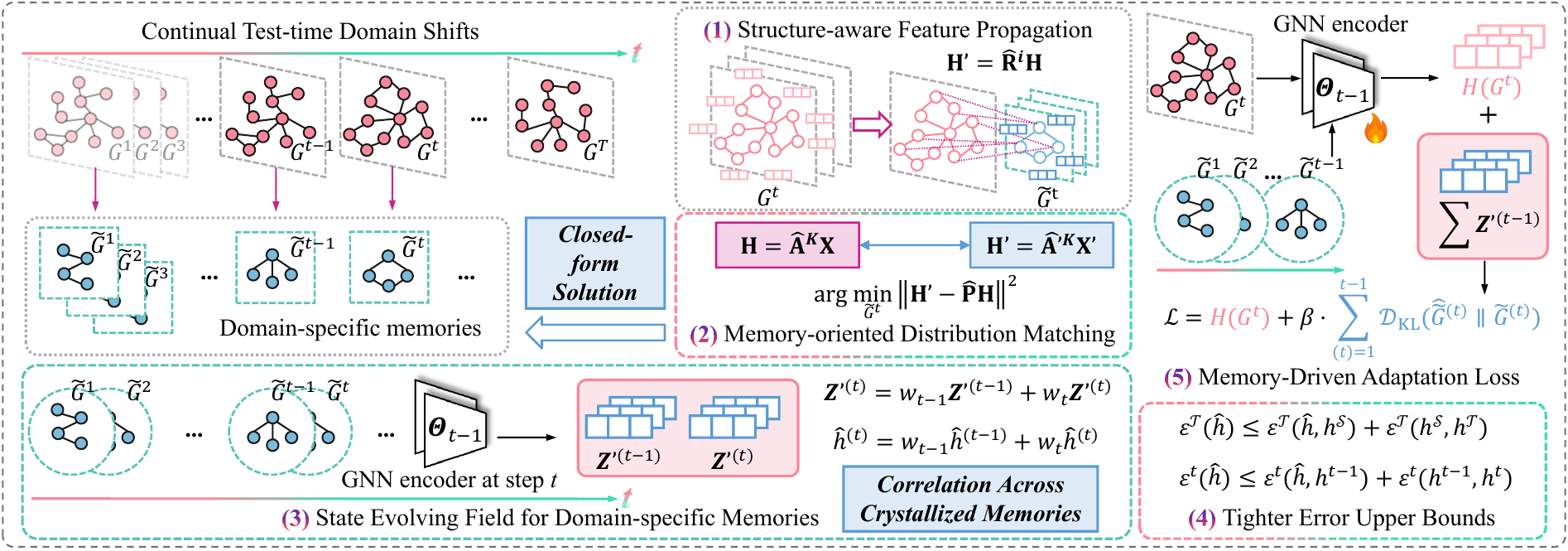}
    \caption{The framework of \ourmethod.
(1) Multi-hop feature propagation on the input graph $G^t$.
(2) A closed-form solution to the memory-oriented distribution-matching objective crystallizes the domain-specific memory $\tilde{G}^t$.
(3) The state evolving field models the correlation across consecutive domains.
(4) Tighter generalization error upper bounds guarantee effectiveness.
(5) The adaptation loss updates the parameters $\mathbf{\Theta}_{t-1}\!\to\!\mathbf{\Theta}_t$.}
    \label{fig:framework}
    \vspace{-1.5em}
\end{figure*}

\vspace{-3.0mm}
\section{Methodology}
\vspace{-2.0mm}
In this section, we elaborate on the proposed efficient test-time framework \ourmethod with its overall pipeline shown in Figure~\ref{fig:framework}.
To efficiently extract low-dimensional representations with reduced redundancy within each domain, we first reformulate a memory-oriented distribution-matching objective and derive a closed-form solution that \emph{crystallizes} each incoming domain into a compact, semantically faithful memory.
To enable continual adaptation by capturing inter-domain relationships, we then propose state-evolving memories, which yield a tighter generalization error bound.
\vspace{-2.0mm}
\subsection{Efficient Memory Crystallization} \label{sec:proto}
\vspace{-2.0mm}
The original input graphs exhibit covariate shifts across domains and redundancies within each sample, which are unrelated to the underlying semantic information.
To retain only the representative characteristics of each domain, we crystallize a domain-specific memory as a low-rank surrogate graph that approximates the original graph in the embedding space.

\vspace{-2.0mm}
\paragraph{Structure-aware Feature Propagation.}
Given an input graph $G^{(t)}=\{\bf A,\bf X\}$, the non-parametric feature propagation module is deployed to smooth the node features at the $K$-th order as
${\bf H} = {\hat{\bf A}}^K {\bf X}$,
where ${\hat{\bf{A}}}={\tilde{\bf{D}}}^{\frac{1}{2}}{\tilde{\bf{A}}}{\tilde{\bf{D}}}^{\frac{1}{2}}$ is the normalized adjacency matrix with self-loops, and ${\tilde{\bf{D}}}$ is the degree matrix for ${\tilde{\bf{A}}=\bf{A}+\bf{I}}$.
This propagation scheme can be replaced with other alternatives, such as SGC~\cite{wu2019simplifying} and personalized PageRank (PPR)~\cite{page1999pagerank}.
Higher-order propagation produces increasingly smooth features that capture more global structural information.

To preserve semantic structure, we apply class-wise node partitioning,
which divides the $|C_i|$ nodes assigned to class $i$ into $|C'_i|$ non-overlapping sub-classes $\{S^i_1, ..., S^i_{|C'_i|}\}$.
At test time, the class assignment is guided by the model's predicted semantics, which serve as \emph{semantic anchors} for grouping nodes with similar semantic content rather than as supervisory targets.
Each node $k$ further carries a confidence score $r_k$ that down-weights uncertain contributions during aggregation; a detailed discussion of this design is deferred to Appendix~\ref{app:semantic_anchor}.
The node mapping function of this partition is defined as $\pi^i:\{1, \dots, |C_i|\} \to \{1,\dots, |C'_i|\}$, and the class-wise aggregation matrix ${\bf{R}}^i \in \mathbb{R}^{|C'_i|\times |C_i|}$ is formulated as
\begin{equation}
{{\bf R}}^i_{j,k} = \begin{cases}
r_k /{\tau} & \text{if } \pi^i(k) = j, \\
0 & \text{otherwise},
\end{cases}
\label{eq_agg2}
\end{equation}
where $\tau$ is a temperature parameter.
Aggregated embeddings for class $i$ are obtained as ${\bf H}'_i = \hat{\bf R}^i {\bf H}_i$, where $\hat{\bf R}^i$ is the row-normalized version of ${\bf R}^i$ and ${\bf H}_i$ collects the propagated features of nodes in class $i$.
The full aggregation matrix is constructed by stacking all class-wise aggregation matrices block-diagonally as $\hat{\bf P}=\text{diag}(\hat{\bf R}^1, ..., \hat{\bf R}^{C})$, where $C$ is the total number of classes.

\vspace{-2.0mm}
\paragraph{Memory-oriented Distribution Matching.}
We denote the memory crystallized from $G^{(t)}$ as $\tilde{G}^{(t)} = \{ {\bf A}', {\bf X}' \}$.
To ensure that the memory faithfully represents the original graph, we reformulate a memory-oriented distribution-matching objective between the representations of the original graph and its memory.

\begin{proposition}
\label{prop:4.1}
After multi-hop propagation, let ${\bf Z} = {\bf H} \mathbf{\Psi} = {\hat{\bf A}}^{K} {\bf X} \mathbf{\Psi}$ and ${\bf Z}' = {\bf H}' \mathbf{\Psi} = {\hat{\bf A}}'^{K} {\bf X}' \mathbf{\Psi}$ denote the representations of the original graph $G^{(t)}$ and its memory $\tilde{G}^{(t)}$ respectively, where $\mathbf{\Psi}$ is the learnable linear weight of the GNN encoder and $\Phi$ is the distribution over its parameter configurations.
To directly align the distribution of the input domain with that of its memory, the discrepancy between the propagated features should be minimized as:
\begin{equation}
\min_{\tilde{G}^{(t)}} {\mathop{\mathbb{E}}_{{\mathbf{\Psi} } \sim \Phi}} \left \| {\bf P}'{\bf{H}}'{\mathbf{\Psi}}- {\bf P}{\bf{H}}{\mathbf{\Psi}} \right \|^{2}.
\end{equation}
Assuming that nodes in $\tilde{G}^{(t)}$ are organized in ascending order, ${\bf P}'$ degenerates into the identity matrix ${\bf I}$. Then, the distribution matching objective can be upper-bounded by:
\begin{equation}
 \small
{\mathop{\mathbb{E}}_{{\mathbf{\Psi}} \sim \Phi}}  \left \| {\bf{H}}'{\mathbf{\Psi}}- \hat{\bf P}{\bf{H}}{\mathbf{\Psi}} \right \|^{2}  \leq {\mathop{\mathbb{E}}_{{\mathbf{\Psi} } \sim \Phi}}  \left \| {\bf{H}}'- \hat{\bf P}{\bf{H}} \right \|^{2} \left \| {\mathbf{\Psi}}  \right \|^{2}.
\end{equation}
\end{proposition}
\vspace{-1em}
Proof for Proposition~\ref{prop:4.1} is provided in Appendix~\ref{ap-dis-match}. The expectation $\mathbb{E}_{\Psi\sim\Phi}$ is a technical device that yields a bound holding uniformly over the encoder parameter space; it is not a random projection.

\noindent \textbf{Remark.}
The reformulated objective aligns the essential semantic information between the input graph and its crystallized memory.
Since $\mathbf{\Psi}$ is independent of both $\tilde{G}^{(t)}$ and $\hat{\bf P}$, we can minimize the upper bound by:
\begin{equation}
\arg\min_{\tilde{G}^{(t)}, \hat{\bf P}} \left\| {\bf H}' - \hat{\bf P} {\bf H} \right\|^{2}.
\end{equation}
\vspace{-1em}

\vspace{-2.0mm}
\paragraph{Redundancy-eliminated Memory Construction.}
According to Proposition~\ref{prop:4.1}, we expect a symmetric encoding procedure between the original ${G}^{(t)}$ and its crystallized memory $\tilde{G}^{(t)}$.
Therefore, our objective is to construct $\bf{A}'$ and ${\bf X}'$ satisfying ${\hat{\bf{A}}}'^{K}{\bf{X}'} = {\bf{H}'}$.
To this end, we utilize the pre-defined graph structure and obtain ${\bf X}'$ in closed form.
Specifically, we construct the structure via a similarity-based thresholding scheme:
\vspace{0.5em}
\begin{equation}
\label{eq_adj}
{\bf A}'_{i,j} =
\begin{cases}
1 & \text{if } \cos({\bf H}'_i, {\bf H}'_j) > \tau, \\
0 & \text{otherwise},
\end{cases}
\end{equation}
where $\cos(\cdot, \cdot)$ measures the cosine similarity and $\tau$ is the hyper-parameter for graph sparsification.
To ensure that the reconstructed features change smoothly between connected nodes, we introduce the Dirichlet energy constraint~\cite{kalofolias2016learn} into the feature reconstruction loss to quantify the smoothness of graph signals. The optimization objective for ${\bf X}'$ is then formulated as:
\begin{equation}
\arg\min_{{\bf X}'} \left\| {\hat{\bf A}}'^K {\bf X}' - {\bf H}' \right\|^2 + \alpha \operatorname{tr}({\bf X}'^\top {\bf L}' {\bf X}'),
\label{eq_feat}
\end{equation}
where ${\bf L}' = {\bf D}' - {\bf A}'$ is the Laplacian matrix, $\textbf{D}'$ is the degree matrix, and $\alpha$ controls the smoothness strength.

\begin{proposition}
\label{prop-4.2}
Assume that $\tilde{G}^{(t)} = \{ {\bf A}', {\bf X}' \}$ is the crystallized memory of an input graph; the closed-form solution of Eq.~\eqref{eq_feat} takes the form ${\bf X}'=({\bf Q}^{\top}{\bf Q}+\alpha {\bf L}')^{-1}{\bf Q}^{\top}{\bf H}'$, where ${\bf Q} = \hat{\bf{A}}'^{K}$.
\end{proposition}
\noindent \textit{Proof Sketch.} Taking the first- and second-order partial derivatives of the objective in Eq.~\eqref{eq_feat} with respect to ${\bf X}'$ shows that the objective is positive definite. Therefore, the closed-form solution that satisfies the first-order condition achieves the minimum.
A detailed proof is provided in Appendix~\ref{ap:prop-1}.
The closed-form solution operates on a small $N_t' \times N_t'$ system over the crystallized memory rather than the full $N_t \times N_t$ system on the original graph (with $N_t'$ set to 1\% of $N_t$ in our experiments), making memory construction lightweight and naturally scalable. Further empirical scaling analysis is provided in Appendix~\ref{ap:scalability}.

\vspace{-2.0mm}
\subsection{Memory-Driven Adaptation with Tighter Bound} \label{sec:adapt}
\vspace{-2.0mm}
To capture state transitions and evolutionary patterns across domains, we introduce the State Evolving Field mechanism over crystallized memories with derivation of a tighter generalization error bound.

\begin{proposition}
\label{prop:4-3}
Given any pair of sequences of samples from the target domain, $\mathcal{T} = \{G^t\}_{t=1}^T$ and the corresponding crystallized memories $\tilde{\mathcal{T}} = \{\tilde{G}^t\}_{t=1}^T$, following the joint distributions $\mathbb{P}_{G^{1:T}}$ and $\mathbb{Q}_{\tilde{G}^{1:T}}$ respectively, with each $G^t$ and $\tilde{G}^t$ following the marginal distributions $\mathbb{P}_{G^t}$ and $\mathbb{Q}_{\tilde{G}^t}$. Then the following relation holds:
\vspace{-1.0mm}
\begin{equation}
\mathcal{D}_{\mathrm{KL}} (\mathbb{P}_{G^{1:T}} \,\|\, \mathbb{Q}_{\tilde{G}^{1:T}})
\equiv \sum_{t=1}^T  \mathcal{D}_{\mathrm{KL}}(\mathbb{P}_{G^t}\,\|\,\mathbb{Q}_{\tilde G^t} )
\quad \text{iff}\quad \mathbb{P}_{G^{1:T}} = \prod_{t=1}^T \mathbb{P}_{G^t},\  \mathbb{Q}_{\tilde{G}^{1:T}} = \prod_{t=1}^T \mathbb{Q}_{\tilde{G}^t},
\end{equation}
where $\mathcal{D}_{\mathrm{KL}}$ denotes the Kullback--Leibler divergence and $\equiv$ indicates that the two expressions are identically equal.
\end{proposition}

\noindent \textbf{Remark.}
Proposition~\ref{prop:4-3} (proof in Appendix~\ref{ap:prop-3}) is a \emph{biconditional}: it does not assume that domains are independent. Rather, it characterizes precisely when summing per-domain KL divergences equals the joint KL divergence, namely when and only when the domains (and their memories) are independent. Since real test-time domains are generally correlated under continual drift, this implies that aligning marginal distributions alone is insufficient. This motivates us to model inter-domain relationships explicitly through the state evolving field below.

\vspace{-2.0mm}
\paragraph{State Evolving Field over Crystallized Memories.}
To model the temporal correlation across domains, we define a state evolving field over crystallized memories. At each time step $t$, the memory representation ${\bf Z}'^{(t)}$ evolves from its predecessor ${\bf Z}'^{(t-1)}$ through a convex combination:
\begin{equation} \label{eq-8}
    {\bf Z}'^{(t)}\leftarrow w_{t-1}{\bf Z}'^{(t-1)} + w_t {\bf Z}'^{(t)}, \text{where } w_{t-1}+w_{t}=1.
\end{equation}
Here $w_{t-1}$ and $w_t$ are fixed non-negative hyperparameters satisfying the convex-combination constraint and selected by grid search; an adaptive weighting strategy is a natural extension since Theorem~\ref{bound3} below holds for any valid choice.
During adaptation, $f(G^{t};\mathbf{\Theta}_{t-1}) = h\circ g$ is the composition of an encoder $h(\cdot)$ and a fixed classifier $g(\cdot)$, \ie, $\mathbf{Z}'^{(t)} = h(\tilde{G}^{t})$ and $\hat{{Y}}^{t} = g(\mathbf{Z}'^{(t)})$.
At each step $t$, the mixed memory ${\bf Z}'^{(t)} = w_{t-1} {\bf Z}'^{(t-1)} + w_t {\bf Z}'^{(t)}$ induces a corresponding hypothesis $\hat{h} = w_{t-1} h^{t-1} + w_t h^t$ in the decision space.

\begin{lemma}[Error Upper Bound for Domain Shift]
\label{bound1}
Let $\tilde{\mathbb{P}}_\mathcal{S}$ and $\tilde{\mathbb{P}}_\mathcal{T}$ denote the induced distributions over the feature space for $\mathbb{P}_\mathcal{S}$ and $\mathbb{P}_\mathcal{T}$ respectively. The following inequality holds for the risk $\epsilon_{0}^\mathcal{T}(\hat{h})$ with single-step adaptation on the target domain $\mathbb{P}_\mathcal{T}$:
\begin{equation}
\epsilon_{0}^\mathcal{T}(\hat{h}) \leq \min\{\epsilon^\mathcal{S}(h^\mathcal{S}, h^\mathcal{T}), \epsilon^\mathcal{T}(h^\mathcal{S}, h^\mathcal{T})\} + \epsilon^\mathcal{S}(\hat{h}) + d_{\mathcal{H}}(\tilde{\mathbb{P}}_\mathcal{S}, \tilde{\mathbb{P}}_\mathcal{T}),
\end{equation}
where $d_{\mathcal{H}}$ is the $\mathcal{H}$-divergence. 
\end{lemma}

\begin{theorem} [Error Upper Bound with Crystallized Memories]
\label{bound2}
Let $\tilde{\mathbb{P}}_\mathcal{S}$ and $\tilde{\mathbb{P}}_\mathcal{T}$ denote the induced distributions over the feature space for $\mathbb{P}_\mathcal{S}$ and $\mathbb{P}_\mathcal{T}$ respectively. The risk $\epsilon^\mathcal{T}(\hat{h})$ with crystallized memories (with mixed weight $w_0 + w_1 = 1$, $w_0 \geq 0$, $w_1 \geq 0$, and $\hat{h} = w_0 \hat{h}^\mathcal{S} + w_1 \hat{h}^\mathcal{T}$) is upper bounded by:
\begin{equation}
\epsilon^\mathcal{T}(\hat{h}) \leq \epsilon^\mathcal{T}(\hat{h}, h^\mathcal{S}) + \epsilon^\mathcal{T}(h^\mathcal{S}, h^\mathcal{T}).
\end{equation}
\end{theorem}

\begin{theorem}[Continual Error Upper Bound with State-Evolving Memories]
\label{bound3}
Let $\tilde{\mathbb{P}}_{G^{t-1}}$ and $\tilde{\mathbb{P}}_{G^t}$ denote the induced feature distributions of two consecutive target domains at steps $t{-}1$ and $t$. The risk $\epsilon^{t}(\hat{h})$ with state-evolving memories (with mixed weight $w_{t-1} + w_t = 1$, $w_{t-1}, w_t \geq 0$, and $\hat{h} = w_{t-1} \hat{h}^{t-1} + w_t \hat{h}^t$) is upper bounded by:
\begin{equation} \label{eq:thm-bound3}
\epsilon^t(\hat{h}) \leq \epsilon^t(\hat{h}, h^{t-1}) + \epsilon^t(h^{t-1}, h^t).
\end{equation}
\end{theorem}
\vspace{-2.0mm}
\noindent \textbf{Remark.}
The bounds on $\epsilon_{0}^{\mathcal{T}}(\hat{h})$ characterize per-step generalization at each adaptation step $t$, while the continual aspect across $\{G^t\}_{t=1}^T$ is handled by the state-evolving design ($\rhd$ Theorem~\ref{bound3}) with past memories. At each step, the improvement of Theorem~\ref{bound2} over Lemma~\ref{bound1} comes from eliminating the explicit source--target $\mathcal{H}$-divergence penalty $d_{\mathcal{H}}(\tilde{\mathbb{P}}_\mathcal{S},\tilde{\mathbb{P}}_\mathcal{T})$ that appears in Lemma~\ref{bound1}, which yields a bound that dominates $\epsilon_0^\mathcal{T}(\hat{h})$ by a non-negative margin of $d_{\mathcal{H}}(\tilde{\mathbb{P}}_\mathcal{S},\tilde{\mathbb{P}}_\mathcal{T}) \geq 0$ regardless of the relative ordering of $\epsilon^\mathcal{S}(h^\mathcal{S},h^\mathcal{T})$ and $\epsilon^\mathcal{T}(h^\mathcal{S},h^\mathcal{T})$. We therefore arrive at the following conclusions:
\begin{itemize}[leftmargin=*,noitemsep,topsep=1.5pt]
    \item From Theorem~\ref{bound2}, the bound on $\epsilon^{\mathcal{T}}(\hat{h})$ obtained with crystallized memories admits a tighter form than the direct source-to-target bound $\epsilon_{0}^{\mathcal{T}}(\hat{h})$.
    \item \item By Theorem~\ref{bound3}, when EMC traverses consecutive domains via state-evolving memories, the per-step bound likewise admits a tighter form than $\epsilon_{0}^{\mathcal{T}}(\hat{h})$ at each step.
\end{itemize}
Detailed proofs and comparisons are in Appendix~\ref{ap:prop-3}.

\vspace{-2.0mm}
\paragraph{Derivation of the Adaptation Loss.}
Given a test graph from the $t$-th target domain and the sequence of memories from the previous $t-1$ domains, $\tilde{\mathcal{T}} = \{\tilde{G}^1, \tilde{G}^2, \ldots, \tilde{G}^{t-1}\}$, the GNN model with parameters ${\mathbf{\Theta}_{t-1}}$ is continually refined to adapt to the new graph while preserving performance on previously seen graphs.
The overall optimization objective of \ourmethod is:
\begin{equation}
    \mathcal{L} \doteq H(G^{t}) +
    \beta \cdot \mathcal{L}_\text{Adp}(\tilde{{G}}^{1:t-1};\mathbf{\Theta}_{t-1}),
    \text{ and } \mathcal{L}_\text{Adp}(\tilde{{G}}^{1:t-1};\mathbf{\Theta}_{t-1}) = \sum_{(t)=1}^{t-1}\mathcal{D}_\text{KL}(\hat{\tilde{G}}^{(t)}\|\tilde{G}^{(t)}),
\label{eq:loss-2}
\end{equation}
where $H(\cdot)$ denotes the Shannon entropy that measures the uncertainty within the $t$-th domain, $\beta$ is a hyperparameter, and $\mathcal{D}_\text{KL}(\cdot\|\cdot)$ is the KL divergence between two distributions. Here $\tilde{G}^{(t)}$ denotes the stored memory of the $(t)$-th domain (frozen anchor), and $\hat{\tilde{G}}^{(t)}$ denotes the predictive distribution induced by the current model with parameters $\mathbf{\Theta}_{t-1}$ evaluated on the same memory $\tilde{G}^{(t)}$. At each step the system first crystallizes $\tilde{G}^{t}$ and forms the mixed representation, and only then minimizes Eq.\eqref{eq:loss-2} to update $\mathbf{\Theta}_{t-1}\rightarrow \mathbf{\Theta}_{t}$.
The full pipeline and complexity analysis are provided in Appendix~\ref{sec:algorithm}.

\begin{proposition}[Stability Bound of the Adaptation Loss]
\label{prop:kl_stability}
Consider the adaptation step from $\mathbf{\Theta}_{t-1}$ to $\mathbf{\Theta}_{t}$.
For each previously encountered domain indexed by $(t)\in\{1,\ldots,t-1\}$,
$\mathcal{D}_{\mathrm{KL}}(\hat{\tilde{G}}^{(t)}\|\tilde{G}^{(t)})$ can be interpreted as the divergence between the
predictive distributions induced by the model before and after adaptation
on the same crystallized memory $\tilde{G}^{(t)}$.
Moreover, assume that:
(i) the predictive distribution is given by a softmax with temperature $\tau>0$;
(ii) all predicted class probabilities are lower-bounded by $\gamma>0$; and
(iii) the encoder $h(\cdot)$ and the classifier $g(\cdot)$ are Lipschitz continuous in the model parameters.
Then the following upper bound holds:
\begin{equation}
\mathcal{D}_{\mathrm{KL}}(\hat{\tilde{G}}^{(t)}\|\tilde{G}^{(t)})
\;\le\;
\frac{L_g^2 L_h^2}{4\,\gamma\,\tau^2}\,
\big\|\mathbf{\Theta}_{t}-\mathbf{\Theta}_{t-1}\big\|_2^2,
\label{eq:kl_explicit_bound}
\end{equation}
where $L_h$ and $L_g$ are the corresponding Lipschitz constants of $h(\cdot)$ and $g(\cdot)$.
\end{proposition}

\noindent \textbf{Remark.} The bound asserts that small parameter updates yield small predictive shifts on stored memories, providing a theoretical justification for the KL regularization in Eq.~\eqref{eq:loss-2} as a means to mitigate forgetting. Assumption~(iii) is mild and is satisfied by standard GNN backbones under common weight regularization. Detailed discussion together with proof is provided in Appendix~\ref{apsec-kl_stability}.

%% file: chapter/4-experiment.tex
\input{table/1-new-metric-tab}

\section{Experiments}
\vspace{-2.0mm}
In this section, we empirically evaluate the effectiveness and efficiency of the proposed \ourmethod.\footnote{The code of \ourmethod is available at: \url{https://github.com/name-is-what/EMC}.}
We organize the experiments around the following research questions:
\begin{itemize}[leftmargin=*,noitemsep,topsep=1.5pt]
    \item \textbf{RQ1:} How well does \ourmethod achieve {generalizability}, {adaptability}, and {forgetting mitigation} efficiently with low memory and computational cost under continual drift?
    \item \textbf{RQ2:} How does \ourmethod scale with respect to graph size and the number of domains?
    \item \textbf{RQ3:} What are the impacts of different components on the effectiveness and efficiency of \ourmethod?
    \item \textbf{RQ4:} How sensitive is \ourmethod to its hyperparameters?
\end{itemize}

\vspace{-2mm}
\subsection{Experimental Setup}
\vspace{-0.2cm}
\noindent \textbf{Datasets.}
We evaluate \ourmethod on 4 node classification datasets, following the benchmark in GCAL~\cite{qiao2025gcal}, categorized into two types:
\begin{itemize}[leftmargin=*,noitemsep,topsep=1.5pt]
    \item For \textbf{regional shifts}, the Facebook-100~\cite{fb100} and Twitch-Explicit~\cite{twitch} datasets consist of multiple social network graphs collected from different regions.
    \item For \textbf{temporal shifts}, OGB-Arxiv~\cite{ogbn-arxiv} is a citation network of academic papers, and Elliptic~\cite{elliptic} is a Bitcoin transaction network, with both containing graphs spanning different time steps.
\end{itemize}
In all cases, each graph is treated as a separate domain. We pre-train a GNN model on selected source domains and then perform continual adaptation on the remaining domains. Dataset statistics and further details are in Appendix~\ref{app:datasets}.

\noindent \textbf{Baselines.}
We compare \ourmethod against 10 representative baseline methods.
``Base Model'' refers to the pre-trained model directly applied to target datasets without any adaptation, serving as a lower bound. The baselines are grouped into two categories: (1) \textit{One-step Test-Time Adaptation Methods}~\cite{DANN,TENT,Norm,EERM,jin2022empowering,zheng2024online}, and (2) \textit{Continual Test-Time Adaptation Methods}~\cite{CoTTA,EATA,qiao2025gcal}.

\noindent \textbf{Evaluation and Implementation.}
We report the performance matrix $M \in \mathbb{R}^{T \times T}$, a lower triangular matrix where $M_{i,j}$ (for $i \geq j$) denotes the performance on domain $j$ after training on domain $i$. Following~\cite{jin2022empowering,EERM}, we use ROC-AUC for the Twitch-Explicit dataset, Accuracy for Facebook-100 and OGB-Arxiv, and F1 Score for Elliptic.
To summarize performance across all domains, we compute the Average Performance (AP) as $\text{AP}=\frac{1}{T}\sum_{i=1}^T M_{T,i}$, which reflects overall adaptation ability, and the Average Forgetting (AF) as $\text{AF}=\frac{1}{T-1}\sum_{i=1}^{T-1} (M_{T,i} - M_{i,i})$, which quantifies forgetting during sequential training.
To further assess generalizability and adaptability, we follow~\cite{liu2023deja} to introduce 2 complementary metrics:
{Target Domain Generalization (TDG)} measures the model's performance on each domain before adaptation as $\text{TDG} = \frac{1}{T-1} \sum_{i=2}^{T} M_{i-1,i}$. {Target Domain Adaptation (TDA)} measures performance after adaptation as $\text{TDA} = \frac{1}{T} \sum_{i=1}^{T} M_{i,i} $.
All reported results are averaged over 5 runs with standard deviations.
We perform a grid search to select the key hyperparameters of \ourmethod.
More detailed experimental settings are provided in Appendix~\ref{app:setting}.

\vspace{-2mm}
\subsection{Overall Performance Comparison (RQ1)} \label{sec:overall}
\vspace{-0.2cm}
\textbf{Generalizability and Adaptability Analysis.}
To answer RQ1, we evaluate the generalizability and adaptability of \ourmethod using the comprehensive metrics in Table~\ref{tab:more-metric}.
We compare \ourmethod with two recent baselines representing different categories of domain adaptation methods.
In addition to AP and AF, \ourmethod consistently maintains strong generalizability and adaptability under continual drift.
LEBED, not specifically designed for continual settings, shows weaker generalization. Although GCAL is tailored for continual adaptation, it incurs substantial computational overhead.
Notably, on the Elliptic dataset with GCN as the backbone, \ourmethod reduces runtime by 96.78\% and GPU memory by 96.58\%. On OGB-Arxiv with SAGE, it reduces runtime by 98.19\% and GPU memory by 98.25\%.
The results confirm \ourmethod's superior effectiveness and efficiency for graph learning under non-stationary distribution shifts at test time, ensuring generalizability and adaptability while significantly reducing computational overhead.

\input{table/1-main-tab}

\textbf{Extended Comparisons.}
We further compare \ourmethod with 10 representative baselines following the GCAL setting ($\rhd$ Table~\ref{tab:all_performance} and Appendix~\ref{ap:full-baseline}), and conduct extended comparison against the very recent baselines GCAL~\cite{qiao2025gcal} and STEM~\cite{luo2026space} on both node-level and graph-level adaptation ($\rhd$ Appendix~\ref{ap:stem-comparison}). These results confirm that \ourmethod generalizes across task granularities and adaptation protocols, attaining strong effectiveness without sacrificing memory efficiency.

\begin{figure*}[t]
\centering
\hspace{-0.28cm}
 \subfigure[GCAL on Facebook-100]{\label{fig:ma-2}
   \includegraphics[width=0.24\linewidth]{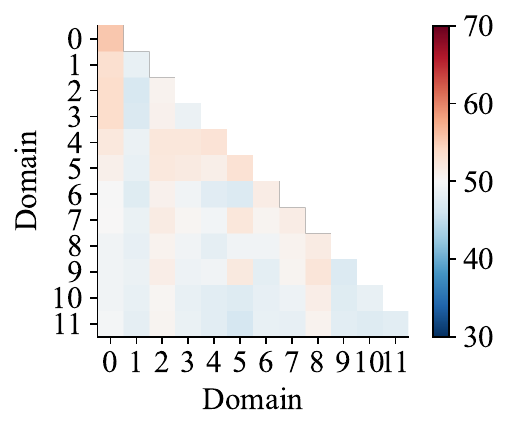}}
\hspace{-0.28cm}
 \subfigure[\ourmethod on Facebook-100]{\label{fig:ma-1}
   \includegraphics[width=0.24\linewidth]{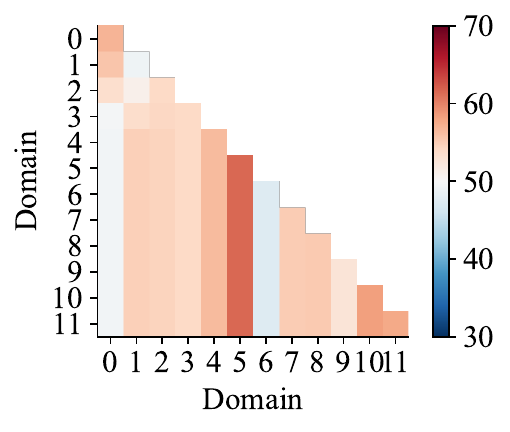}}
\hspace{-0.28cm}
 \subfigure[GCAL on OGB-Arxiv]{\label{fig:ma-4}
   \includegraphics[width=0.24\linewidth]{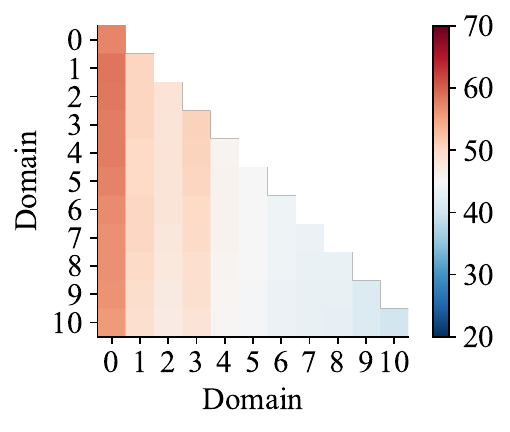}}
\hspace{-0.28cm}
 \subfigure[\ourmethod on OGB-Arxiv]{\label{fig:ma-3}
   \includegraphics[width=0.24\linewidth]{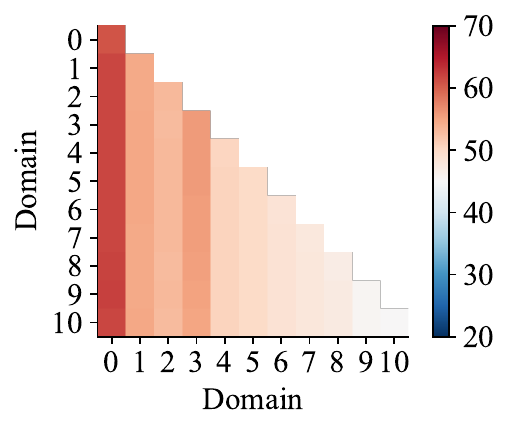}}
\hspace{-0.28cm}

\vspace{-3mm}
\caption{Performance matrices of \ourmethod and GCAL on different datasets.
}
\vspace{-1em}
\label{fig:matrix}
\end{figure*}

\textbf{Visualization of the Performance Matrix.}
We visualize the overall performance during continual adaptation using the performance matrix. As shown in Figure~\ref{fig:matrix}, \ourmethod effectively mitigates forgetting under continual drift, maintaining strong performance across multiple target domains throughout the adaptation process. Detailed analysis with additional datasets is provided in Appendix~\ref{ap:addition-result}.

\begin{figure*}
    \centering
    \includegraphics[width=\linewidth]{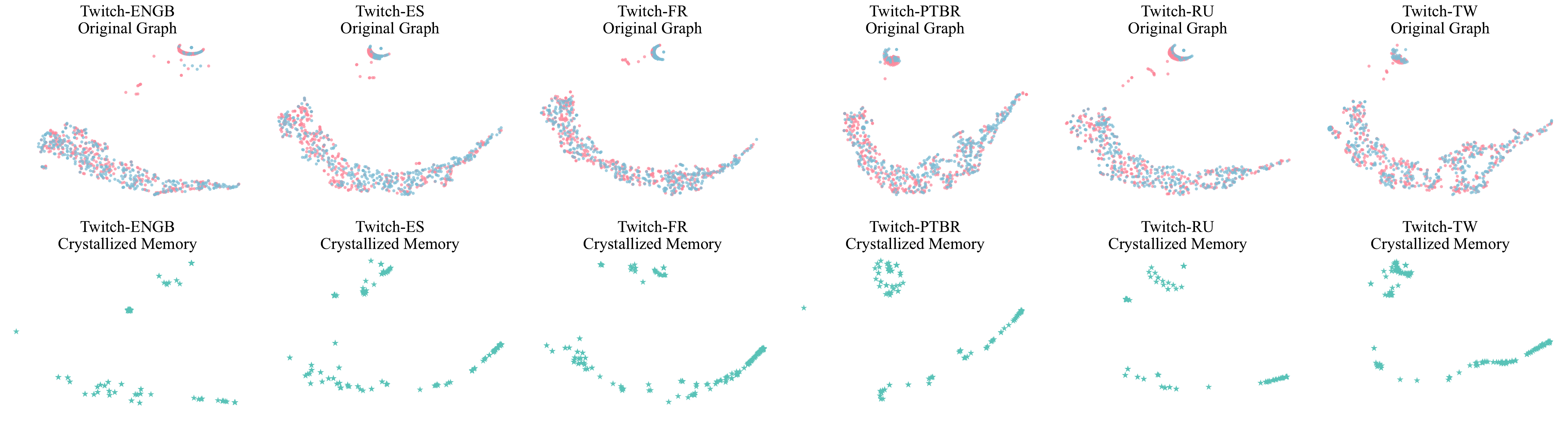}
    \vspace{-1em}
    \caption{T-SNE visualization of representations of the original graph and its crystallized memory.}
    \label{fig:tsne-twitch-gcn}
\end{figure*}

\textbf{Visualization of Crystallized Memories.}
To intuitively examine what the crystallized memories preserve, we provide t-SNE projections of the original and memory embeddings produced by the GNN encoder on the continuous-domain Twitch graphs in Figure~\ref{fig:tsne-twitch-gcn}. The two embedding sets are closely aligned in the representation space, confirming that \ourmethod retains the semantically informative content of each domain in a much more compact form. Additional structural visualizations rendered with NetworkX are provided in Appendix~\ref{ap:more-vis}.

\begin{figure*}[!t]
    \begin{minipage}{0.5\textwidth}
        \centering
        \hspace{-0.5em}
        \includegraphics[width=0.5\linewidth]{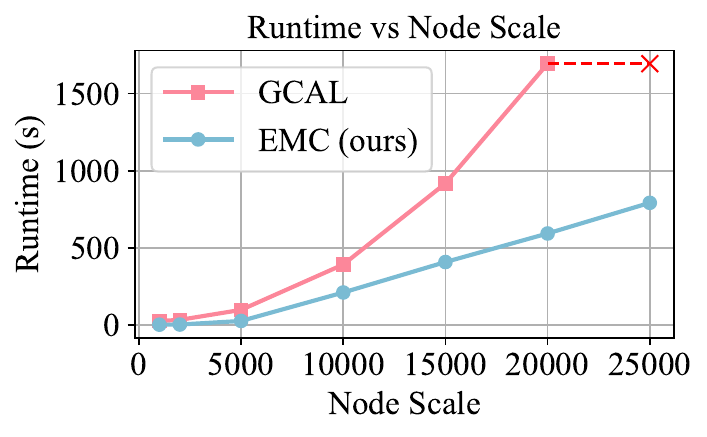}
        \hspace{-0.5em}
        \includegraphics[width=0.5\linewidth]{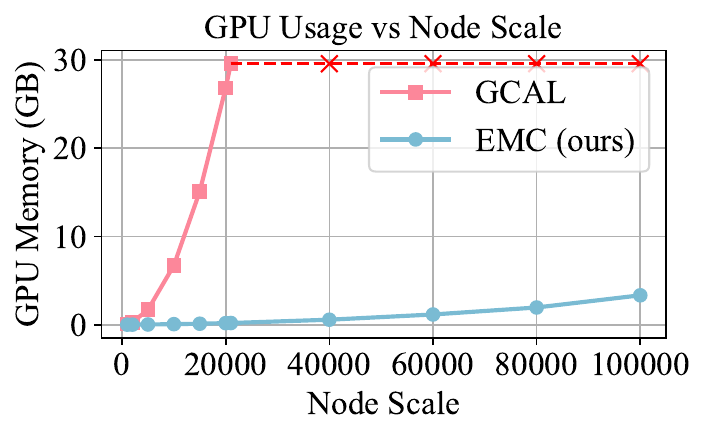}
        \vspace{-0.5em}
        \caption{Node scaling efficiency evaluation.}
        \label{fig:eff2}
    \end{minipage}
    \begin{minipage}{0.5\textwidth}
        \centering
        \hspace{-0.5em}
        \includegraphics[width=0.5\linewidth]{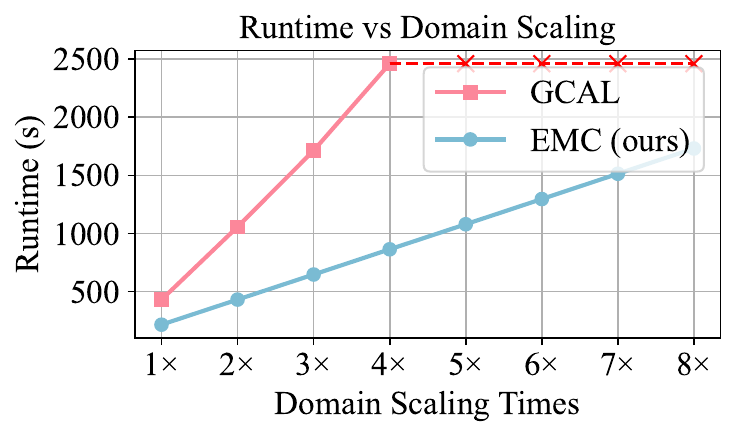}
        \hspace{-0.5em}
        \includegraphics[width=0.5\linewidth]{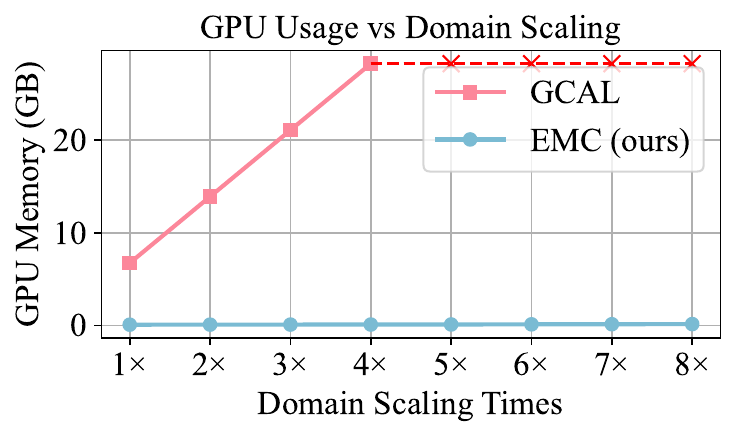}
        \vspace{-0.5em}
        \caption{Domain scaling efficiency evaluation.}
        \label{fig:eff1}
    \end{minipage}
\vspace{-1em}
\end{figure*}

\vspace{-2mm}
\subsection{Scaling Efficiency Analysis (RQ2)}
\vspace{-2mm}
To evaluate the efficiency of \ourmethod, we generate synthetic datasets by varying the number of nodes and the number of domains. We plot the runtime and peak GPU memory usage in Figures~\ref{fig:eff2} and~\ref{fig:eff1}, where the red cross indicates an Out-Of-Memory (OOM) error. Additional results and detailed settings are provided in Appendix~\ref{ap:scale}.
When the number of nodes exceeds 40k, both GCAL and LEBED are stuck due to OOM, while \ourmethod still operates within GPU limits even at the highest tested scale. We further benchmark the memory crystallization step alone on a single graph up to \textbf{200k nodes} (see in Appendix~\ref{ap:scalability}), where the closed-form inversion takes only $0.71$s, confirming that the cost on the small $N_t'\times N_t'$ system remains negligible relative to the propagation cost on the original graph.

\input{table/2-abla-tab}

\vspace{-2mm}
\subsection{Ablation Studies (RQ3)}
\vspace{-2mm}
We conduct ablation studies to analyze the effectiveness of two architectural variants:
\begin{itemize} [leftmargin=*,noitemsep,topsep=1.5pt]
    \item \textbf{\textit{w/o DM}:} We remove the Domain-specific Memory crystallization in Section~\ref{sec:proto}.
    \item \textbf{\textit{w/o EA}:} We remove the memory-driven adaptation in Section~\ref{sec:adapt}.
\end{itemize}
Results are shown in Table~\ref{tab:abla}.
Overall, \ourmethod outperforms both variants, validating the indispensable contribution of each component.
In addition, removing ``\textit{DM}'' leads to a significant drop in both efficiency and effectiveness, while removing ``\textit{EA}'' reduces runtime and GPU memory but causes considerable degradation in adaptation performance.

\begin{figure*}[!t]
    \begin{minipage}{0.5\textwidth}
        \centering
        \includegraphics[width=\linewidth]{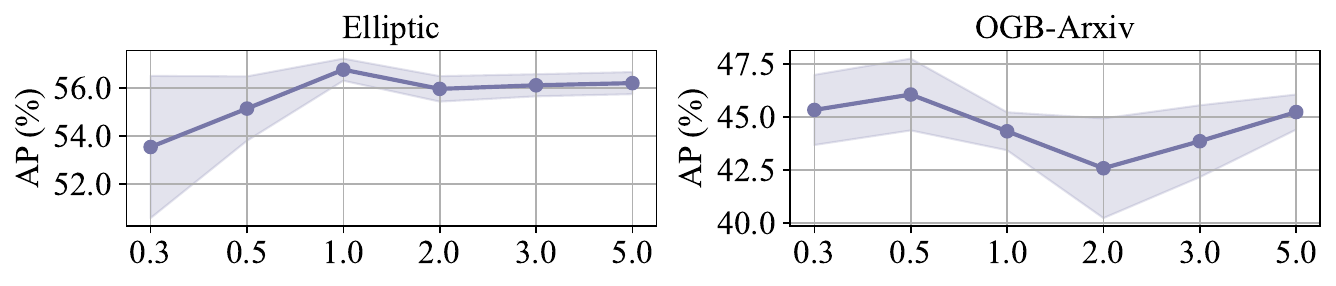}
        \vspace{-1.5em}
        \caption{The sensitivity of hyperparameter $\alpha$.}
        \label{fig:sensi-a}
    \end{minipage}
    \begin{minipage}{0.5\textwidth}
        \centering
        \includegraphics[width=\linewidth]{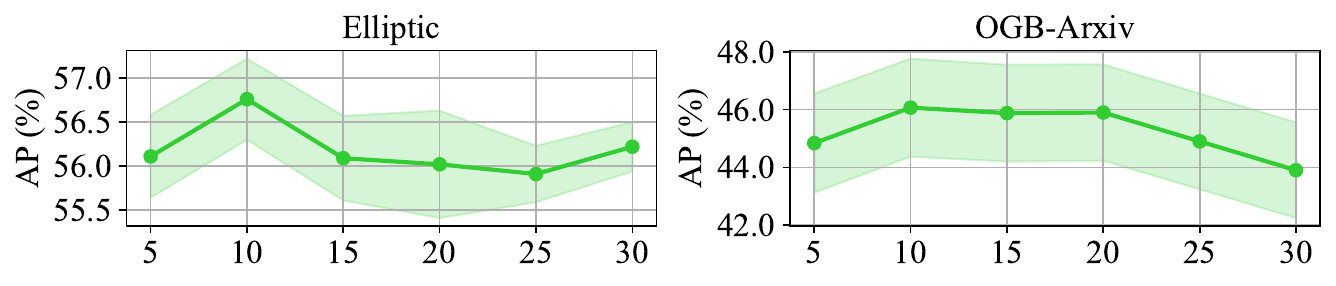}
        \vspace{-1.5em}
        \caption{The sensitivity of hyperparameter $\beta$.}
        \label{fig:sensi-t}
    \end{minipage}
\vspace{-1em}
\end{figure*}

\vspace{-2mm}
\subsection{Parameter Study (RQ4)}
\vspace{-0.2cm}
\noindent \textbf{Sensitivity Analysis of $\alpha$.}
To analyze the sensitivity of $\alpha$ for \ourmethod, we vary it from 0.3 to 5.0. The AP w.r.t. different $\alpha$ values is plotted in Figure~\ref{fig:sensi-a}.
The results confirm that \ourmethod is robust to reasonable choices of $\alpha$, while tuning remains important for optimal performance.

\noindent \textbf{Sensitivity Analysis of $\beta$.}
We further evaluate the sensitivity of $\beta$, which controls the importance of memory replay during adaptation.
Results in Figure~\ref{fig:sensi-t} show that the performance is sensitive to $\beta$ within a reasonable range.
Further analysis on $\tau$ and $\lambda$ is provided in Appendix~\ref{ap:sensi}.

%% file: table/1-new-metric-tab.tex
\begin{table*}[t]
\caption{Performance comparison 
($\%$, mean $\pm$ standard deviation for five runs)
on different backbones with comprehensive metrics, Runtime (s) and GPU Memory Usage (MiB). 
We compare \ourmethod with the most recent 2 baselines from two categories of domain adaptation methods.
`$\uparrow$' indicates higher is better, and `$\downarrow$' indicates lower is better.
`OOM' implies Out-Of-Memory.
The best results are highlighted with \textbf{bold} and the runner-ups are \underline{underlined}.
}
\centering

  \centering
  \resizebox{\textwidth}{!}{
\begin{tabular}{c|c|cccccccccccc}
  \toprule

\multirow{2}[4]{*}{\raisebox{0.15em}{\textbf{Dataset}}} & \multirow{2}[4]{*}{\textbf{Methods}} & \multicolumn{6}{c}{\textbf{Backbone: GCN}} & \multicolumn{6}{c}{\textbf{Backbone: SAGE}} \\

\cmidrule(lr){3-8} \cmidrule(lr){9-14}

         &      & AP ($\uparrow$)   & AF ($\uparrow$)   & TDG ($\uparrow$)  & TDA ($\uparrow$)  & Time ($\downarrow$) & GPU ($\downarrow$)  & AP ($\uparrow$)   & AF ($\uparrow$)   & TDG ($\uparrow$)  & TDA ($\uparrow$)  & Time ($\downarrow$) & GPU ($\downarrow$) \\
\midrule

\multirow{3}[0]{*}{\textbf{Twitch}} & LEBED & 51.06±0.05 & 0.24±0.05 & 51.87±0.06 & 48.90±0.04 & \underline{114.97} & 1440.57 & 56.37±0.03 & 0.58±0.04 & 57.55±0.02 & 55.08±0.04 & \underline{252.80} & \underline{3652.58} \\
         & GCAL & \underline{55.65±0.09} & \underline{0.42±0.13} & \underline{54.78±0.60} & \underline{53.20±0.62} & 217.51 & \underline{11159.8} & \underline{58.20±0.35} & \underline{0.79±0.13} & \underline{59.65±0.39} & \underline{57.02±0.49} & 295.11 & 8580.63 \\
& \cellcolor{myblue}\textbf{\ourmethod} & \cellcolor{myblue}\textbf{55.95±0.15} & \cellcolor{myblue}\textbf{0.47±0.08} & \cellcolor{myblue}\textbf{55.56±0.13} & \cellcolor{myblue}\textbf{55.83±0.14} & \cellcolor{myblue}\textbf{74.64} & \cellcolor{myblue}\textbf{854.97} & \cellcolor{myblue}\textbf{61.62±0.12} & \cellcolor{myblue}\textbf{0.86±0.04} & \cellcolor{myblue}\textbf{62.02±0.13} & \cellcolor{myblue}\textbf{60.36±0.16} & \cellcolor{myblue}\textbf{138.02} & \cellcolor{myblue}\textbf{2833.18} \\

\midrule
\multirow{3}[0]{*}{\textbf{Facebook}} & LEBED & 46.90±0.00 & -0.25±0.00 & 47.72±0.00 & 46.38±0.00 & \underline{93.59} & \underline{3481.53} & 46.92±0.00 & -2.20±0.01 & 48.78±0.00 & 48.39±0.01 & \underline{94.81} & \underline{3718.50} \\
         & GCAL & \underline{52.72±0.36} & \underline{0.72±0.19} & \underline{49.37±0.28} & \underline{48.96±0.29} & 1675.2 & 17107.84 & \underline{51.87±0.12} & \underline{0.04±0.11} &\underline {52.34±0.17} & \underline{51.55±0.08} & 1356.07 & 13633.79 \\
& \cellcolor{myblue}\textbf{\ourmethod} & \cellcolor{myblue}\textbf{53.10±0.00} & \cellcolor{myblue}\textbf{2.15±0.27} & \cellcolor{myblue}\textbf{53.49±0.27} & \cellcolor{myblue}\textbf{52.30±0.04} & \cellcolor{myblue}\textbf{18.84} & \cellcolor{myblue}\textbf{3006.08} & \cellcolor{myblue}\textbf{53.41±0.07} & \cellcolor{myblue}\textbf{0.63±0.20} & \cellcolor{myblue}\textbf{53.11±0.10} & \cellcolor{myblue}\textbf{52.87±0.25} & \cellcolor{myblue}\textbf{18.61} & \cellcolor{myblue}\textbf{3137.31} \\

\midrule
\multirow{3}[0]{*}{\textbf{Elliptic}} & LEBED & 49.44±0.19 & -4.86±0.20 & 52.10±0.06 & 53.82±0.02 & \underline{114.01} & \underline{745.36} & 62.12±0.11 & -6.65±0.15 & \underline{69.30±0.09} & \underline{67.99±0.08} & \underline{142.57} & \underline{1004.31} \\
         & GCAL & \underline{56.57±0.14} & \underline{0.88±0.13} & \underline{53.08±2.54} & \underline{50.31±2.98} & 707.24 & 5533.99 & \underline{71.35±8.57} & \underline{-3.40±1.10} & 63.45±8.82 & 61.38±9.82 & 453.87 & 972.59 \\
& \cellcolor{myblue}\textbf{\ourmethod} & \cellcolor{myblue}\textbf{56.76±0.46} & \cellcolor{myblue}\textbf{0.96±0.30} & \cellcolor{myblue}\textbf{58.73±0.40} & \cellcolor{myblue}\textbf{55.92±0.31} & \cellcolor{myblue}\textbf{22.76} & \cellcolor{myblue}\textbf{189.36} & \cellcolor{myblue}\textbf{72.28±1.92} & \cellcolor{myblue}\textbf{1.53±1.29} & \cellcolor{myblue}\textbf{74.55±2.50} & \cellcolor{myblue}\textbf{70.22±1.46} & \cellcolor{myblue}\textbf{21.67} & \cellcolor{myblue}\textbf{208.91} \\

\midrule
\multirow{3}[0]{*}{\textbf{Arxiv}} & LEBED & \underline{19.09±0.04} & \underline{-0.04±0.05} & 14.18±0.02 & \underline{17.44±0.02} & \underline{227.80} & \underline{839.39} & 41.81±0.01 & -1.85±0.01 & \underline{47.00±0.01} & \underline{42.12±0.02} & \underline{229.83} & \underline{1082.70} \\
         & GCAL & 14.15±6.94 & -1.80±2.06 & \underline{17.90±8.58} & 13.77±5.57 & 1362.76 & 22342.47 & \underline{45.22±0.17} & \underline{0.76±0.05} & 42.40±7.74 & 41.80±6.59 & 1371.78 & 22365.84 \\
& \cellcolor{myblue}\textbf{\ourmethod} & \cellcolor{myblue}\textbf{19.32±1.57} & \cellcolor{myblue}\textbf{3.38±1.99} & \cellcolor{myblue}\textbf{19.39±1.47} & \cellcolor{myblue}\textbf{20.39±1.71} & \cellcolor{myblue}\textbf{24.95} & \cellcolor{myblue}\textbf{299.18} & \cellcolor{myblue}\textbf{46.07±1.70} & \cellcolor{myblue}\textbf{0.99±0.66} & \cellcolor{myblue}\textbf{47.11±1.80} & \cellcolor{myblue}\textbf{43.52±1.25} & \cellcolor{myblue}\textbf{24.84} & \cellcolor{myblue}\textbf{391.02} \\

\bottomrule
\end{tabular}%
}

  \resizebox{\textwidth}{!}{
\begin{tabular}{c|c|cccccccccccc}
  \toprule

\multirow{2}[4]{*}{\raisebox{0.15em}{\textbf{Dataset}}} & \multirow{2}[4]{*}{\textbf{Methods}} & \multicolumn{6}{c}{\textbf{Backbone: GAT}} & \multicolumn{6}{c}{\textbf{Backbone: GIN}} \\

\cmidrule(lr){3-8} \cmidrule(lr){9-14}

         &      & AP ($\uparrow$)   & AF ($\uparrow$)   & TDG ($\uparrow$)  & TDA ($\uparrow$)  & Time ($\downarrow$) & GPU ($\downarrow$)  & AP ($\uparrow$)   & AF ($\uparrow$)   & TDG ($\uparrow$)  & TDA ($\uparrow$)  & Time ($\downarrow$) & GPU ($\downarrow$) \\
\midrule

    \multirow{3}[0]{*}{\textbf{Twitch}} & LEBED & \underline{53.55±0.10} & \underline{-0.84±0.05} & \underline{54.87±0.07} & \underline{53.26±0.13} & \underline{164.51} & \underline{22461.43} & 50.29±0.10 & -4.05±0.27 & 52.73±0.14 & \underline{52.33±0.19} & \underline{131.26} & \underline{3542.58} \\
         & GCAL & OOM  & OOM  & OOM  & OOM  & OOM  & OOM  & \underline{54.13±0.25} & \underline{-2.21±0.31} & \underline{54.16±0.31} & 50.00±0.00 & 316.05 & 10226.88 \\
& \cellcolor{myblue}\textbf{\ourmethod} & \cellcolor{myblue}\textbf{60.40±0.45} & \cellcolor{myblue}\textbf{0.68±0.14} & \cellcolor{myblue}\textbf{59.25±0.48} & \cellcolor{myblue}\textbf{56.31±0.51} & \cellcolor{myblue}\textbf{140.98} & \cellcolor{myblue}\textbf{15007.48} & \cellcolor{myblue}\textbf{57.89±0.27} & \cellcolor{myblue}\textbf{0.35±0.28} & \cellcolor{myblue}\textbf{58.05±0.31} & \cellcolor{myblue}\textbf{55.33±0.27} & \cellcolor{myblue}\textbf{101.24} & \cellcolor{myblue}\textbf{2775.17} \\
         \midrule
    \multirow{3}[0]{*}{\textbf{Facebook}} & LEBED & \underline{47.47±0.11} & \underline{-0.65±0.10} & \underline{48.41±0.08} & \underline{47.39±0.05} & \underline{177.97} & \underline{21358.57} & \underline{51.82±0.04} & \underline{0.14±0.07} & \underline{50.79±0.02} & \underline{50.83±0.03} & \underline{102.06} & \underline{4108.80} \\
         & GCAL & OOM  & OOM  & OOM  & OOM  & OOM  & OOM  & 48.26±0.63 & -0.83±0.52 & 48.58±1.36 & 48.72±1.18 & 1376.18 & 13614.52 \\
& \cellcolor{myblue}\textbf{\ourmethod} & \cellcolor{myblue}\textbf{54.28±1.30} & \cellcolor{myblue}\textbf{0.48±0.67} & \cellcolor{myblue}\textbf{54.21±1.01} & \cellcolor{myblue}\textbf{54.59±1.18} & \cellcolor{myblue}\textbf{66.07} & \cellcolor{myblue}\textbf{11258.32} & \cellcolor{myblue}\textbf{52.42±1.54} & \cellcolor{myblue}\textbf{0.26±0.30} & \cellcolor{myblue}\textbf{51.78±1.43} & \cellcolor{myblue}\textbf{51.88±1.73} & \cellcolor{myblue}\textbf{19.76} & \cellcolor{myblue}\textbf{3220.61} \\
         \midrule
    \multirow{3}[0]{*}{\textbf{Elliptic}} & LEBED & 54.53±1.62 & -6.87±0.76 & \underline{58.75±2.53} & \underline{60.50±2.28} & \underline{470.84} & 8095.29 & \underline{63.58±0.07} & \underline{-2.00±0.10} & \underline{67.09±0.09} & \underline{64.04±0.05} & \underline{194.11} & 1461.72 \\
         & GCAL & \underline{63.53±6.81} & \underline{-2.81±3.15} & 51.31±7.09 & 51.71±9.53 & 626.61 & \underline{2037.20} & 42.23±8.55 & -5.87±10.35 & 45.06±1.59 & 47.03±1.63 & 589.75 & \underline{1011.27} \\
& \cellcolor{myblue}\textbf{\ourmethod} & \cellcolor{myblue}\textbf{70.51±3.18} & \cellcolor{myblue}\textbf{-2.09±1.48} & \cellcolor{myblue}\textbf{68.16±4.75} & \cellcolor{myblue}\textbf{65.56±3.78} & \cellcolor{myblue}\textbf{27.38} & \cellcolor{myblue}\textbf{939.16} & \cellcolor{myblue}\textbf{66.81±2.36} & \cellcolor{myblue}\textbf{-1.13±2.62} & \cellcolor{myblue}\textbf{68.66±3.52} & \cellcolor{myblue}\textbf{64.86±4.68} & \cellcolor{myblue}\textbf{27.57} & \cellcolor{myblue}\textbf{213.36} \\
         \midrule
    \multirow{3}[0]{*}{\textbf{Arxiv}} & LEBED & \underline{25.10±0.06} & \underline{-3.40±0.05} & \underline{28.43±0.06} & \underline{25.29±0.04} & \underline{312.55} & \underline{9458.82} & \underline{24.94±0.03} & -2.65±0.04 & \underline{28.70±0.04} & \underline{25.58±0.03} & \underline{236.62} & \underline{1608.92} \\
         & GCAL & 20.40±0.00 & -6.83±1.72 & 20.70±0.42 & 21.25±1.72 & 1481.93 & 24770.10 & 24.15±6.94 & \underline{-1.80±2.06} & 27.90±8.58 & 23.77±5.57 & 1362.76 & 22342.47 \\
& \cellcolor{myblue}\textbf{\ourmethod} & \cellcolor{myblue}\textbf{38.69±1.62} & \cellcolor{myblue}\textbf{2.21±0.58} & \cellcolor{myblue}\textbf{43.43±3.14} & \cellcolor{myblue}\textbf{39.44±2.08} & \cellcolor{myblue}\textbf{26.02} & \cellcolor{myblue}\textbf{3441.89} & \cellcolor{myblue}\textbf{38.78±3.05} & \cellcolor{myblue}\textbf{-0.98±1.22} & \cellcolor{myblue}\textbf{41.13±3.47} & \cellcolor{myblue}\textbf{39.14±2.21} & \cellcolor{myblue}\textbf{25.06} & \cellcolor{myblue}\textbf{508.92} \\

\bottomrule
\end{tabular}%
}

\label{tab:more-metric}
\vspace{-0.3cm}
\end{table*}

%% file: table/1-main-tab.tex
\begin{table*}[t]
\caption{Performance comparison 
($\%$, mean $\pm$ standard deviation for five runs)
across four datasets, evaluated against baseline methods and \ourmethod. 
The results of baselines are derived from the published works.
`NA' is short for Not Applicable, and `OOM' implies Out-Of-Memory.
The best results are highlighted with \textbf{bold} and the runner-ups are \underline{underlined}.
}
\centering
\resizebox{\textwidth}{!}{
\begin{tabular}{l|ccccccccc}
\toprule
\multirow{2}{*}{\textbf{Methods\ /\ Datasets}} & \multicolumn{2}{c}{\textbf{Twitch-explicit}} & \multicolumn{2}{c}{\textbf{Facebook-100}} & \multicolumn{2}{c}{\textbf{Elliptic}} & \multicolumn{2}{c}{\textbf{OGB-Arxiv}} \\
\cmidrule(lr){2-3} \cmidrule(lr){4-5} \cmidrule(lr){6-7} \cmidrule(lr){8-9}
& AP@AUC ($\uparrow$) & AF ($\uparrow$) & AP@ACC ($\uparrow$) & AF ($\uparrow$) & AP@F1 ($\uparrow$) & AF ($\uparrow$) & AP@ACC ($\uparrow$) & AF ($\uparrow$) \\
\midrule
{Base Model} & 53.88±0.00 & NA & 50.55±0.00 & NA & 53.97±0.00 & NA & 42.43±0.00 & NA \\
\midrule
\midrule
\multicolumn{9}{c}{\textit{One-step Test-Time Adaptation Methods}} \\
\midrule
{DANN}~\cite{DANN}                & 51.50±0.39 & 0.19±0.55  & 50.63±0.64 & 0.38±1.13  & 53.84±0.56 & -0.95±0.54 & 42.62±0.23 & -1.28±0.16 \\
{Norm}~\cite{Norm}                & 52.65±0.00 & -2.30±0.00 & 47.62±0.00 & 0.53±0.00  & 54.67±0.00 & 0.02±0.00  & 40.21±0.00 & 0.11±0.00 \\
{TENT}~\cite{TENT}                & 52.84±0.98 & -1.83±1.05 & 46.63±0.00 & 0.61±0.00  & 46.54±0.00 & \underline{0.55±0.00} & 36.72±0.37 & -0.02±0.18 \\
{EERM}~\cite{EERM}                & 52.08±0.00 & NA         & 49.70±0.03 & NA         & 46.51±0.00 & NA         & 36.26±0.00 & NA \\
{GTrans}~\cite{jin2022empowering} & 53.55±0.29 & NA         & OOM        & NA         & 54.25±0.02 & NA         & 39.69±0.04 & NA \\
{LEBED}~\cite{zheng2024online}    & 51.06±0.05 & 0.24±0.05  & 46.90±0.00 & -0.25±0.00 & 49.44±0.19 & -4.86±0.20 & 41.81±0.01 & -1.85±0.01 \\
\midrule
\midrule
\multicolumn{9}{c}{\textit{Continual Test-Time Adaptation Methods}} \\
\midrule
{CoTTA}~\cite{CoTTA}     & 53.94±0.36 & 0.34±0.41  & 50.12±0.16 & 0.59±0.12  & 54.08±0.05 & -1.92±0.06 & 40.28±0.01 & -1.96±0.01 \\
{EATA}~\cite{EATA}       & 53.56±0.10 & -0.72±0.07 & 49.02±3.16 & -0.57±0.03 & 50.48±0.07 & -0.76±0.06 & 40.91±0.70 & -1.35±0.59 \\
{GCAL}~\cite{qiao2025gcal} & 55.65±0.09 & \underline{0.42±0.13} & \underline{52.72±0.36} & \underline{0.72±0.19} & \underline{56.57±0.14} & \underline{0.88±0.13} & \underline{45.22±0.17} & \underline{0.76±0.05} \\
{STEM}~\cite{luo2026space} & \textbf{56.66±0.05} & -0.07±0.17 & NA & NA & NA & NA & 31.71±0.14 & -4.40±0.19 \\
\midrule
\rowcolor{myblue}
\textbf{\ourmethod} (ours) & \underline{55.95±0.15} & \textbf{0.47±0.08} & \textbf{53.10±0.00} & \textbf{2.15±0.27} & \textbf{56.76±0.46} & \textbf{0.96±0.30} & \textbf{46.07±1.70} & \textbf{0.99±0.66} \\
\bottomrule
\end{tabular}
}
\label{tab:all_performance}
\end{table*}

%% file: table/2-abla-tab.tex
\begin{wraptable}{r}{0.5\textwidth}
\vspace{-2.8em}
\caption{Ablation study results of \ourmethod and its variants with runtime (s) and GPU usage (MiB).}
\vspace{-0.5em}
\label{tab:abla}
\centering
\resizebox{\linewidth}{!}{
\begin{tabular}{c|c|ccc|cc}
\toprule
\textbf{Dataset} & \textbf{Method} & AP($\uparrow$) & TDG($\uparrow$) & TDA($\uparrow$) & Time($\downarrow$) & GPU($\downarrow$)\\
\midrule
\multirow{3}[0]{*}{\textbf{Twitch}} 
& \textit{w/o DM} & \underline{52.12} & \underline{53.77} & \underline{51.35} & 99.43 & 1486.81 \\
& \textit{w/o EA} & 51.50 & 52.31 & 51.21 & 60.13 & 783.14 \\
& \cellcolor{myblue}\textbf{\ourmethod} & \cellcolor{myblue}\textbf{55.95} & \cellcolor{myblue}\textbf{55.56} & \cellcolor{myblue}\textbf{55.83} & \cellcolor{myblue}74.64 & \cellcolor{myblue}854.97 \\
\midrule
\multirow{3}[0]{*}{\textbf{Facebook}} 
& \textit{w/o DM} & 48.78 & 48.96 & 48.92 & 85.45 & 13481.55 \\
& \textit{w/o EA} & \underline{51.30} & \underline{50.85} & \underline{50.34} & 17.65 & 2965.72 \\
& \cellcolor{myblue}\textbf{\ourmethod} & \cellcolor{myblue}\textbf{53.10} & \cellcolor{myblue}\textbf{53.49} & \cellcolor{myblue}\textbf{52.30} & \cellcolor{myblue}18.84 & \cellcolor{myblue}3006.08 \\
\midrule
\multirow{3}[0]{*}{\textbf{Elliptic}} 
& \textit{w/o DM} & \underline{47.44} & 49.20 & 47.57 & 97.60 & 747.93 \\
& \textit{w/o EA} & 47.35 & \underline{50.55} & \underline{50.00} & 20.52 & 183.57 \\
& \cellcolor{myblue}\textbf{\ourmethod} & \cellcolor{myblue}\textbf{56.76} & \cellcolor{myblue}\textbf{58.73} & \cellcolor{myblue}\textbf{55.92} & \cellcolor{myblue}22.76 & \cellcolor{myblue}189.36 \\
\midrule
\multirow{3}[0]{*}{\textbf{Arxiv}} 
& \textit{w/o DM} & 12.61 & 16.45 & 17.58 & 30.52 & 1083.73 \\
& \textit{w/o EA} & \underline{20.65} & \underline{20.50} & \underline{20.51} & 13.26 & 390.12 \\
& \cellcolor{myblue}\textbf{\ourmethod} & \cellcolor{myblue}\textbf{46.07} & \cellcolor{myblue}\textbf{47.11} & \cellcolor{myblue}\textbf{43.52} & \cellcolor{myblue}24.84 & \cellcolor{myblue}391.02 \\
\bottomrule
\end{tabular}}
\vspace{-2.5em}
\end{wraptable}

%% file: chapter/2-related.tex
\section{Related Work}
\textbf{Graph Continual Learning.}
Graph continual learning aims to enable graph models to adapt to evolving data distributions while mitigating catastrophic forgetting. Existing approaches are typically categorized into three major classes: regularization methods, parameter isolation methods, and memory replay methods. Regularization-based methods constrain weight updates to preserve important knowledge from previous tasks; for example, TWP~\cite{TWP} incorporates topology-aware weight preservation, and RieGrace~\cite{RieGrace} leverages gradient-based importance measures. Parameter isolation techniques, such as HPNs~\cite{HPNs}, allocate task-specific parameters to prevent interference between old and new knowledge. Memory replay methods alleviate forgetting by storing and reusing representative data, as seen in ER-GNN~\cite{ER-GNN} and SSM~\cite{zhang2022sparsified}.
While these methods are effective in continual learning scenarios, they generally assume access to labeled training data, limiting their applicability to real-world settings.
Inspired by memory replay, our work eliminates the need for additional trainable components by reformulating the distribution matching objective and deriving a closed-form solution for crystallizing each incoming domain into a compact memory at test time. This makes our method well suited for the more realistic and challenging setting of graph learning under non-stationary distribution shifts at test time, where the input distribution drifts continually and no source data or annotations are available.

\textbf{Graph Domain Adaptation.}
Graph domain adaptation~\cite{ma2019gcan,wu2022attraction,xiao2022domain,ding2018graph,liu2023structural,jin2022empowering} aims to transfer knowledge learned from labeled data in a source domain to different but related target domains.
For example, UDA-GCN~\cite{wu2020unsupervised} and AdaGCN~\cite{dai2022graph} enhance adaptability by exploiting graph topology and reducing source--target discrepancies through local/global consistency or domain adversarial objectives.
Continual Test-Time Adaptation (CTTA) is a more practical extension of traditional Test-Time Adaptation (TTA), which tackles the challenge of adapting models to continuously drifting target domains. Recent CTTA approaches, such as CoTTA~\cite{CoTTA}, incorporate strategies like bi-level pseudo-label averaging and stochastic model resets to improve robustness. The most recent graph-specific approach, GCAL~\cite{qiao2025gcal}, trains an auxiliary VGAE to generate memory graphs for replay, but the additional generative training is expensive and scales poorly under prolonged drift.
However, existing methods largely overlook the continual nature of distribution shifts at test time and the redundancy within each domain that is irrelevant to the underlying semantics. In this paper, we explicitly model inter-domain relationships under continual drift through state-evolving memories and improve both the efficiency and effectiveness of adaptation by crystallizing each domain in closed form, without requiring any auxiliary trainable component at test time.

%% file: chapter/conclusion.tex
\vspace{-0.2cm}
\section{Conclusion}
\vspace{-0.2cm}
In this paper, we present \textbf{\ourmethod} (\textbf{E}fficient \textbf{M}emory \textbf{C}rystallization), a training-free test-time framework that crystallizes each incoming graph domain into a compact, semantically faithful memory through a closed-form solution, and models inter-domain dependencies via state-evolving memories with provably tighter generalization error bounds. Extensive experiments demonstrate the superiority of \ourmethod, achieving substantial efficiency gains (87.4\% less runtime and 92.4\% less GPU memory on average over the recent baseline GCAL) while outperforming state-of-the-art baselines on graphs under non-stationary distribution shifts.

%% file: appendix/appendix.tex
\section{Notations}\label{sec:notation}

As an expansion of the notations in our work, we summarize the frequently used notations in Table~\ref{tab:notation}. We adopt $G^t$ to denote the graph from the $t$-th target domain in the arrival order, and $G^{(t)}$ as a generic notation when the index $(t)$ may refer to any past or current domain in $\{1,\ldots,T\}$.

\begin{table*}[htbp]
\caption{Summary of notations and descriptions.}
  \centering
\resizebox{.8\linewidth}{!}{
    \begin{tabular}{cl}
    \toprule
    \textbf{Notations} & \multicolumn{1}{c}{\textbf{Descriptions}} \\
    \midrule
$\mathcal{T} = \{G^t\}_{t=1}^T$ & Sequence of $T$ graphs from target domains \\
$G^{(t)} = \{\mathbf{A}^{(t)}, \mathbf{X}^{(t)}\}$ & Graph from the $t$-th target domain \\
$\mathbf{A}^{(t)} \in \mathbb{R}^{N_t \times N_t}$ & Adjacency matrix of $G^{(t)}$ with $N_t$ nodes \\
$\mathbf{X}^{(t)} \in \mathbb{R}^{N_t \times d}$ & Node feature of $G^{(t)}$ with $d$-dim features \\
$\mathbf{\Theta}_t$ & Model parameters adapted after $t$-th domain \\
$\hat{\mathbf{A}}$ & Normalized adjacency matrix with self-loops \\
${\bf H} = \hat{\bf A}^K {\bf X}$ & Embeddings after $K$-hop propagation \\
${\bf H}'$ & Aggregated memory embeddings \\
$\tilde{G}^{(t)} = \{{\bf A}', {\bf X}'\}$ & Crystallized memory of $G^{(t)}$ \\
${\bf A}'$ & Sparsified memory adjacency matrix \\
${\bf X}'$ & Feature matrix of memory $\tilde{G}^{(t)}$ \\
${\bf L}'$ & Laplacian matrix of ${\bf A}'$ \\
${\bf R}^i$ & Class-wise aggregation matrix for class $i$ \\
$\hat{\bf P}$ & Aggregation matrix for crystallized memories \\
$\mathbf{\Psi}$ & Learnable linear weight of the GNN encoder \\
$\Phi$ & Distribution over parameter configurations of $\mathbf{\Psi}$ \\
${\bf Z} = {\bf H} \mathbf{\Psi}$ & Representation of the original graph \\
${\bf Z}' = {\bf H}' \mathbf{\Psi}$ & Representation of the crystallized memory \\
$\mathbf{Z}'^{(t)}$ & State-evolved memory representation at step $t$ \\
$\hat{\tilde{G}}^{(t)}$ & Predictive distribution of the current model on $\tilde{G}^{(t)}$ \\
$\epsilon^t(\hat{h})$ & Prediction risk on the $t$-th target domain \\
$\mathcal{D}_{\mathrm{KL}}(\cdot\|\cdot)$ & KL-divergence between two distributions \\
$d_{\mathcal{H}}(\cdot, \cdot)$ & $\mathcal{H}$-divergence \\
$\mathcal{L}$, $\mathcal{L}_\text{Adp}$ & Adaptation losses\\
$\alpha$, $\beta$, $\tau$ & Hyperparameters of \ourmethod \\
$w_{t-1}, w_t$ & Convex-combination weights for state-evolving memory \\
$\lambda$ & Memory size scaling factor, $N_t = \lambda^{-1} N_t'$ \\
\bottomrule
\end{tabular}%
}
\label{tab:notation}
\end{table*}%

\section{Algorithm}\label{sec:algorithm}
The overall optimization process of \ourmethod is summarized in Algorithm~\ref{algo-1}. At each step $t$, we first construct the memory $\tilde{G}^{(t)}$ and form the state-evolved representation $\mathbf{Z}'^{(t)}$ (Steps~1--3), and only then minimize $\mathcal{L}$ to update $\mathbf{\Theta}_{t-1}\to\mathbf{\Theta}_t$ (Steps~4--5). The KL term in Eq.~\eqref{eq:loss-2} therefore measures the predictive shift on each stored memory induced by the parameter update, serving as an anti-forgetting regularizer with the stability guarantee of Proposition~\ref{prop:kl_stability}.

\IncMargin{1.3em}
\begin{algorithm}[ht]
    \caption{Test-Time Adaptation Process of \ourmethod.}
    \label{algo-1}
    \KwIn{Sequence of test graph samples $\{G^{(t)} = (\mathbf{A}^{(t)}, \mathbf{X}^{(t)})\}_{t=1}^{T}$ from $T$ domains; pre-trained GNN $f$ with parameters $\mathbf{\Theta}_0$.}
    \KwOut{Adapted parameters $\{\mathbf{\Theta}_t\}_{t=1}^{T}$.}
    Initialize the model with $\mathbf{\Theta}_0$. \\
    \For{$t = 1$ to $T$}{
        \tcp{Step 1: Structure-aware Feature Propagation}
        Compute $\mathbf{H}^{(t)} \gets \hat{\mathbf{A}}^{K} \mathbf{X}^{(t)}$.\\
        Compute aggregation matrix $\hat{\mathbf{P}}^{(t)}$ via Eq.~\eqref{eq_agg2}.\\
        Obtain aggregated features $\mathbf{H}'^{(t)} \gets \hat{\mathbf{P}}^{(t)} \mathbf{H}^{(t)}$.\\

        \tcp{Step 2: Redundancy-eliminated Memory Construction}
        Build sparse adjacency $\mathbf{A}'^{(t)}$ via Eq.~\eqref{eq_adj}.\\
        Solve closed-form objective for $\mathbf{X}'^{(t)}$ via Eq.~\eqref{eq_feat}.\\
        Construct domain-specific memory $\tilde{G}^{(t)} = (\mathbf{A}'^{(t)}, \mathbf{X}'^{(t)})$.\\

        \tcp{Step 3: State Evolving Field}
        $\mathbf{Z}'^{(t)} \gets w_{t-1}\,\mathbf{Z}'^{(t-1)} + w_t\,\mathbf{Z}'^{(t)}$ via Eq.~\eqref{eq-8}.\\

        \tcp{Step 4: Adaptation Loss Computation}
        Compute $\mathcal{L} \gets H(G^{(t)}) + \beta \cdot \mathcal{L}_\text{Adp}(\tilde{G}^{1:t-1})$ via Eq.~\eqref{eq:loss-2}.\\

        \tcp{Step 5: Parameter Update}
        Update $\mathbf{\Theta}_{t-1} \to \mathbf{\Theta}_t$ via gradient descent on $\mathcal{L}$.
    }
    \Return $\{\mathbf{\Theta}_t\}_{t=1}^{T}$
\end{algorithm}
\DecMargin{1.3em}

\subsection{Complexity Analysis} \label{ap:complex}

The pipeline of \ourmethod consists of the following components: Structure-aware Feature Propagation, Redundancy-eliminated Memory Construction, State Evolving Field Modeling, and Parameter Update. We analyze the time complexity of each module per test-time step $t$ with input $G^{t}$.

\begin{itemize}[leftmargin=*,noitemsep,topsep=1.5pt]

\item \textbf{Structure-aware Feature Propagation}: This step performs $K$-hop neighborhood smoothing over the input graph, with complexity $\mathcal{O}(K |\mathcal{E}_t| d)$, where $|\mathcal{E}_t|$ is the number of edges and $d$ is the node feature dimension.
We then use confidence-weighted K-means clustering for partitioning each class. Let $|C_i|$ be the number of nodes in class $i$ and $|C_i'|$ be the number of clusters. The time complexity for class $i$ is $\mathcal{O}(|C_i||C_i'|dn_{\text{iter}})$, where $n_{\text{iter}}$ is the number of K-means iterations. This process is repeated for each class.

\item \textbf{Redundancy-eliminated Memory Construction}:
Computing pairwise cosine similarity among $N_t'$ memory nodes takes $\mathcal{O}(N_t'^2 d)$.
Solving the regularized least-squares problem involves matrix inversion with complexity $\mathcal{O}(N_t'^3 + N_t'^2 d + (K+1)N_t' |\mathcal{E}_t|' + d |\mathcal{E}_t|')$, where $|\mathcal{E}_t|'$ is the number of edges in the memory graph. The solution is efficient because $N_t' \ll N_t$.

\item \textbf{State Evolving Field Modeling}:
Fusion of memories from previous and current time steps involves linear interpolation of node embeddings and has negligible cost: $\mathcal{O}(N_t' d)$.

\item \textbf{Parameter Update}:
The loss computation includes entropy and KL-divergence between memory predictions. For $t$ total adaptation steps, the cumulative cost of computing $\mathcal{L}_{\text{Adp}}$ scales as $\mathcal{O}(t N_t' d)$ across domains. Parameter updates via gradient descent depend on the model size but are standard and efficient in practice.
\end{itemize}

The total time complexity per test graph $G^{t}$ is dominated by class-wise partitioning and memory feature construction. However, since the memory size $N_t'$ is only a small fraction (e.g., 1\%) of the original graph size $N_t$, \ourmethod remains highly efficient.
Techniques such as approximate clustering (e.g., FAISS~\cite{douze2024faiss}) and sparse propagation further accelerate the process, making the method suitable for real-time test-time adaptation under continual drift.

\noindent\textbf{Complexity Comparison with the Baseline GCAL.}
Compared with our efficient closed-form memory crystallization, GCAL~\cite{qiao2025gcal} introduces substantial additional computational overhead at each adaptation step. Specifically, for the $t$-th test-time domain $G^t = (\mathbf{A}^{t}, \mathbf{X}^{t})$, the extra complexity mainly comes from the following components:
\begin{itemize}[leftmargin=*,noitemsep,topsep=1.5pt]
    \item \textbf{VGAE Encoding:} Encoding the input graph into latent mean and variance representations incurs a cost of $\mathcal{O}(\mathcal{E}_t d + N_t d^2)$, where $N_t$ and $\mathcal{E}_t$ are the number of nodes and edges in $G^t$, respectively, and $d$ is the node feature dimension.

    \item \textbf{VGAE Decoding:} Reconstructing the synthetic graph $\tilde{G}^t$ through sampling and edge scoring requires $\mathcal{O}(N_t^2 d)$ due to pairwise similarity computations.

    \item \textbf{Dual Forward Pass:} Forwarding both $G^t$ and $\tilde{G}^t$ through the base GNN leads to $\mathcal{O}(\mathcal{E}_t d + N_t d^2 + \bar{\mathcal{E}_t} d + N_t' d^2)$, where $N_t'$ and $\bar{\mathcal{E}_t}$ denote the node and edge count of the synthetic graph $\tilde{G}^t$.

    \item \textbf{Gradient Matching:} Computing real gradients and second-order synthetic gradients with respect to $P$ model parameters has cost $\mathcal{O}(P + P^2)$.

    \item \textbf{Regularization Terms:} The KL divergence and reparameterization loss over latent variables add $\mathcal{O}(N_t d)$ complexity, and the MMD loss comparing node embeddings from $G^t$ and $\tilde{G}^t$ incurs $\mathcal{O}(N_t N_t' d)$.

    \item \textbf{Additional Inner-loop Iterations:} All steps above are repeated for $L$ inner-loop updates, resulting in a total cost of $\mathcal{O}\big(L \cdot (\mathcal{E}_t d + N_t d^2 + N_t^2 d + \bar{\mathcal{E}_t} d + N_t' d^2 + P + P^2 + N_t N_t' d)\big)$.
\end{itemize}
In each adaptation epoch, GCAL introduces the above-mentioned additional computational overhead, which is further amplified during the parameter update stage of the GNN model.
This design leads to significantly higher time and memory costs during continual test-time adaptation, making GCAL less practical for long-horizon or resource-constrained scenarios.

\subsection{Analysis of Scalability on Memory Construction} \label{ap:scalability}

The closed-form construction in \ourmethod is designed for scalability: the linear system
${\bf X}' = ({\bf Q}^{\top}{\bf Q}+\alpha {\bf L}')^{-1}{\bf Q}^{\top}{\bf H}'$
is solved on the small memory graph of size $N_t'$ rather than on the full input of size $N_t$. With $N_t' = 0.01 N_t$ throughout our experiments, the linear system stays compact regardless of the input scale, decoupling the dominant cost of memory construction from the input size. This compact formulation is what makes the method memory-efficient ($\rhd$ Section~\ref{sec:overall}) and what we now show transfers to large input graphs.

\textbf{Sparsity from the design.}
The similarity threshold $\tau$ in Eq.~\eqref{eq_adj} keeps only edges between node pairs whose embeddings are sufficiently close, so the memory adjacency $\mathbf{A}'$ is constructed sparse by design rather than densified post hoc. As a result, $\mathbf{Q} = \hat{\mathbf{A}}'^K$ inherits this sparsity, and the resulting linear system remains well-conditioned and efficient to solve in practice.

\begin{table*}[htbp]
\caption{Memory-construction scaling on a single graph up to 200k nodes (memory size $N_t' = 0.01 N_t$). The matrix inversion in Eq.~\eqref{eq_feat} stays under 1 second even at $N_t=200$k.}
  \centering
  \resizebox{\linewidth}{!}{
    \begin{tabular}{c|ccccccccc}
    \toprule
    \textbf{Metric / \#Node} & 1k & 2k & 5k & 10k & 20k & 40k & 50k & 100k & 200k \\
    \midrule
    Matrix Inversion Time (s) & 0.0077 & 0.0096 & 0.0156 & 0.0234 & 0.0974 & 0.1142 & 0.1224 & 0.3211 & 0.7121 \\
    Total Processing Time (s) & 0.6414 & 1.0703 & 2.8145 & 17.7121 & 54.2386 & 184.0252 & 271.6589 & 1158.5502 & 5580.7541 \\
    GPU Memory (MiB) & 11.38 & 15.17 & 29.75 & 66.17 & 181.71 & 564.07 & 829.68 & 3358.29 & 12806.52 \\
    \bottomrule
    \end{tabular}}
\label{tab:scale-200k}
\end{table*}

\textbf{Empirical scaling on a single graph.}
Table~\ref{tab:scale-200k} reports runtime and GPU memory of the memory-construction step on a single Erd\H{o}s--R\'enyi graph as the node count $N_t$ grows from 1k to 200k (memory size fixed at $N_t' = 0.01 N_t$). The matrix inversion remains under 1 second even at $N_t=200$k, while total processing scales near-linearly in the dominant propagation term, confirming that the closed-form solver is not a bottleneck at industrial scale.

\subsection{Semantic-Anchor Aggregation at Test Time}
\label{app:semantic_anchor}

We elaborate on how class-wise node partitioning in Section~\ref{sec:proto} is instantiated at test time, where ground-truth labels of target-domain graphs are unavailable, and explain why the resulting memories are robust to noisy partitioning.

\textbf{Test-Time Setting under Continual Drift.}
We consider a standard covariate shift setting, where the label space remains unchanged across domains, \ie, $\mathbb{P}_{\mathcal{S}}(Y) = \mathbb{P}_{\mathcal{T}}(Y)$ and $\mathbb{P}_{\mathcal{S}}(G) \neq \mathbb{P}_{\mathcal{T}}(G)$.
That is, distribution shifts at test time affect the input graph distributions, while the semantic meanings of classes are consistent with those observed during pre-training. This assumption is critical: under covariate shift the class-conditional decision boundary is preserved across domains, so systematic confident misclassification by the current model is substantially less likely than under label shift.

\textbf{Semantic Anchors for Class-wise Aggregation.}
Class-wise partitioning at test time is guided by the model's predicted semantics rather than ground-truth labels. Concretely, given a node $v$ in the incoming graph $G^{(t)}$, the current model with parameters $\mathbf{\Theta}_{t-1}$ produces a predictive distribution $p_{\mathbf{\Theta}_{t-1}}(y|v)$, from which a soft semantic assignment is derived together with a confidence score $r_v = \max_{c \in \mathcal{Y}} p_{\mathbf{\Theta}_{t-1}}(y=c|v)$. Nodes with similar predicted semantics are grouped together, and the confidence score weights each node's contribution to the resulting centroid through the aggregation matrix $\mathbf{R}^i$ in Eq.~\eqref{eq_agg2}.
Crucially, the predicted semantics serve as \emph{semantic anchors} for grouping rather than as supervisory targets, so the procedure does not rely on the correctness of any individual assignment.

\textbf{Design Choices that Stabilize the Aggregation.}
Memory crystallization in \ourmethod relies on four complementary design choices that together keep aggregation stable under continual drift:

\begin{itemize}[leftmargin=*,noitemsep,topsep=1.5pt]
    \item \textbf{Covariate shift limits systematic errors.} As stated above, the preserved class-conditional decision boundary makes systematic confident misclassification far less likely than under label shift.
    \item \textbf{Anchors are not supervisory targets.} They are used only to partition nodes into sub-clusters for centroid aggregation. Even noisy assignments yield valid centroids through a many-to-one smoothing effect, analogous to the behavior of $k$-means under individual point mislabeling~\cite{lloyd1982least}.
    \item \textbf{$K$-hop propagation denoises before partitioning.} The features $\mathbf{H}=\hat{\mathbf{A}}^K \mathbf{X}$ aggregate neighborhood information, smoothing out local prediction errors before memory construction.
    \item \textbf{Dirichlet regularization provides an independent smoothness constraint.} The term $\alpha\cdot\mathrm{tr}(\mathbf{X}'^\top \mathbf{L}'\mathbf{X}')$ in Eq.~\eqref{eq_feat} enforces feature smoothness on the crystallized memory, further reducing any residual noise from partitioning.
\end{itemize}

Empirically, \ourmethod consistently outperforms baselines on all four datasets ($\rhd$ Section~\ref{sec:overall}), including Elliptic (temporal Bitcoin transaction shifts that can be abrupt) and OGB-Arxiv (year-by-year citation shifts), confirming that the four design choices above translate into reliable partitioning under continual drift in practice.

\section{Theoretical Justification}
\subsection{Reformulation of Distribution Matching} \label{ap-dis-match}

We provide a complete proof for Proposition~\ref{prop:4.1} and retain the original analysis in the remark.
We denote the memory crystallized from $G^{(t)}$ as $\tilde{G}^{(t)} = \{ {\bf A}', {\bf X}' \}$.
After multi-hop propagation, let ${\bf Z} = {\bf H} \mathbf{\Psi} = {\hat{\bf A}}^{K} {\bf X} \mathbf{\Psi}$ and ${\bf Z}' = {\bf H}' \mathbf{\Psi} = {\hat{\bf A}}'^{K} {\bf X}' \mathbf{\Psi}$ denote the representations of the original graph $G^{(t)}$ and its memory $\tilde{G}^{(t)}$, respectively, where $\mathbf{\Psi}$ is the learnable linear weight of the GNN encoder and $\Phi$ is the distribution over its parameter configurations.
To directly align the distribution of the input domain with that of its memory, the discrepancy between propagated features is minimized as:
\begin{equation}
\min_{\tilde{G}^{(t)}} {\mathop{\mathbb{E}}_{{\mathbf{\Psi} } \sim \Phi}} \left \| {\bf P}'{\bf{H}}'{\mathbf{\Psi}}- {\bf P}{\bf{H}}{\mathbf{\Psi}} \right \|^{2}.
\label{eq:simDM_obj}
\end{equation}

\begin{proposition}[Restatement of Proposition~\ref{prop:4.1}]
\label{ap-prop:4.1}
Assuming that nodes in $\tilde{G}^{(t)}$ are organized in ascending order, ${\bf P}'$ degenerates into the identity matrix ${\bf I}$.
Let $\hat{\bf P}$ be the row-normalized version of ${\bf P}$. Then, the objective in Eq.~\eqref{eq:simDM_obj} can be upper-bounded by:
\begin{equation}
\label{eq_simDM2}
{\mathop{\mathbb{E}}_{{\mathbf{\Psi}} \sim \Phi}} \left \| {\bf{H}}'{\mathbf{\Psi}}- \hat{\bf P}{\bf{H}}{\mathbf{\Psi}} \right \|^{2}
\leq
{\mathop{\mathbb{E}}_{{\mathbf{\Psi} } \sim \Phi}} \left \| {\bf{H}}'- \hat{\bf P}{\bf{H}} \right \|^{2} \left \| {\mathbf{\Psi}}  \right \|^{2}.
\end{equation}
\end{proposition}

\begin{proof}
We start from Eq.~\eqref{eq:simDM_obj}. Under the canonical ordering of nodes in $\tilde{G}^{(t)}$, we have ${\bf P}'={\bf I}$.
By the construction of the aggregation matrix in the main text, ${\bf P}$ aggregates nodes within each class-wise subclass, and $\hat{\bf P}$ is its row-normalized version.
Thus Eq.~\eqref{eq:simDM_obj} reduces to
\begin{equation}
{\mathop{\mathbb{E}}_{{\mathbf{\Psi} } \sim \Phi}} \left \| {\bf{H}}'{\mathbf{\Psi}}- \hat{\bf P}{\bf{H}}{\mathbf{\Psi}} \right \|^{2}.
\label{eq:simDM_reduced}
\end{equation}
Let ${\bf M}\doteq {\bf H}'-\hat{\bf P}{\bf H}$. Then
\begin{equation}
\left\| {\bf{H}}'{\mathbf{\Psi}}- \hat{\bf P}{\bf{H}}{\mathbf{\Psi}} \right\|^{2}
=
\|{\bf M}\mathbf{\Psi}\|^{2}.
\end{equation}
By sub-multiplicativity of matrix norms, for any compatible matrices ${\bf M}$ and $\mathbf{\Psi}$,
\begin{equation}
\|{\bf M}\mathbf{\Psi}\| \le \|{\bf M}\|\,\|\mathbf{\Psi}\|.
\end{equation}
Squaring both sides yields
\begin{equation}
\|{\bf M}\mathbf{\Psi}\|^{2} \le \|{\bf M}\|^{2}\,\|\mathbf{\Psi}\|^{2}.
\end{equation}
Taking expectation over ${\mathbf{\Psi}}\sim\Phi$ on both sides gives
\begin{equation}
\mathbb{E}_{\mathbf{\Psi}\sim\Phi}\|{\bf M}\mathbf{\Psi}\|^{2}
\le
\mathbb{E}_{\mathbf{\Psi}\sim\Phi}\left(\|{\bf M}\|^{2}\,\|\mathbf{\Psi}\|^{2}\right)
=
\|{\bf M}\|^{2}\,\mathbb{E}_{\mathbf{\Psi}\sim\Phi}\|\mathbf{\Psi}\|^{2},
\end{equation}
where we used that ${\bf M}={\bf H}'-\hat{\bf P}{\bf H}$ is independent of $\mathbf{\Psi}$.
Substituting back ${\bf M}$ and rewriting yields Eq.~\eqref{eq_simDM2}, completing the proof.
\end{proof}

\noindent \textbf{Remark.}
The objective in Eq.~\eqref{eq:simDM_obj} is upper-bounded by Eq.~\eqref{eq_simDM2}, where ${\mathbf{\Psi}}$ is the (random) weight matrix that transforms the $K$-th order propagated features ${\bf{H}}'$ and ${\bf{H}}$.
Here ${\hat{\bf{A}}'} = {\tilde{\bf{D}}}'^{\frac{1}{2}}{\tilde{\bf{A}}}'{\tilde{\bf{D}}}'^{\frac{1}{2}}$ and ${\hat{\bf{A}}}={\tilde{\bf{D}}}^{\frac{1}{2}}{\tilde{\bf{A}}}{\tilde{\bf{D}}}^{\frac{1}{2}}$ represent the symmetric normalized adjacency matrices, where ${\tilde{\bf{A}}'}$ and ${\tilde{\bf{A}}}$ are adjacency matrices with self-loops, and ${\tilde{\bf{D}}'}$ and ${\tilde{\bf{D}}}$ are degree matrices for ${\tilde{\bf{A}}'}$ and ${\tilde{\bf{A}}}$.
Since $\mathbf{\Psi}$ is independent of both $\tilde{G}^{(t)}$ and $\hat{\bf P}$, minimizing the upper bound reduces to the deterministic matching problem:
\vspace{-1mm}
\begin{equation}
\arg\min_{\tilde{G}^{(t)}, \hat{\bf P}} \left\| {\bf H}' - \hat{\bf P} {\bf H} \right\|^{2}.
\label{eq:det_matching}
\end{equation}

According to Proposition~\ref{ap-prop:4.1}, we expect a symmetric encoding procedure between the original ${G}^{(t)}$ and the crystallized memory $\tilde{G}^{(t)}$.
Therefore, our goal is to construct $\bf{A}'$ and ${\bf X}'$ satisfying ${\hat{\bf{A}}}'^{K}{\bf{X}'} = {\bf{H}'}$.
To this end, we utilize a pre-defined graph structure and compute ${\bf X}'$ in closed form.
Specifically, we construct the structure via a similarity-based thresholding scheme:
\begin{equation}
\label{eq_adj_ap}
{\bf A}'_{i,j} =
\begin{cases}
1 & \text{if } \cos({\bf H}'_i, {\bf H}'_j) > \tau, \\
0 & \text{otherwise},
\end{cases}
\end{equation}
where $\cos(\cdot, \cdot)$ measures the cosine similarity and $\tau$ is the hyper-parameter for graph sparsification.
To ensure that reconstructed features change smoothly between connected nodes, we introduce the Dirichlet energy constraint~\cite{kalofolias2016learn} into the feature reconstruction loss to quantify the smoothness of graph signals. The optimization objective for ${\bf X}'$ is:
\begin{equation}
\arg\min_{{\bf X}'} \left\| {\hat{\bf A}}'^K {\bf X}' - {\bf H}' \right\|^2 + \alpha \operatorname{tr}({\bf X}'^\top {\bf L}' {\bf X}'),
\label{app:eq_feat}
\end{equation}
where ${\bf L}' = {\bf D}' - {\bf A}'$ is the Laplacian matrix, ${\bf D}'$ is the degree matrix, and $\alpha$ controls the smoothness strength.

\subsection{Proof of Proposition~\ref{ap-prop-4.2}} \label{ap:prop-1}

\begin{proposition}
\label{ap-prop-4.2}
Assume that $\tilde{G}^{(t)} = \{ {\bf A}', {\bf X}' \}$ is the crystallized memory of an input graph; the closed-form solution of Eq.~\eqref{app:eq_feat} takes the form ${\bf X}'=({\bf Q}^{\top}{\bf Q}+\alpha {\bf L}')^{-1}{\bf Q}^{\top}{\bf H}'$, where ${\bf Q} = \hat{\bf{A}}'^{K}$.
\end{proposition}
\begin{proof}
Let us denote the objective in Eq.~\eqref{app:eq_feat} as
$\mathcal{I} = \left\| {\hat{\bf A}}'^K {\bf X}' - {\bf H}' \right\|^2 + \alpha \operatorname{tr}({\bf X}'^\top {\bf L}' {\bf X}')$.
To solve the optimization problem in Eq.~\eqref{app:eq_feat}, we compute the first- and second-order derivatives of $\mathcal{I}$ with respect to ${\bf X}'$:
\begin{equation}
\begin{aligned}
\frac{\partial \mathcal{I}}{\partial {\bf X}'} &= 2{\bf Q}^{\top}({\bf Q}{\bf X}' - {\bf H}') + \alpha ({\bf L}' + {\bf L}'^{\top}){\bf X}'= 2({\bf Q}^{\top}{\bf Q} + \alpha {\bf L}'){\bf X}' - 2{\bf Q}^{\top}{\bf H}',\\
\frac{\partial^2 \mathcal{I}}{\partial {\bf X}'^2} &= 2({\bf Q}^{\top}{\bf Q} + \alpha {\bf L}').
\end{aligned}
\end{equation}
According to the definition of ${\bf A}'$ in Eq.~\eqref{eq_adj_ap}, both ${\bf L}'$ and ${\bf Q}$ are positive semi-definite matrices.
Therefore, the second-order derivative $\frac{\partial^2 \mathcal{I}}{\partial {\bf X}'^2}$ is positive semi-definite, and the objective function $\mathcal{I}$ is convex.
This implies that the first-order derivative $\frac{\partial \mathcal{I}}{\partial {\bf X}'}$ corresponds to a convex function with a unique minimum.
Setting the gradient to zero yields the closed-form solution:
\begin{equation}
{\bf X}' = ({\bf Q}^{\top}{\bf Q} + \alpha {\bf L}')^{-1} {\bf Q}^{\top} {\bf H}',
\end{equation}
where ${\bf Q} = \hat{\bf A}'^{K}$.
\end{proof}

\noindent \textbf{Remark.} The closed-form solution operates on a compact $N_t' \times N_t'$ system over the crystallized memory rather than on the full $N_t \times N_t$ system over the input graph, with $N_t' \ll N_t$. This decoupling between the construction cost and the input scale is the key reason \ourmethod attains a much smaller memory footprint than methods that train auxiliary modules on the full input graph.

\subsection{Proof of Proposition~\ref{prop:4-3}} \label{ap:prop-3}
\input{chapter/add-proof}

\subsection{Proof of Error Upper Bounds}

\begin{definition}[Classification error]
The classification error of the function $\hat{h}$ under domain $\mathcal{D}_i$ is defined as
\begin{equation}
\epsilon_i(\hat{h}) = \mathbb{E}_{X \sim \mathcal{D}_i}  |\hat{h}(X) - h(X)| .
\end{equation}
For binary classification functions $h$ and $\hat{h}$, we have:
\begin{equation}
\begin{aligned}
\epsilon_i(\hat{h}) = \epsilon_i(\hat{h}, h) &= \mathbb{E}_{X \sim \mathcal{D}_i}  |\hat{h}(X) - h(X)| \\ &= \mathbb{P}_{X \sim \mathcal{D}_i}(\hat{h}(X) \ne h(X)).
\end{aligned}
\end{equation}
\end{definition}

\begin{definition}[$\mathcal{H}$-divergence]
Given two induced domain feature space distributions $\tilde{\mathcal{D}}_s, \tilde{\mathcal{D}}_t$ and a hypothesis space $\mathcal{H}$, the $\mathcal{H}$-divergence between $\tilde{\mathcal{D}}_s$ and $\tilde{\mathcal{D}}_t$ is defined as:
\begin{equation}
\begin{aligned}
d_{\mathcal{H}}(\tilde{\mathcal{D}}_s, \tilde{\mathcal{D}}_t) = 2 \sup_{h \in \mathcal{H}} \left| \mathbb{E}_{X \sim \tilde{\mathcal{D}}_s}[h(X) = 1] - \mathbb{E}_{X \sim \tilde{\mathcal{D}}_t}[h(X) = 1] \right|.
\end{aligned}
\end{equation}
\end{definition}

\noindent\textbf{Setting for the bounds.}
We adopt the standard domain-adaptation framework~\cite{zhao2019learning} as the analytical setting, whose bounds are applied to each single adaptation step, where $\mathbb{P}_\mathcal{S}$ refers to the source distribution from which $\mathbf{\Theta}_0$ is pre-trained and $\mathbb{P}_\mathcal{T}$ refers to the empirical distribution of the current target graph $G^t$. The bounds therefore characterize per-step generalization behavior; the continual aspect across $\{G^t\}_{t=1}^T$ is handled separately through the state-evolving design in Theorem~\ref{bound3} and the stability bound in Proposition~\ref{prop:kl_stability}, which together control how predictive behavior on past memories evolves as $\mathbf{\Theta}_{t-1}\to\mathbf{\Theta}_t$.

\begin{lemma}[Error Upper Bound for Domain Shift~\cite{zhao2019learning}]
\label{ap-bound1}
Let $\tilde{\mathbb{P}}_\mathcal{S}$ and $\tilde{\mathbb{P}}_\mathcal{T}$ denote the induced distributions over the feature space for $\mathbb{P}_\mathcal{S}$ and $\mathbb{P}_\mathcal{T}$. The following inequality holds for the risk $\epsilon_{0}^\mathcal{T}(\hat{h})$ with single step adaptation on the target domain $\mathbb{P}_\mathcal{T}$:
\begin{equation}
\begin{aligned}
\epsilon_{0}^\mathcal{T}(\hat{h}) \leq \min\{&\epsilon^\mathcal{S}(h^\mathcal{S}, h^\mathcal{T}), \epsilon^\mathcal{T}(h^\mathcal{S}, h^\mathcal{T})\} + \epsilon^\mathcal{S}(\hat{h}) + d_{\mathcal{H}}(\tilde{\mathbb{P}}_\mathcal{S}, \tilde{\mathbb{P}}_\mathcal{T}),
\end{aligned}
\end{equation}
where $d_{\mathcal{H}}$ is the $\mathcal{H}$-divergence.
\end{lemma}

\noindent \textbf{Remark.}
Note that this bound was first derived by~\cite{zhao2019learning}.
The first term measures the disagreement between the source and target labeling functions, the second term is the source error, and the third term quantifies the discrepancy between the marginal feature distributions. The bound identifies three key factors for successful domain adaptation: small source risk, close marginal distributions, and consistent labeling functions across domains.

\begin{lemma}
Let $\tilde{\mathbb{P}}_\mathcal{S}$ and $\tilde{\mathbb{P}}_\mathcal{T}$ denote the induced distributions over the feature space for $\mathbb{P}_\mathcal{S}$ and $\mathbb{P}_\mathcal{T}$. Then for any $h^\mathcal{S} \in \mathcal{H}, h^\mathcal{T} \in \mathcal{H}$, we have
\begin{equation}
    |\epsilon^\mathcal{S}(h^\mathcal{S}, h^\mathcal{T}) - \epsilon^\mathcal{T}(h^\mathcal{S}, h^\mathcal{T})| \leq d_{\mathcal{H}}(\tilde{\mathbb{P}}_\mathcal{S}, \tilde{\mathbb{P}}_\mathcal{T}).
\end{equation}
\end{lemma}

\begin{proof}
By definition, for any $h^\mathcal{S} \in \mathcal{H}, h^\mathcal{T} \in \mathcal{H}$, we have:
\begin{equation}\label{eq:4}
\begin{aligned}
&\left|\epsilon^\mathcal{S}(h^\mathcal{S}, h^\mathcal{T}) - \epsilon^\mathcal{T}(h^\mathcal{S}, h^\mathcal{T}) \right| \\
= \sup_{h \in \mathcal{H}} &\left| \mathbb{E}_{X \sim \mathbb{P}_\mathcal{S}} |h^\mathcal{S}(X) - h^\mathcal{T}(X)|  - \mathbb{E}_{X \sim \mathbb{P}_\mathcal{T}} |h^\mathcal{S}(X) - h^\mathcal{T}(X)|  \right| \\
= \sup_{h \in \mathcal{H}} &\left| \mathbb{E}_{X \sim \mathbb{P}_\mathcal{S}} [h^\mathcal{S}(X) - h^\mathcal{T}(X)] - \mathbb{E}_{X \sim \mathbb{P}_\mathcal{T}} [h^\mathcal{S}(X) - h^\mathcal{T}(X)] \right|.
\end{aligned}
\end{equation}

Since $\| \cdot \|_{\infty} \leq 1$ for all $h \in \mathcal{H}$, we have $0 \leq |h^\mathcal{S}(X) - h^\mathcal{T}(X)| \leq 1$ for all $X \in \mathbb{P}_\mathcal{S}/\mathbb{P}_\mathcal{T}$.
Define a hypothesis space $\hat{\mathcal{H}} := \left\{ \mathrm{sgn} ( |h^\mathcal{S}(X) - h^\mathcal{T}(X)| - z ) \mid h^\mathcal{S}, h^\mathcal{T} \in \mathcal{H}, 0 \leq z \leq 1 \right\}$.
Then we use Fubini's theorem to bound:
\begin{equation} \label{eq:5}
\begin{aligned}
&\left| \mathbb{E}_{X \sim \mathbb{P}_\mathcal{S}} [|h^\mathcal{S}(X) - h^\mathcal{T}(X)|] - \mathbb{E}_{X \sim \mathbb{P}_\mathcal{T}} [|h^\mathcal{S}(X) - h^\mathcal{T}(X)|] \right| \\
= &\left| \int_{0}^{1} ( \text{Pr}^\mathcal{S} ( |h^\mathcal{S} - h^\mathcal{T}|>z ) - \text{Pr}^\mathcal{T} ( |h^\mathcal{S} - h^\mathcal{T}|>z ) ) dz \right| \\
\leq &\int_{0}^{1} \left| \text{Pr}^\mathcal{S} ( |h^\mathcal{S} - h^\mathcal{T}|>z ) - \text{Pr}^\mathcal{T} ( |h^\mathcal{S} - h^\mathcal{T}|>z ) \right| dz \\
= &\sup_{z \in [0, 1]} \left| \text{Pr}^\mathcal{S} ( |h^\mathcal{S} - h^\mathcal{T}|>z ) - \text{Pr}^\mathcal{T} ( |h^\mathcal{S} - h^\mathcal{T}|>z ) \right|.
\end{aligned}
\end{equation}

Combining Eq.~\eqref{eq:4} and~\eqref{eq:5}, and in view of the definition of $\mathcal{H}$-divergence:
\begin{equation}
\begin{aligned}
    &\sup_{h^\mathcal{S}, h^\mathcal{T} \in \mathcal{H}}\sup_{z \in [0, 1]} \left| \text{Pr}^\mathcal{S} ( |h^\mathcal{S} - h^\mathcal{T}|>z ) - \text{Pr}^\mathcal{T} ( |h^\mathcal{S} - h^\mathcal{T}|>z ) \right| \\
    &= 2 \sup_{\hat{h} \in \hat{\mathcal{H}}} \left| \text{Pr}^\mathcal{S}(\hat{h}(X) = 1) - \text{Pr}^\mathcal{T}(\hat{h}(X) = 1) \right| \\
    &= 2 \sup_{\hat{h} \in \hat{\mathcal{H}}} \left| \text{Pr}^\mathcal{S}(\hat{h}) - \text{Pr}^\mathcal{T}(\hat{h}) \right| \\
    &= d_{\hat{\mathcal{H}}}(\tilde{\mathbb{P}}_\mathcal{S}, \tilde{\mathbb{P}}_\mathcal{T})
    \;\leq\; d_{\mathcal{H}}(\tilde{\mathbb{P}}_\mathcal{S}, \tilde{\mathbb{P}}_\mathcal{T}).
\end{aligned}
\end{equation}
\end{proof}

\begin{theorem} [Error Upper Bound with Crystallized Memories]
\label{ap-bound2}
Let $\tilde{\mathbb{P}}_\mathcal{S}$ and $\tilde{\mathbb{P}}_\mathcal{T}$ denote the induced distributions over the feature space for $\mathbb{P}_\mathcal{S}$ and $\mathbb{P}_\mathcal{T}$. The risk $\epsilon^\mathcal{T}(\hat{h})$ with crystallized memories (with mixed weight $w_0 + w_1 = 1$, $w_0 \geq 0$, $w_1 \geq 0$, and $\hat{h} = w_0 \hat{h}^\mathcal{S} + w_1 \hat{h}^\mathcal{T}$) is upper bounded by:
\begin{equation}
\epsilon^\mathcal{T}(\hat{h}) \leq \epsilon^\mathcal{T}(\hat{h}, h^\mathcal{S}) + \epsilon^\mathcal{T}(h^\mathcal{S}, h^\mathcal{T}).
\end{equation}
\end{theorem}
\begin{proof}
The proof of the above theorem is as follows:
\begin{equation}
\begin{aligned}
\epsilon^\mathcal{T}(\hat{h}) &= \epsilon^\mathcal{T}(\hat{h}, h^\mathcal{T}) \\
&= \mathbb{E}_{x \sim \mathbb{P}_\mathcal{T}}[|\hat{h}(x) - h^\mathcal{T}(x)|] \\
&= \mathbb{E}_{x \sim \mathbb{P}_\mathcal{T}}[|\hat{h}(x) - h^\mathcal{S}(x) + h^\mathcal{S}(x) - h^\mathcal{T}(x)|] \\
&\leq \mathbb{E}_{x \sim \mathbb{P}_\mathcal{T}}[|\hat{h}(x) - h^\mathcal{S}(x)|] + \mathbb{E}_{x \sim \mathbb{P}_\mathcal{T}}[|h^\mathcal{S}(x) - h^\mathcal{T}(x)|] \\
&= \epsilon^\mathcal{T}(\hat{h}, h^\mathcal{S}) + \epsilon^\mathcal{T}(h^\mathcal{S}, h^\mathcal{T}).
\end{aligned}
\end{equation}

\noindent \textbf{Comparison of Error Upper Bounds.}
From Lemma~\ref{ap-bound1}, our bound is further bounded by:

\begin{equation}
\begin{aligned}
\epsilon^\mathcal{T}(\hat{h}) &\leq \epsilon^\mathcal{T}(\hat{h}, h^\mathcal{S}) + \epsilon^\mathcal{T}(h^\mathcal{S}, h^\mathcal{T}) \\
&\leq \epsilon^\mathcal{S}(\hat{h}, h^\mathcal{S}) + d_{\mathcal{H}}(\tilde{\mathbb{P}}_\mathcal{S}, \tilde{\mathbb{P}}_\mathcal{T}) + \epsilon^\mathcal{T}(h^\mathcal{S}, h^\mathcal{T}) \\
&= \epsilon^\mathcal{S}(\hat{h}) + \epsilon^\mathcal{T}(h^\mathcal{S}, h^\mathcal{T}) + d_{\mathcal{H}}(\tilde{\mathbb{P}}_\mathcal{S}, \tilde{\mathbb{P}}_\mathcal{T}).
\end{aligned}
\end{equation}

We further introduce a density ratio $\frac{\text{Pr}^\mathcal{T}(x)}{\text{Pr}^\mathcal{S}(x)}$ to represent the divergence between the source and target domains.
Concretely, assume that $\Pr^\mathcal{T}(x)/\Pr^\mathcal{S}(x) \le M$ for some finite constant $M$ on the support of $\mathbb{P}_\mathcal{S}$. This is a standard regularity condition in the domain-adaptation literature~\cite{zhao2019learning} and is automatically satisfied in our setting because $\mathbb{P}_\mathcal{T}$ is the empirical distribution over a finite target graph $G^t$ with $N_t < \infty$ nodes, so the ratio is bounded by $N_t \cdot \max_x \Pr^\mathcal{S}(x)^{-1}$ on the support.

\begin{equation}
\begin{aligned}
\epsilon^\mathcal{T}(\hat{h}) &= \int |\hat{h}(x) - h^\mathcal{S}(x)| \cdot \text{Pr}^\mathcal{T}(x) dx + \epsilon^\mathcal{T}(h^\mathcal{S}, h^\mathcal{T}) \\
&= \int |\hat{h}(x) - h^\mathcal{S}(x)| \cdot \frac{\text{Pr}^\mathcal{T}(x)}{\text{Pr}^\mathcal{S}(x)} \cdot \text{Pr}^\mathcal{S}(x) dx + \epsilon^\mathcal{T}(h^\mathcal{S}, h^\mathcal{T}) \\
&= \mathbb{E}_{x \sim \mathbb{P}_\mathcal{S}} \frac{\text{Pr}^\mathcal{T}(x)}{\text{Pr}^\mathcal{S}(x)} \cdot |\hat{h}(x) - h^\mathcal{S}(x)|  + \epsilon^\mathcal{T}(h^\mathcal{S}, h^\mathcal{T}) \\
&=\frac{\text{Pr}^\mathcal{T}(x)}{\text{Pr}^\mathcal{S}(x)}\epsilon^\mathcal{S}(\hat{h}) + \epsilon^\mathcal{T}(h^\mathcal{S}, h^\mathcal{T})\\
&=\epsilon^\mathcal{S}(\hat{h}) + \epsilon^\mathcal{T}(h^\mathcal{S}, h^\mathcal{T}),
\end{aligned}
\end{equation}

where the density ratio $\frac{\text{Pr}^\mathcal{T}(x)}{\text{Pr}^\mathcal{S}(x)}$ is generally unobservable and is commonly treated as a constant.
The structural difference between Theorem~\ref{ap-bound2} and Lemma~\ref{ap-bound1} is the absence of the explicit $\mathcal{H}$-divergence penalty $d_{\mathcal{H}}(\tilde{\mathbb{P}}_\mathcal{S}, \tilde{\mathbb{P}}_\mathcal{T})$ in the prototype-based bound. To verify that this structural difference indeed yields a tighter bound, we examine both possible orderings of $\epsilon^\mathcal{S}(h^\mathcal{S}, h^\mathcal{T})$ and $\epsilon^\mathcal{T}(h^\mathcal{S}, h^\mathcal{T})$, since Lemma~\ref{ap-bound1} takes the minimum of these two terms.

\textbf{Case 1:} If $\epsilon^\mathcal{S}(h^\mathcal{S}, h^\mathcal{T}) \leq \epsilon^\mathcal{T}(h^\mathcal{S}, h^\mathcal{T})$, then the direct adaptation bound from Lemma~\ref{ap-bound1} becomes:
\begin{equation}
\epsilon_{0}^\mathcal{T}(\hat{h}) \leq \epsilon^\mathcal{S}(h^\mathcal{S}, h^\mathcal{T}) + \epsilon^\mathcal{S}(\hat{h}) + d_{\mathcal{H}}(\tilde{\mathbb{P}}_\mathcal{S}, \tilde{\mathbb{P}}_\mathcal{T}).
\end{equation}
Since $\epsilon^\mathcal{T}(h^\mathcal{S}, h^\mathcal{T}) \le \epsilon^\mathcal{S}(h^\mathcal{S}, h^\mathcal{T}) + d_{\mathcal{H}}(\tilde{\mathbb{P}}_\mathcal{S}, \tilde{\mathbb{P}}_\mathcal{T})$, we obtain:
\begin{equation}
\begin{aligned}
\epsilon^\mathcal{T}(\hat{h}) &\leq \epsilon^\mathcal{S}(\hat{h}) + \epsilon^\mathcal{T}(h^\mathcal{S}, h^\mathcal{T}) \\
&\leq \epsilon^\mathcal{S}(h^\mathcal{S}, h^\mathcal{T}) + \epsilon^\mathcal{S}(\hat{h}) + d_{\mathcal{H}}(\tilde{\mathbb{P}}_\mathcal{S}, \tilde{\mathbb{P}}_\mathcal{T})
= \epsilon_0^\mathcal{T}(\hat{h}).
\end{aligned}
\end{equation}

\textbf{Case 2:} If $\epsilon^\mathcal{S}(h^\mathcal{S}, h^\mathcal{T}) \geq \epsilon^\mathcal{T}(h^\mathcal{S}, h^\mathcal{T})$, then $\min\{\cdot\} = \epsilon^\mathcal{T}(h^\mathcal{S}, h^\mathcal{T})$, and:
\begin{equation}
\epsilon^\mathcal{T}(\hat{h}) \leq \epsilon^\mathcal{S}(\hat{h}) + \epsilon^\mathcal{T}(h^\mathcal{S}, h^\mathcal{T}) = \epsilon_0^\mathcal{T}(\hat{h}) - d_{\mathcal{H}}(\tilde{\mathbb{P}}_\mathcal{S}, \tilde{\mathbb{P}}_\mathcal{T}) \leq \epsilon_0^\mathcal{T}(\hat{h}).
\end{equation}

In both cases, the memory-based bound dominates $\epsilon_0^\mathcal{T}(\hat{h})$ by exactly $d_{\mathcal{H}}(\tilde{\mathbb{P}}_\mathcal{S}, \tilde{\mathbb{P}}_\mathcal{T}) \geq 0$. The fact that the same dominance margin holds under both orderings confirms that the improvement stems from the structural removal of $d_\mathcal{H}$ in Theorem~\ref{ap-bound2}, rather than from the choice of which term Lemma~\ref{ap-bound1} happens to minimize over.
\end{proof}

\begin{theorem}[Continual Error Upper Bound with State-Evolving Memories]
\label{ap-bound3}
Let $\tilde{\mathbb{P}}_{G^{t-1}}$ and $\tilde{\mathbb{P}}_{G^t}$ denote the induced feature distributions of two consecutive target domains at steps $t{-}1$ and $t$. The risk $\epsilon^{t}(\hat{h})$ with state-evolving memories (with mixed weight $w_{t-1} + w_t = 1$, $w_{t-1}, w_t \geq 0$, and $\hat{h} = w_{t-1} \hat{h}^{t-1} + w_t \hat{h}^t$) is upper bounded by:
\begin{equation}
\epsilon^t(\hat{h}) \leq \epsilon^t(\hat{h}, h^{t-1}) + \epsilon^t(h^{t-1}, h^t).
\end{equation}
\end{theorem}

\begin{proof}
By substituting $\mathcal{T}$ with $t$ and $\mathcal{S}$ with $t-1$ in Theorem~\ref{ap-bound2}, we obtain the following inequality, which holds throughout the continual adaptation process:
\begin{equation}
\epsilon^t(\hat{h}) \leq \epsilon^t(\hat{h}, h^{t-1}) + \epsilon^t(h^{t-1}, h^t).
\end{equation}

\noindent \textbf{Comparison of Error Upper Bounds.}
Theorem~\ref{ap-bound3} can be regarded as a special case of Theorem~\ref{ap-bound2}, where the $t$-th domain serves as the initial target domain for adaptation.
Therefore, the same ordering of error upper bounds holds, and the bound remains tighter than that of $\epsilon_{0}^{\mathcal{T}}(\hat{h})$.
\end{proof}

\noindent \textbf{Remark.}
Lemma~\ref{ap-bound1} provides a generalization error bound for adaptation using samples from the target domain.
Theorem~\ref{ap-bound2} gives the error bound when adapting a model from the source domain to the target domain with the help of crystallized memories.
Theorem~\ref{ap-bound3} further refines this by providing the error bound when modeling inter-domain relations through state-evolving memories during continual adaptation.
We arrive at the following conclusions:

\begin{itemize}[leftmargin=*,noitemsep,topsep=1.5pt]
    \item From Theorem~\ref{ap-bound2}, the error bound $\epsilon^{\mathcal{T}}(\hat{h})$ using crystallized memories is \textit{\textbf{tighter}} than the bound $\epsilon_{0}^{\mathcal{T}}(\hat{h})$ for direct source-to-target adaptation, and the gap is at least $d_{\mathcal{H}}(\tilde{\mathbb{P}}_\mathcal{S}, \tilde{\mathbb{P}}_\mathcal{T}) \ge 0$.
    \item According to Theorem~\ref{ap-bound3}, the error bound $\epsilon^{t}(\hat{h})$ using state-evolving memories is also \textit{\textbf{tighter}} than $\epsilon_{0}^{\mathcal{T}}(\hat{h})$ throughout continual adaptation.
\end{itemize}

\subsection{Proof of Proposition~\ref{prop:kl_stability}} \label{apsec-kl_stability}
\begin{proposition}[Stability Bound of the Adaptation Loss]
\label{ap-prop:kl_stability}
Consider the adaptation step from $\mathbf{\Theta}_{t-1}$ to $\mathbf{\Theta}_{t}$.
For each previously encountered domain indexed by $(t)\in\{1,\ldots,t-1\}$,
$\mathcal{D}_{\mathrm{KL}}(\hat{\tilde{G}}^{(t)}\|\tilde{G}^{(t)})$ can be interpreted as the divergence between the
predictive distributions induced by the model before and after adaptation
on the same crystallized memory $\tilde{G}^{(t)}$.
Moreover, assume that:
(i) the predictive distribution is given by a softmax with temperature $\tau>0$;
(ii) all predicted class probabilities are lower-bounded by $\gamma>0$; and
(iii) the encoder $h(\cdot)$ and the classifier $g(\cdot)$ are Lipschitz continuous \emph{with respect to the model parameters}.
Then the following upper bound holds:
\begin{equation}
\mathcal{D}_{\mathrm{KL}}(\hat{\tilde{G}}^{(t)}\|\tilde{G}^{(t)})
\;\le\;
\frac{L_g^2 L_h^2}{4\,\gamma\,\tau^2}\,
\big\|\mathbf{\Theta}_{t}-\mathbf{\Theta}_{t-1}\big\|_2^2,
\end{equation}
where $L_h$ and $L_g$ are the parameter-space Lipschitz constants of $h(\cdot)$ and $g(\cdot)$.
\end{proposition}

\begin{proof}
We derive the stability bound step by step.

\textbf{Step 1: Rewriting the KL divergence.}
For each $(t)\in\{1,\ldots,t-1\}$, define the predictive distributions on the
crystallized memory $\tilde{G}^{(t)}$ as
\begin{equation}
p_{t-1}^{(t)}(\cdot)
=
\operatorname{Softmax}\!\left(
g\!\left(h(\tilde{G}^{(t)};\mathbf{\Theta}_{t-1})\right)/\tau
\right),
\quad
p_{t}^{(t)}(\cdot)
=
\operatorname{Softmax}\!\left(
g\!\left(h(\tilde{G}^{(t)};\mathbf{\Theta}_{t})\right)/\tau
\right).
\end{equation}
By definition, the KL divergence in Eq.~\eqref{eq:loss-2} can be written as
\begin{equation}
\mathcal{D}_{\mathrm{KL}}(\hat{\tilde{G}}^{(t)}\|\tilde{G}^{(t)})
=
\mathcal{D}_{\mathrm{KL}}\!\left(
p_{t-1}^{(t)} \,\|\, p_{t}^{(t)}
\right).
\end{equation}

\textbf{Step 2: Upper bounding KL divergence by probability differences.}
Using the inequality $\log x \le x-1$ and the assumption $p_{t}^{(t)}(y)\ge \gamma$ for all classes $y$,
\begin{align}
\mathcal{D}_{\mathrm{KL}}\!\left(
p_{t-1}^{(t)} \,\|\, p_{t}^{(t)}
\right)
&= \sum_y p_{t-1}^{(t)}(y)\log\frac{p_{t-1}^{(t)}(y)}{p_{t}^{(t)}(y)} \\
&\le \sum_y p_{t-1}^{(t)}(y)
\left(
\frac{p_{t-1}^{(t)}(y)}{p_{t}^{(t)}(y)}-1
\right)
\le
\frac{1}{\gamma}
\big\|p_{t-1}^{(t)}-p_{t}^{(t)}\big\|_2^2.
\end{align}

\textbf{Step 3: Lipschitz continuity of the softmax operator.}
Let $\mathbf{s}_{t-1}^{(t)} = g(h(\tilde{G}^{(t)};\mathbf{\Theta}_{t-1}))$ and $\mathbf{s}_{t}^{(t)} = g(h(\tilde{G}^{(t)};\mathbf{\Theta}_{t}))$ denote the corresponding logits.
The softmax operator with temperature $\tau$ is $1/(2\tau)$-Lipschitz, yielding
\begin{equation}
\big\|p_{t-1}^{(t)}-p_{t}^{(t)}\big\|_2
\le
\frac{1}{2\tau}
\big\|\mathbf{s}_{t-1}^{(t)}-\mathbf{s}_{t}^{(t)}\big\|_2.
\end{equation}

\textbf{Step 4: Lipschitz chain with respect to model parameters.}
By the parameter-space Lipschitz continuity of $g(\cdot)$ and $h(\cdot)$,
\begin{equation}
\big\|\mathbf{s}_{t-1}^{(t)}-\mathbf{s}_{t}^{(t)}\big\|_2
\le
L_g L_h\,
\big\|\mathbf{\Theta}_{t-1}-\mathbf{\Theta}_{t}\big\|_2.
\end{equation}

\textbf{Step 5: Combining all bounds.}
Substituting the above inequalities yields
\begin{equation}
\mathcal{D}_{\mathrm{KL}}(\hat{\tilde{G}}^{(t)}\|\tilde{G}^{(t)})
\le
\frac{L_g^2 L_h^2}{4\,\gamma\,\tau^2}\,
\big\|\mathbf{\Theta}_{t}-\mathbf{\Theta}_{t-1}\big\|_2^2,
\end{equation}
which completes the proof.
\end{proof}

\noindent\textbf{Remark on the Lipschitz Assumption.}
Assumption~(iii) is parameter-space Lipschitz continuity, namely that the network output changes smoothly as the parameters $\mathbf{\Theta}$ vary for a fixed input. This is fundamentally different from input-space Lipschitz continuity (which would require explicit spectral normalization). For any bounded-weight network, the parameter-space sensitivity satisfies $\|\partial h/\partial\mathbf{\Theta}\| \le \prod_{\ell=1}^{L}\|\mathbf{W}^\ell\|_2 \cdot \|x\|$, which is finite on compact inputs and is directly bounded by weight decay through $\|\mathbf{W}^\ell\|_2 \le \|\mathbf{W}^\ell\|_F$~\cite{bartlett2017spectrally,neyshabur2018pac}. Therefore, standard backbones such as GCN and GraphSAGE satisfy Assumption~(iii) under the weight decay used in our experiments, without any additional spectral normalization. Importantly, the proposition does not require $L_h L_g$ to be tight: it only asserts that small parameter updates lead to small predictive shifts on stored memories, which is sufficient to justify the KL regularization in Eq.~\eqref{eq:loss-2} as an anti-forgetting mechanism.

\noindent\textbf{Interpretation of the Two Loss Terms.}
Proposition~\ref{ap-prop:kl_stability} provides an explicit interpretation of the adaptation loss in Eq.~\eqref{eq:loss-2} and clarifies how the two components in the overall objective serve complementary purposes during adaptation.

\begin{itemize}[leftmargin=*,noitemsep,topsep=1.5pt]
    \item \textbf{Adaptation to the current domain.}
    The entropy minimization term $H(G^{t})$ encourages confident predictions on the current test graph. By reducing predictive uncertainty on $G^{t}$, this term drives the model to better align with the data distribution of the newly arriving domain. This component focuses exclusively on the current domain.

    \item \textbf{Stability on previously encountered domains.}
    The KL divergence term $\mathcal{D}_{\mathrm{KL}}(\hat{\tilde{G}}^{(t)}\|\tilde{G}^{(t)})$ does not compare different domains. Instead, for each previously stored memory $\tilde{G}^{(t)}$, it constrains the change in predictive distributions induced by updating the model from $\mathbf{\Theta}_{t-1}$ to $\mathbf{\Theta}_{t}$. By Proposition~\ref{ap-prop:kl_stability}, this divergence admits an explicit quadratic upper bound with respect to the parameter update magnitude, implying that small updates lead to bounded changes in predictive behavior on past domains.

    \item \textbf{Mitigating catastrophic forgetting through controlled transfer.}
    Together, the two terms balance adaptation and stability: entropy minimization promotes rapid adaptation to the current domain, while the KL regularization prevents excessive drift on previously encountered domains. The adaptation loss can thus be viewed as a mechanism that regulates cross-domain knowledge transfer during continual adaptation.
\end{itemize}

Overall, Proposition~\ref{ap-prop:kl_stability} formalizes the intuition that test-time adaptation under non-stationary drift should not only fit the current domain but also preserve predictive consistency on past domains.

\section{Additional Experiments}

\input{table/ap-dataset-info-tab}
\subsection{Dataset Description} \label{app:datasets}

We evaluate \ourmethod on four benchmark node classification datasets, encompassing two distinct types of distribution shifts: \textbf{regional shifts} and \textbf{temporal shifts}. This setting follows the continual adaptation protocol in GCAL~\cite{qiao2025gcal}, where each graph is treated as an individual domain. The statistical summary of all datasets is provided in Table~\ref{table_dataset_intro}.

\textbf{Regional Shifts.}
\textit{Facebook-100}~\cite{fb100} contains 100 ego-centric Facebook friendship networks collected from U.S. universities in 2005. Each network represents a unique university and varies significantly in node count, density, and degree distribution. We select three university graphs (Amherst41, Caltech36, and Johns Hopkins55) for initial training, and perform continual adaptation on eleven other university graphs without labels, following the protocol in~\cite{EERM}.
\textit{Twitch-Explicit}~\cite{twitch} includes seven social graphs of Twitch users from different regions (DE, ENGB, ES, FR, PTBR, RU, TW). Each graph contains users as nodes and friendship relations as edges. The DE network is used for pretraining, and the remaining regional graphs are used for continual test-time adaptation, also without labels.

\textbf{Temporal Shifts.}
\textit{OGB-Arxiv}~\cite{ogbn-arxiv} is a citation network of Arxiv CS papers across 40 research areas. Nodes represent papers, and edges indicate citation links. We pretrain the model on papers published before 2011 and adapt it sequentially to papers published from 2011 onward, capturing the temporal evolution of scientific collaborations.
\textit{Elliptic}~\cite{elliptic} consists of 49 graph snapshots of Bitcoin transaction networks. Nodes represent transactions, and edges denote payment flows. Approximately 20\% of nodes are labeled as licit or illicit. To avoid extreme class imbalance, we use snapshots 7--9 for pretraining and adapt the model continually using the remaining graphs, reflecting the temporal dynamics of illicit activity detection.

Across all datasets, we follow a consistent strategy: pretrain a GNN on a small set of labeled graphs, then continually adapt to the remaining unlabeled graphs in a test-time online setting. To simplify notation in our tables and figures, we refer to Twitch-Explicit as \textbf{Twitch}, Facebook-100 as \textbf{Facebook}, and OGB-Arxiv as \textbf{Arxiv}.

\subsection{Baseline Details}

We compare \ourmethod against ten representative baseline methods.
The ``Base Model'' refers to the GNN pre-trained on source domains and directly applied to target domains without any adaptation, serving as a lower-bound reference.
These methods are grouped into two main categories: (1) one-step test-time adaptation methods and (2) continual test-time adaptation methods. To ensure fair comparison, all non-graph-specific methods are adapted to GNN-based architectures.
\begin{itemize}[leftmargin=*,noitemsep,topsep=1.5pt]

\item \textbf{One-step test-time adaptation methods} perform adaptation using only the current domain without memory or explicit long-term modeling.
DANN~\cite{DANN} is a classical domain adaptation approach that minimizes domain discrepancy via adversarial training using source and target data.
Tent~\cite{TENT} minimizes the entropy of predictions on unlabeled target samples to adapt the model.
BN Stats Adapt~\cite{Norm} updates network weights and batch normalization statistics based on the current test batch.
EERM~\cite{EERM}, designed specifically for graphs, maximizes environment risk variance to address domain shift and out-of-distribution challenges.
GTRANS~\cite{jin2022empowering} refines both graph structures and node features using a contrastive surrogate loss to improve test-time performance.

\item \textbf{Continual test-time adaptation methods} aim to maintain performance under sequential distribution shifts.
CoTTA~\cite{CoTTA} combines weight-averaged predictions and stochastic neuron restoration to mitigate error accumulation and catastrophic forgetting.
EATA~\cite{EATA} improves adaptation efficiency through entropy-based updates and a Fisher-based regularizer for stability.
GCAL~\cite{qiao2025gcal} enhances cross-domain adaptation by training an auxiliary VGAE to generate memory graphs, but its reliance on a separately trained generator limits scalability in long-term scenarios.
STEM~\cite{luo2026space} introduces a controller-worker architecture to dynamically generate adapter parameters for a frozen base GNN.

\end{itemize}

\subsection{Experimental Setting} \label{app:setting}

In this work, we are the first to investigate the generalizability and adaptability of graph models under non-stationary distribution shifts at test time, while also addressing the challenge of catastrophic forgetting.

We distinguish between two key concepts: \textbf{generalizability} and \textbf{adaptability}. Generalizability refers to the model's inherent ability to perform well on previously seen domains, even while being adapted to new domains. This emphasizes the model's capacity to retain knowledge throughout the continual adaptation process. In contrast, adaptability reflects how effectively a model can enhance its performance on a newly encountered domain through test-time optimization using unlabeled data. In other words, generalizability focuses on performance retention, while adaptability emphasizes performance gain under novel domain shifts. Figure~\ref{fig:metric} provides a visual comparison of these two metrics, with exact formulations introduced below.

\begin{figure}[t]
\centering
\includegraphics[width=0.25\linewidth]{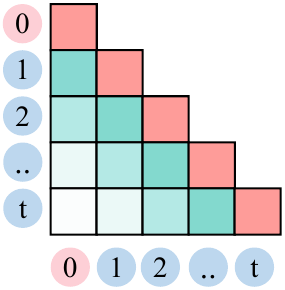}
\caption{An example of the performance matrix. For the current $t$-th domain:
TDG is computed as the mean of the green elements;
TDA is the mean of the pink elements;
AP is the average of the elements in the last row;
AF is the average of the differences between each element in the last row (excluding the last one) and the corresponding element in the previous row.
}
\label{fig:metric}
\end{figure}

For our experimental configuration, we use a 2-layer GCN as the backbone for three datasets, except for OGB-Arxiv where we adopt GraphSAGE~\cite{hamilton2017inductive}. For pretraining, we split the source graphs into 60\% training, 20\% validation, and 20\% testing. The pretraining phase uses a learning rate of 1e-04 and a weight decay of 5e-04. During adaptation, the learning rate is set to 1e-03 with the same weight decay. Pretraining lasts for 100--200 epochs depending on the dataset, while adaptation uses 1--10 epochs per domain. All hyperparameter choices follow the protocol of GCAL~\cite{qiao2025gcal}.
The number of nodes in each crystallized memory $N_t'$ is set to 1\% of the total number of nodes in the input graph. We analyze the sensitivity of hyperparameters $\alpha$ and $\beta$ in the main paper, with extended analyses on $\tau$ and the memory size factor $\lambda$ in Section~\ref{ap:sensi}. The convex-combination weights $w_{t-1}, w_t$ are fixed hyperparameters satisfying $w_{t-1}+w_t=1$ and selected by grid search over $w_t\in\{0.3, 0.5, 0.7, 0.9\}$. All results are averaged over five independent runs, with standard deviations reported.

To measure continual adaptation performance, we construct a performance matrix $M \in \mathbb{R}^{T \times T}$, where $M_{i,j}$ (for $i \geq j$) represents the model's performance on domain $j$ after training on domain $i$. Following~\cite{jin2022empowering,EERM}, we use ROC-AUC for Twitch-Explicit, Accuracy for Facebook-100 and OGB-Arxiv, and F1 Score for Elliptic.
We report the following metrics derived from $M$:

\begin{itemize}[leftmargin=1.5em]
    \item \textbf{Average Performance (AP)}: $\text{AP} = \frac{1}{T} \sum_{i=1}^{T} M_{T,i}$, reflecting the model's overall adaptation ability after completing all domain shifts.
    \item \textbf{Average Forgetting (AF)}: $\text{AF} = \frac{1}{T-1} \sum_{i=1}^{T-1} (M_{T,i} - M_{i,i})$, quantifying the degradation in performance on past domains.
    \item \textbf{Target Domain Generalization (TDG)}~\cite{liu2023deja}: $\text{TDG} = \frac{1}{T-1} \sum_{i=2}^{T} M_{i-1,i}$, measuring model generalization to unseen domains before adaptation.
    \item \textbf{Target Domain Adaptation (TDA)}~\cite{liu2023deja}: $\text{TDA} = \frac{1}{T} \sum_{i=1}^{T} M_{i,i}$, measuring performance on each domain after adaptation.
\end{itemize}

\subsection{Configurations}
We conduct the experiments with:
\begin{itemize}[leftmargin=1.5em]
    \item Operating System: Ubuntu 20.04 LTS.
    \item CPU: Intel(R) Xeon(R) Gold 6240 CPU @ 2.60GHz, 256GB RAM.
    \item GPU: Tesla V100 PCIe 32GB GPU.
    \item Software: Python 3.7, Pytorch 1.8, CUDA 11.0, and Pytorch-Geometric 2.0.1.
\end{itemize}

\subsection{Performance Comparison under the GCAL Setting} \label{ap:full-baseline}

We provide the complete comparison summarized in Section~\ref{sec:overall} between \ourmethod and the 10 representative baselines under the GCAL setting in Table~\ref{tab:all_performance}. 
We use GCN as the backbone on Twitch-explicit, Facebook-100, and Elliptic, and GraphSAGE on OGB-Arxiv. The baselines span both \textit{one-step test-time adaptation} (DANN, Norm, TENT, EERM, GTrans, LEBED) and \textit{continual test-time adaptation} (CoTTA, EATA, GCAL, STEM). For STEM, we report results reproduced from the official codebase, which only supports Twitch-explicit and OGB-Arxiv, and the remaining two datasets are therefore marked as Not Applicable. 

Results show that \ourmethod achieves the best AP on three of the four datasets (Facebook-100, Elliptic, OGB-Arxiv) and the best AF on all four. On Twitch-explicit, STEM attains a slightly higher AP but exhibits a negative AF, indicating that its gain on the latest domain comes at the cost of forgetting earlier ones. By contrast, \ourmethod consistently maintains a positive AF across all datasets, confirming that the closed-form memory crystallization and state-evolving design jointly support adaptability without sacrificing knowledge retention.

\begin{table*}[!t]
\caption{Performance comparison ($\%$, mean $\pm$ standard deviation for five runs) across node-level (Twitch-explicit, OGB-Arxiv) and graph-level (Mutagenicity, FRANKENSTEIN) datasets under continual distribution shifts. ``Avg.'' is computed uniformly as the mean of TDG and TDA. ``$\uparrow$'' indicates higher is better and ``$\downarrow$'' indicates lower is better. The best results are highlighted in \textbf{bold} and the runner-ups are \underline{underlined}.}
\centering
\resizebox{\textwidth}{!}{
\begin{tabular}{c|c|cccccccccccc}
\toprule
\multirow{2}[4]{*}{\textbf{Dataset}} & \multirow{2}[4]{*}{\textbf{Methods}} & \multicolumn{6}{c}{\textbf{Backbone: GCN}} & \multicolumn{6}{c}{\textbf{Backbone: SAGE}} \\
\cmidrule(lr){3-8} \cmidrule(lr){9-14}
& & AP ($\uparrow$) & AF ($\uparrow$) & TDG ($\uparrow$) & TDA ($\uparrow$) & Avg.\ ($\uparrow$) & GPU ($\downarrow$) & AP ($\uparrow$) & AF ($\uparrow$) & TDG ($\uparrow$) & TDA ($\uparrow$) & Avg.\ ($\uparrow$) & GPU ($\downarrow$) \\
\midrule
\multicolumn{14}{c}{\textit{Node-level Adaptation}} \\
\midrule

\multirow{3}{*}{\textbf{Twitch-explicit}}
& GCAL & 55.65{\scriptsize$\pm$0.09} & \underline{0.42{\scriptsize$\pm$0.13}} & 54.78{\scriptsize$\pm$0.60} & 53.20{\scriptsize$\pm$0.62} & 53.99 & 11159.80
       & 58.20{\scriptsize$\pm$0.35} & \underline{0.79{\scriptsize$\pm$0.13}} & 59.65{\scriptsize$\pm$0.39} & \underline{57.02{\scriptsize$\pm$0.49}} & 58.34 & 8580.63 \\
& STEM & \textbf{56.66{\scriptsize$\pm$0.05}} & -0.07{\scriptsize$\pm$0.17} & \textbf{57.35{\scriptsize$\pm$0.07}} & \textbf{55.87{\scriptsize$\pm$0.15}} & \textbf{56.61} & \underline{7393.27}
       & \underline{59.49{\scriptsize$\pm$0.06}} & -0.45{\scriptsize$\pm$0.02} & \textbf{63.56{\scriptsize$\pm$0.01}} & 56.33{\scriptsize$\pm$0.05} & \underline{59.95} & \underline{7961.84} \\
& \cellcolor{myblue}\textbf{\ourmethod} & \cellcolor{myblue}\underline{55.95{\scriptsize$\pm$0.15}} & \cellcolor{myblue}\textbf{0.47{\scriptsize$\pm$0.08}} & \cellcolor{myblue}\underline{55.56{\scriptsize$\pm$0.13}} & \cellcolor{myblue}\underline{55.83{\scriptsize$\pm$0.14}} & \cellcolor{myblue}\underline{55.70} & \cellcolor{myblue}\textbf{854.97}
       & \cellcolor{myblue}\textbf{61.62{\scriptsize$\pm$0.12}} & \cellcolor{myblue}\textbf{0.86{\scriptsize$\pm$0.04}} & \cellcolor{myblue}\underline{62.02{\scriptsize$\pm$0.13}} & \cellcolor{myblue}\textbf{60.36{\scriptsize$\pm$0.16}} & \cellcolor{myblue}\textbf{61.19} & \cellcolor{myblue}\textbf{2833.18} \\

\midrule

\multirow{3}{*}{\textbf{OGB-Arxiv}}
& GCAL & \underline{14.15{\scriptsize$\pm$6.94}} & \underline{-1.80{\scriptsize$\pm$2.06}} & \underline{17.90{\scriptsize$\pm$8.58}} & \underline{13.77{\scriptsize$\pm$5.57}} & \underline{15.83} & 22342.47
       & \underline{45.22{\scriptsize$\pm$0.17}} & \underline{0.76{\scriptsize$\pm$0.05}} & \underline{42.40{\scriptsize$\pm$7.74}} & \underline{41.80{\scriptsize$\pm$6.59}} & \underline{42.10} & 22365.84 \\
& STEM & 10.25{\scriptsize$\pm$0.01} & -2.79{\scriptsize$\pm$0.01} & 12.75{\scriptsize$\pm$0.01} & 10.56{\scriptsize$\pm$0.01} & 11.66 & \underline{7010.74}
       & 31.71{\scriptsize$\pm$0.14} & -4.40{\scriptsize$\pm$0.19} & 34.19{\scriptsize$\pm$0.25} & 32.77{\scriptsize$\pm$0.18} & 33.48 & \underline{7012.99} \\
& \cellcolor{myblue}\textbf{\ourmethod} & \cellcolor{myblue}\textbf{19.32{\scriptsize$\pm$1.57}} & \cellcolor{myblue}\textbf{3.38{\scriptsize$\pm$1.99}} & \cellcolor{myblue}\textbf{19.39{\scriptsize$\pm$1.47}} & \cellcolor{myblue}\textbf{20.39{\scriptsize$\pm$1.71}} & \cellcolor{myblue}\textbf{19.89} & \cellcolor{myblue}\textbf{299.18}
       & \cellcolor{myblue}\textbf{46.07{\scriptsize$\pm$1.70}} & \cellcolor{myblue}\textbf{0.99{\scriptsize$\pm$0.66}} & \cellcolor{myblue}\textbf{47.11{\scriptsize$\pm$1.80}} & \cellcolor{myblue}\textbf{43.52{\scriptsize$\pm$1.25}} & \cellcolor{myblue}\textbf{45.31} & \cellcolor{myblue}\textbf{391.02} \\

\midrule
\midrule
\multicolumn{14}{c}{\textit{Graph-level Adaptation}} \\
\midrule

\multirow{3}{*}{\textbf{Mutagenicity}}
& GCAL & \underline{67.96{\scriptsize$\pm$0.10}} & \underline{-1.55{\scriptsize$\pm$0.14}} & \underline{67.59{\scriptsize$\pm$0.06}} & \underline{69.13{\scriptsize$\pm$0.04}} & \underline{68.36} & \underline{3805.50}
       & 65.49{\scriptsize$\pm$0.78} & \underline{-0.70{\scriptsize$\pm$0.54}} & 64.22{\scriptsize$\pm$0.63} & 66.19{\scriptsize$\pm$0.33} & 65.20 & \underline{1221.09} \\
& STEM & 64.48{\scriptsize$\pm$0.26} & -1.60{\scriptsize$\pm$0.22} & 66.73{\scriptsize$\pm$0.11} & 64.72{\scriptsize$\pm$0.13} & 65.72 & 10361.79
       & \textbf{67.58{\scriptsize$\pm$0.71}} & -3.92{\scriptsize$\pm$0.77} & \textbf{70.42{\scriptsize$\pm$0.50}} & \textbf{70.33{\scriptsize$\pm$0.15}} & \textbf{70.38} & 10357.06 \\
& \cellcolor{myblue}\textbf{\ourmethod} & \cellcolor{myblue}\textbf{68.40{\scriptsize$\pm$1.66}} & \cellcolor{myblue}\textbf{-0.85{\scriptsize$\pm$0.61}} & \cellcolor{myblue}\textbf{69.94{\scriptsize$\pm$0.81}} & \cellcolor{myblue}\textbf{69.71{\scriptsize$\pm$1.22}} & \cellcolor{myblue}\textbf{69.82} & \cellcolor{myblue}\textbf{85.14}
       & \cellcolor{myblue}\underline{67.15{\scriptsize$\pm$0.25}} & \cellcolor{myblue}\textbf{-0.66{\scriptsize$\pm$0.07}} & \cellcolor{myblue}\underline{66.34{\scriptsize$\pm$0.07}} & \cellcolor{myblue}\underline{66.68{\scriptsize$\pm$0.21}} & \cellcolor{myblue}\underline{66.51} & \cellcolor{myblue}\textbf{111.61} \\

\midrule

\multirow{3}{*}{\textbf{FRANKENSTEIN}}
& GCAL & 59.26{\scriptsize$\pm$0.66} & \underline{-0.22{\scriptsize$\pm$0.21}} & 59.53{\scriptsize$\pm$0.33} & 59.32{\scriptsize$\pm$0.52} & 59.42 & \underline{4950.38}
       & \textbf{58.34{\scriptsize$\pm$0.70}} & \textbf{0.40{\scriptsize$\pm$0.64}} & \textbf{56.87{\scriptsize$\pm$0.46}} & \textbf{57.44{\scriptsize$\pm$0.17}} & \textbf{57.16} & \underline{2085.58} \\
& STEM & \underline{61.61{\scriptsize$\pm$0.41}} & \textbf{1.41{\scriptsize$\pm$0.25}} & \underline{60.61{\scriptsize$\pm$0.38}} & \underline{59.81{\scriptsize$\pm$0.21}} & \underline{60.21} & 8322.35
       & 54.36{\scriptsize$\pm$0.33} & -0.96{\scriptsize$\pm$0.86} & \underline{53.91{\scriptsize$\pm$0.67}} & 54.15{\scriptsize$\pm$0.59} & 54.03 & 10461.86 \\
& \cellcolor{myblue}\textbf{\ourmethod} & \cellcolor{myblue}\textbf{61.86{\scriptsize$\pm$0.67}} & \cellcolor{myblue}-0.45{\scriptsize$\pm$0.15} & \cellcolor{myblue}\textbf{61.25{\scriptsize$\pm$0.15}} & \cellcolor{myblue}\textbf{61.11{\scriptsize$\pm$0.75}} & \cellcolor{myblue}\textbf{61.18} & \cellcolor{myblue}\textbf{201.65}
       & \cellcolor{myblue}\underline{56.42{\scriptsize$\pm$0.10}} & \cellcolor{myblue}\underline{-0.43{\scriptsize$\pm$0.08}} & \cellcolor{myblue}53.80{\scriptsize$\pm$0.13} & \cellcolor{myblue}\underline{55.38{\scriptsize$\pm$0.05}} & \cellcolor{myblue}\underline{54.59} & \cellcolor{myblue}\textbf{222.39} \\

\bottomrule
\end{tabular}}
\label{tab:stem-comparison}
\end{table*}

\subsection{Extended Performance Comparison across Node- and Graph-level Tasks} \label{ap:stem-comparison}

To further situate \ourmethod against the very recent state-of-the-art STEM~\cite{luo2026space}, we extend the comparison to four datasets covering both node-level adaptation (Twitch-explicit, OGB-Arxiv) and graph-level adaptation (Mutagenicity, FRANKENSTEIN), evaluated under two backbones (GCN and GraphSAGE). We compare \ourmethod against the recent strong baselines GCAL~\cite{qiao2025gcal} and STEM~\cite{luo2026space}. Results are reported in Table~\ref{tab:stem-comparison}, where ``Avg.\ '' denotes the mean of TDG and TDA and serves as a balanced summary of pre- and post-adaptation behavior.

\textbf{Effectiveness across node-level and graph-level tasks.}
\ourmethod attains the best Avg.\ on five out of eight (dataset, backbone) combinations and the runner-up on the remaining three, demonstrating consistent effectiveness beyond the node-level setting that the prior literature on graph continual test-time adaptation has primarily focused on. The pattern of best results also covers both task granularities: \ourmethod leads on three of four node-level entries and on two of four graph-level entries, indicating that the closed-form memory crystallization transfers naturally from node-level to graph-level adaptation without any task-specific modification. We further note that \ourmethod consistently achieves the best AF in most settings, confirming that the state-evolving design effectively mitigates forgetting under continual drift, whereas STEM's stronger adaptation on certain domains is often accompanied by a negative AF (\eg, -3.92 on Mutagenicity-GraphSAGE, -4.40 on Arxiv-GraphSAGE), reflecting a less favorable retention--adaptation trade-off.

\textbf{Substantial Efficiency Advantage.}
The most pronounced gap appears in GPU memory consumption, where the closed-form construction in \ourmethod removes the need for any auxiliary trainable component at test time. On node-level tasks, \ourmethod is roughly an order of magnitude more memory-efficient than STEM and over an order of magnitude lighter than GCAL. 
The advantage widens on graph-level tasks, where \ourmethod operates within a few hundred MiB while STEM exceeds ten thousand MiB on both Mutagenicity and FRANKENSTEIN, a difference of two orders of magnitude. Such a gap is structural rather than incidental: STEM and GCAL both rely on auxiliary modules whose memory cost grows with graph size, whereas the closed-form memory in \ourmethod is solved analytically over a fraction of the input nodes ($\rhd$ Section~\ref{sec:proto}), yielding a footprint that stays compact regardless of the underlying task granularity. These results confirm that \ourmethod delivers competitive or superior effectiveness while remaining the most memory-efficient option among recent continual test-time adaptation methods, making it well-suited for prolonged adaptation under non-stationary shifts.

\subsection{Additional Performance Matrices} \label{ap:addition-result}
We extend the performance-matrix visualization in Figure~\ref{fig:matrix} of the main paper to the Twitch-Explicit and Elliptic datasets in Figure~\ref{fig:ap-matrix}. Across both datasets, \ourmethod maintains consistently strong performance across the diagonal and lower-triangular entries, indicating that the model not only adapts to each newly arriving domain but also retains its competence on previously encountered ones. In contrast, GCAL exhibits visibly degraded performance on earlier domains as adaptation proceeds, particularly on Elliptic, where the temporal drift can be abrupt.

We further note an observation on Facebook-100 in Figure~\ref{fig:matrix}: the performance of \ourmethod on the 6th target domain is slightly lower than on neighboring domains. We attribute this to a larger covariate shift between this domain and the rest of the sequence, which makes single-pass adaptation under the test-time setting more challenging. Importantly, this localized dip does not propagate: the performance on subsequent domains recovers, and the overall AP and AF on Facebook-100 in Table~\ref{tab:more-metric} remain ahead of all baselines, confirming that \ourmethod is robust to occasional outlier domains under continual drift.

\begin{figure*}[htbp]
\centering
\hspace{-0.28cm}
 \subfigure[GCAL on Twitch-explicit]{
   \includegraphics[width=0.24\linewidth]{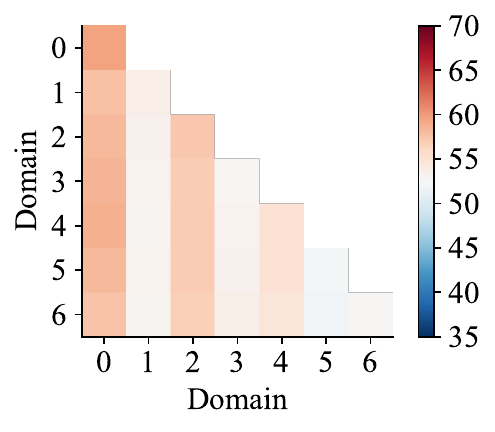}}
\hspace{-0.28cm}
 \subfigure[\ourmethod on Twitch-explicit]{
   \includegraphics[width=0.24\linewidth]{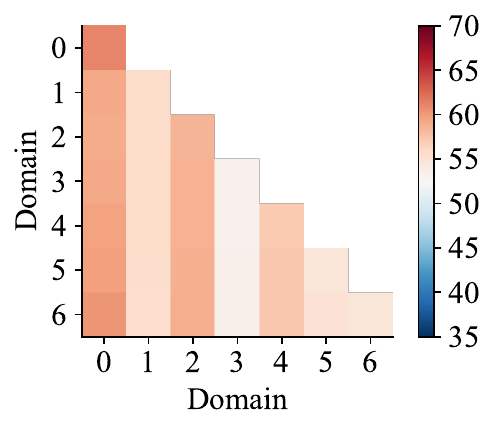}}
\hspace{-0.28cm}
 \subfigure[GCAL on Elliptic]{
   \includegraphics[width=0.24\linewidth]{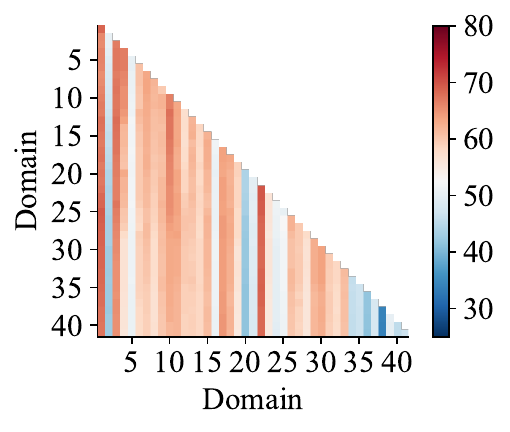}}
\hspace{-0.28cm}
 \subfigure[\ourmethod on Elliptic]{
   \includegraphics[width=0.24\linewidth]{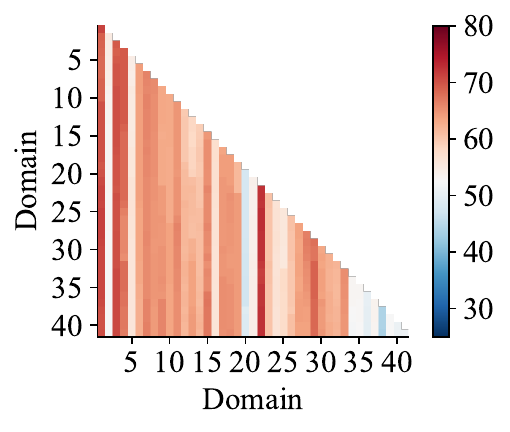}}
\vspace{-2.5mm}
\caption{Performance matrices of \ourmethod and GCAL on different datasets.
}
\vspace{-1em}
\label{fig:ap-matrix}
\end{figure*}

\subsection{Robustness to Different Domain-Shift Patterns} \label{ap:cyclic}
The state evolving field is a first-order linear interpolation, which prompts a natural question: does \ourmethod remain effective when distribution shifts arrive in a non-monotonic or cyclic pattern? We note that Theorem~\ref{bound3} holds for \emph{any} convex combination satisfying $w_{t-1}+w_t=1$, so the theoretical tightness does not depend on a stronger linearity assumption. Empirically, the memory at each step is reconstructed from the current domain rather than carried forward, the state evolving field interpolates only encoded representations of adjacent domains, and entropy minimization keeps the model calibrated regardless of the order of arrival.

We further verify this with two stress tests on Twitch-Explicit with GCN backbone: (i) varying the order of the six target domains to break monotonicity, and (ii) cyclic shifts where target domains repeat. Tables~\ref{tab:cyclic-1} and~\ref{tab:cyclic-2} show that \ourmethod outperforms GCAL across all orderings and cyclic patterns, with much lower runtime and GPU memory.

\begin{table*}[t]
\caption{Robustness to non-monotonic domain orderings on Twitch-Explicit.}
  \centering
  \resizebox{.9\textwidth}{!}{
    \begin{tabular}{l|l|cccccc}
    \toprule
    \textbf{Domain Order} & \textbf{Method} & AP & AF & TDG & TDA & Time (s) & GPU (MiB) \\
    \midrule
    \multirow{2}{*}{\texttt{ENGB, ES, FR, PTBR, RU, TW}} & GCAL & 55.65 & 0.42 & 54.78 & 53.20 & 217.51 & 11159.8 \\
    & \cellcolor{myblue}\textbf{\ourmethod} & \cellcolor{myblue}\textbf{55.95} & \cellcolor{myblue}\textbf{0.47} & \cellcolor{myblue}\textbf{55.56} & \cellcolor{myblue}\textbf{55.83} & \cellcolor{myblue}\textbf{74.64} & \cellcolor{myblue}\textbf{854.97} \\
    \midrule
    \multirow{2}{*}{\texttt{RU, FR, ES, ENGB, TW, PTBR}} & GCAL & 53.58 & 0.18 & 52.67 & 53.32 & 217.11 & 12225.69 \\
    & \cellcolor{myblue}\textbf{\ourmethod} & \cellcolor{myblue}\textbf{57.04} & \cellcolor{myblue}\textbf{0.36} & \cellcolor{myblue}\textbf{55.69} & \cellcolor{myblue}\textbf{57.46} & \cellcolor{myblue}\textbf{75.92} & \cellcolor{myblue}\textbf{832.58} \\
    \midrule
    \multirow{2}{*}{\texttt{FR, PTBR, TW, ENGB, RU, ES}} & GCAL & 53.37 & 0.12 & 52.47 & 53.93 & 216.35 & 10304.73 \\
    & \cellcolor{myblue}\textbf{\ourmethod} & \cellcolor{myblue}\textbf{57.12} & \cellcolor{myblue}\textbf{0.46} & \cellcolor{myblue}\textbf{55.68} & \cellcolor{myblue}\textbf{58.20} & \cellcolor{myblue}\textbf{73.75} & \cellcolor{myblue}\textbf{648.62} \\
    \midrule
    \multirow{2}{*}{\texttt{FR, PTBR, ES, RU, TW, ENGB}} & GCAL & 53.42 & 0.02 & 52.63 & 54.10 & 217.56 & 10677.54 \\
    & \cellcolor{myblue}\textbf{\ourmethod} & \cellcolor{myblue}\textbf{56.97} & \cellcolor{myblue}\textbf{0.30} & \cellcolor{myblue}\textbf{55.67} & \cellcolor{myblue}\textbf{58.32} & \cellcolor{myblue}\textbf{77.85} & \cellcolor{myblue}\textbf{830.88} \\
    \bottomrule
    \end{tabular}
  }
  \label{tab:cyclic-1}
\end{table*}

\begin{table*}[t]
\caption{Robustness to cyclic domain shifts on Twitch-Explicit. Domains repeat to test stability under non-monotonic drift patterns.}
  \centering
  \resizebox{.95\textwidth}{!}{
    \begin{tabular}{l|l|cccccc}
    \toprule
    \textbf{Domain Sequence} & \textbf{Method} & AP & AF & TDG & TDA & Time (s) & GPU (MiB) \\
    \midrule
    \multirow{2}{*}{\texttt{ENGB, ES, FR, ENGB, ES, FR}} & GCAL & 53.78 & -0.21 & 53.27 & 54.62 & 209.09 & 10116.01 \\
    & \cellcolor{myblue}\textbf{\ourmethod} & \cellcolor{myblue}\textbf{54.00} & \cellcolor{myblue}\textbf{0.06} & \cellcolor{myblue}\textbf{53.39} & \cellcolor{myblue}\textbf{54.94} & \cellcolor{myblue}\textbf{56.67} & \cellcolor{myblue}\textbf{605.80} \\
    \midrule
    \multirow{2}{*}{\texttt{PTBR, RU, TW, PTBR, RU, TW}} & GCAL & 52.84 & -0.06 & 52.03 & 54.04 & 223.54 & 14332.55 \\
    & \cellcolor{myblue}\textbf{\ourmethod} & \cellcolor{myblue}\textbf{53.02} & -0.17 & \cellcolor{myblue}\textbf{52.67} & 53.96 & \cellcolor{myblue}\textbf{95.79} & \cellcolor{myblue}\textbf{828.53} \\
    \midrule
    \multirow{2}{*}{\texttt{RU, FR, ES, RU, FR, ES}} & GCAL & 52.11 & 0.00 & 51.25 & 52.56 & 218.63 & 12559.53 \\
    & \cellcolor{myblue}\textbf{\ourmethod} & \cellcolor{myblue}\textbf{52.22} & -0.06 & \cellcolor{myblue}\textbf{51.39} & \cellcolor{myblue}\textbf{52.80} & \cellcolor{myblue}\textbf{77.21} & \cellcolor{myblue}\textbf{774.43} \\
    \midrule
    \multirow{2}{*}{\texttt{ENGB, TW, PTBR, ENGB, TW, PTBR}} & GCAL & \textbf{54.80} & 0.02 & \textbf{54.24} & 54.54 & 214.80 & 10802.73 \\
    & \cellcolor{myblue}\textbf{\ourmethod} & \cellcolor{myblue}54.61 & \cellcolor{myblue}\textbf{0.20} & \cellcolor{myblue}53.84 & \cellcolor{myblue}\textbf{55.09} & \cellcolor{myblue}\textbf{82.64} & \cellcolor{myblue}\textbf{768.01} \\
    \midrule
    \multirow{2}{*}{\makecell[l]{\texttt{ENGB, ES, FR, PTBR, RU, TW,}\\\texttt{ENGB, ES, FR, PTBR, RU, TW}}} & GCAL & 52.87 & -0.35 & 52.71 & 53.27 & 444.15 & 20293.89 \\
    & \cellcolor{myblue}\textbf{\ourmethod} & \cellcolor{myblue}\textbf{53.16} & \cellcolor{myblue}\textbf{0.05} & \cellcolor{myblue}\textbf{52.84} & \cellcolor{myblue}\textbf{53.54} & \cellcolor{myblue}\textbf{152.72} & \cellcolor{myblue}\textbf{785.17} \\
    \bottomrule
    \end{tabular}
  }
  \label{tab:cyclic-2}
\end{table*}

\begin{table*}[t]
\caption{Comparison with alternative training-free memory construction schemes on Twitch-Explicit with GCN backbone.}
  \centering
  \resizebox{.7\linewidth}{!}{
    \begin{tabular}{l|ccccc}
    \toprule
    \textbf{Method} & AP ($\uparrow$) & AF ($\uparrow$) & TDG ($\uparrow$) & Time ($\downarrow$) & GPU ($\downarrow$)\\
    \midrule
Base Model                & 53.88 & NA    & NA    & NA     & NA \\
GCAL with VGAE            & 55.65 & 0.42  & 54.78 & 217.51 & 11159.8 \\
Random subgraph sampling  & 51.26 & -1.45 & 51.32 & 67.81  & 800.96 \\
Feature averaging         & 52.91 & -0.85 & 52.19 & 65.56  & 783.13 \\
\rowcolor{myblue}\textbf{\ourmethod} & \textbf{55.95} & \textbf{0.47} & \textbf{55.56} & \textbf{74.64} & \textbf{854.97} \\
\bottomrule
\end{tabular}}
\label{tab:abla-mem}
\end{table*}

\label{ap:more-vis}

\subsection{Memory Construction Alternatives}
To further verify that \ourmethod's efficiency advantage stems from \emph{both} eliminating auxiliary generative training \emph{and} producing informative memories, we compare against two simpler training-free memory schemes that replace GCAL's VGAE: \textit{Random subgraph sampling} and \textit{Feature averaging}. As shown in Table~\ref{tab:abla-mem}, both alternatives achieve comparable cost to \ourmethod but suffer substantial performance drops, with feature averaging falling below the base model. This confirms that the closed-form crystallization in \ourmethod produces memories that are not only cheap but also semantically faithful.

\subsection{Additional Visualizations of Crystallized Memories}
We complement the t-SNE projection in Figure~\ref{fig:tsne-twitch-gcn} of the main paper with structural visualizations of the original graphs and their crystallized memories on the continuous-domain Twitch graphs, rendered with the NetworkX library from Figure~\ref{fig:networkx-twitch-gcn} to Figure~\ref{fig:networkx-end}. To keep individual figures legible, the original graph is rendered after a fixed-rate node sampling, while the crystallized memory is rendered in full. Across the visualized domains, the structural layout of each crystallized memory closely traces that of its original graph, while the t-SNE projection in the main paper shows that the two also remain aligned in the embedding space. Together, the two views support the design intuition that memory crystallization captures the semantically informative content of each domain in a compact, training-free manner.

\begin{figure*}[t]
    \centering
    \includegraphics[width=\linewidth]{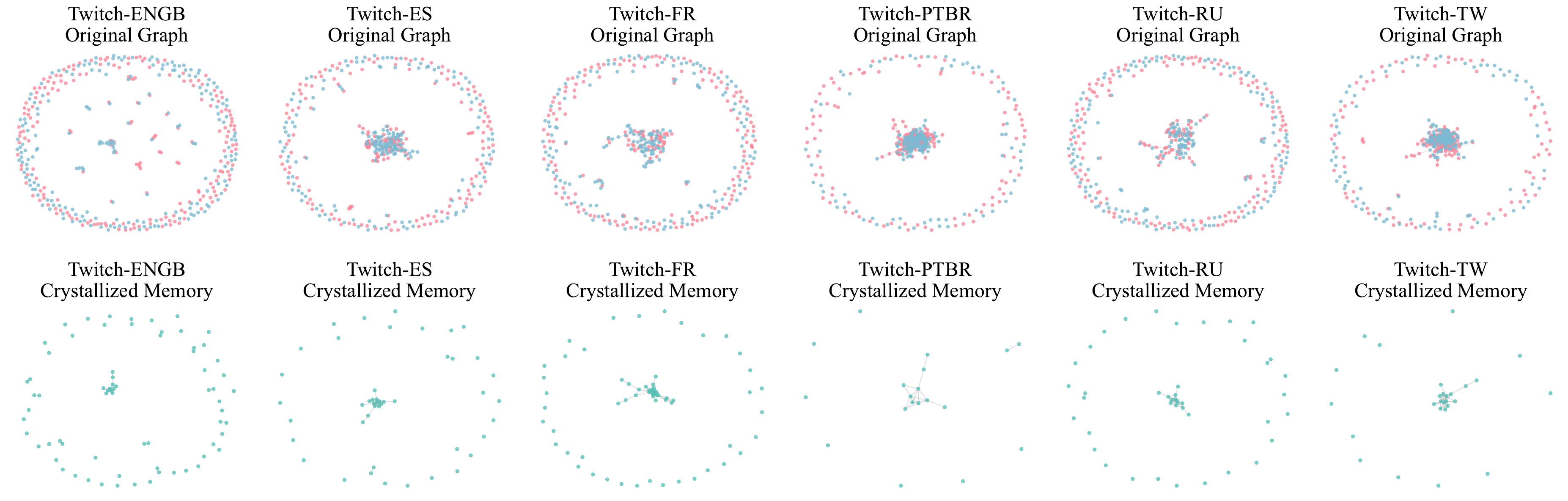}
    \caption{Structural visualization of the original graph and its crystallized memory on continuous-domain Twitch graphs with GCN backbone.}
    \label{fig:networkx-twitch-gcn}
    \vspace{-1em}
\end{figure*}

\begin{figure*}[t]
    \centering
    \includegraphics[width=\linewidth]{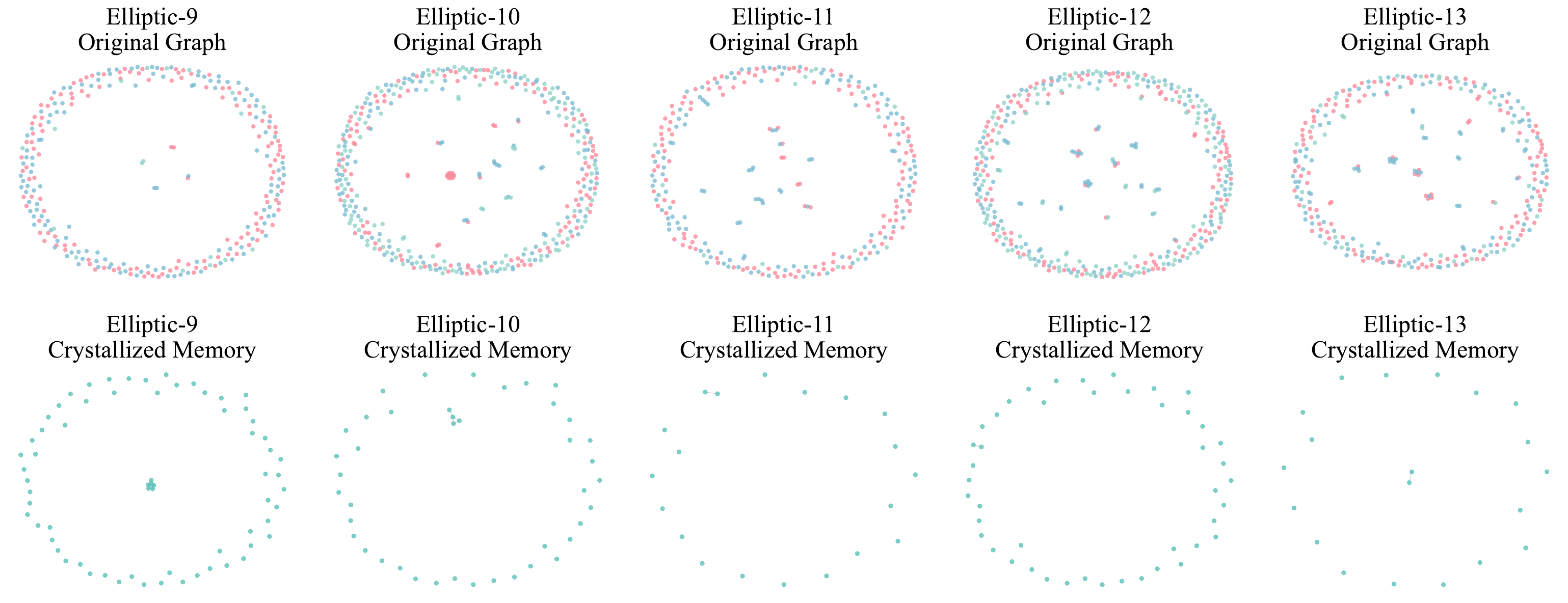}
    \caption{Structural visualization of the original graph and its crystallized memory on continuous-domain Elliptic graphs with GCN backbone (Part I).}
    \label{fig:networkx-elliptic-gcn-1}
    \vspace{-1em}
\end{figure*}

\begin{figure*}[t]
    \centering
    \includegraphics[width=\linewidth]{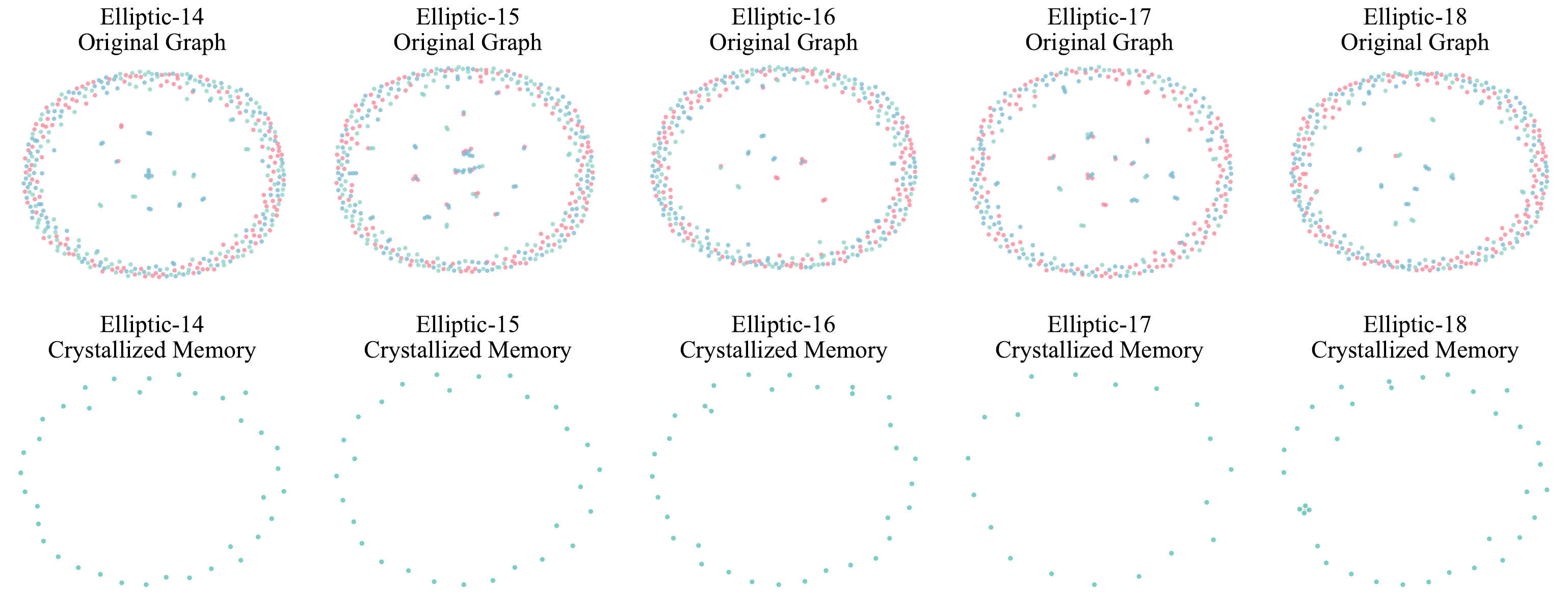}
    \caption{Structural visualization of the original graph and its crystallized memory on continuous-domain Elliptic graphs with GCN backbone (Part II).}
    \label{fig:networkx-elliptic-gcn-2}
    \vspace{-1em}
\end{figure*}

\begin{figure*}[t]
    \centering
    \includegraphics[width=\linewidth]{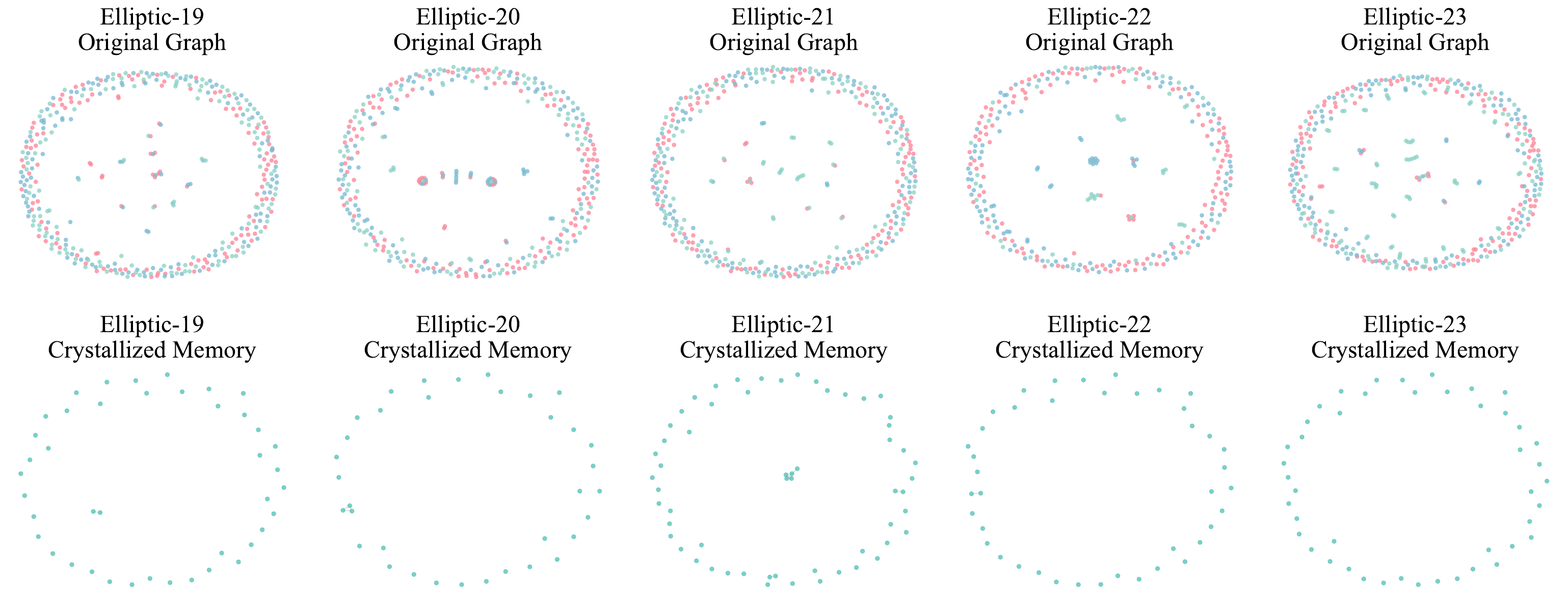}
    \caption{Structural visualization of the original graph and its crystallized memory on continuous-domain Elliptic graphs with GCN backbone (Part III).}
    \label{fig:networkx-elliptic-gcn-3}
    \vspace{-1em}
\end{figure*}

\begin{figure*}[t]
    \centering
    \includegraphics[width=\linewidth]{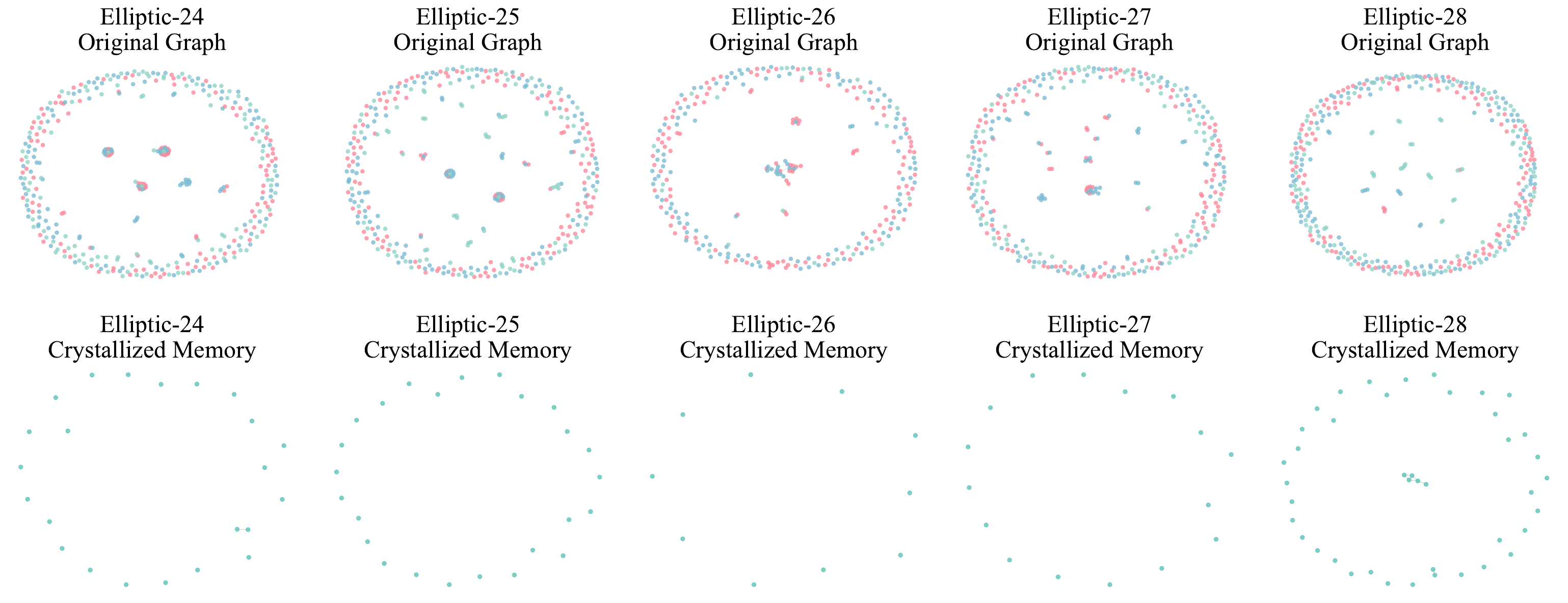}
    \caption{Structural visualization of the original graph and its crystallized memory on continuous-domain Elliptic graphs with GCN backbone (Part IV).}
    \label{fig:networkx-elliptic-gcn-4}
    \vspace{-1em}
\end{figure*}

\begin{figure*}[t]
    \centering
    \includegraphics[width=\linewidth]{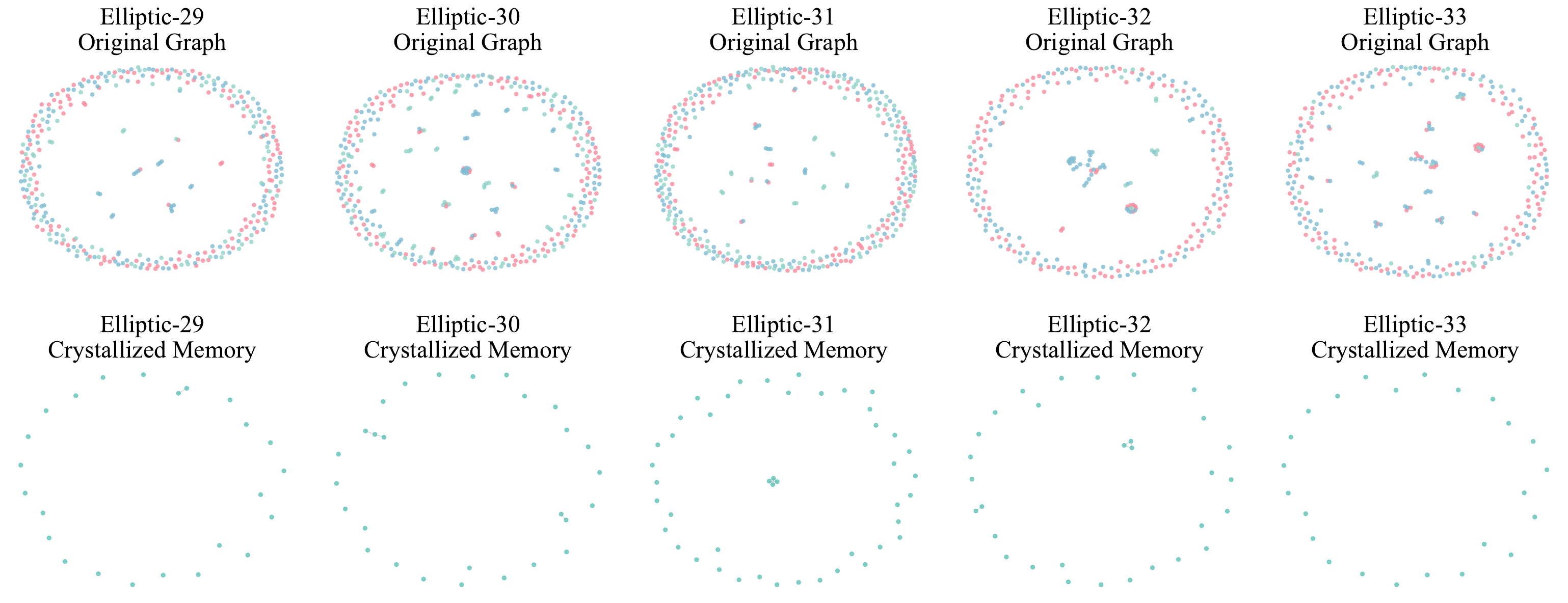}
    \caption{Structural visualization of the original graph and its crystallized memory on continuous-domain Elliptic graphs with GCN backbone (Part V).}
    \label{fig:networkx-elliptic-gcn-5}
    \vspace{-1em}
\end{figure*}

\begin{figure*}[t]
    \centering
    \includegraphics[width=\linewidth]{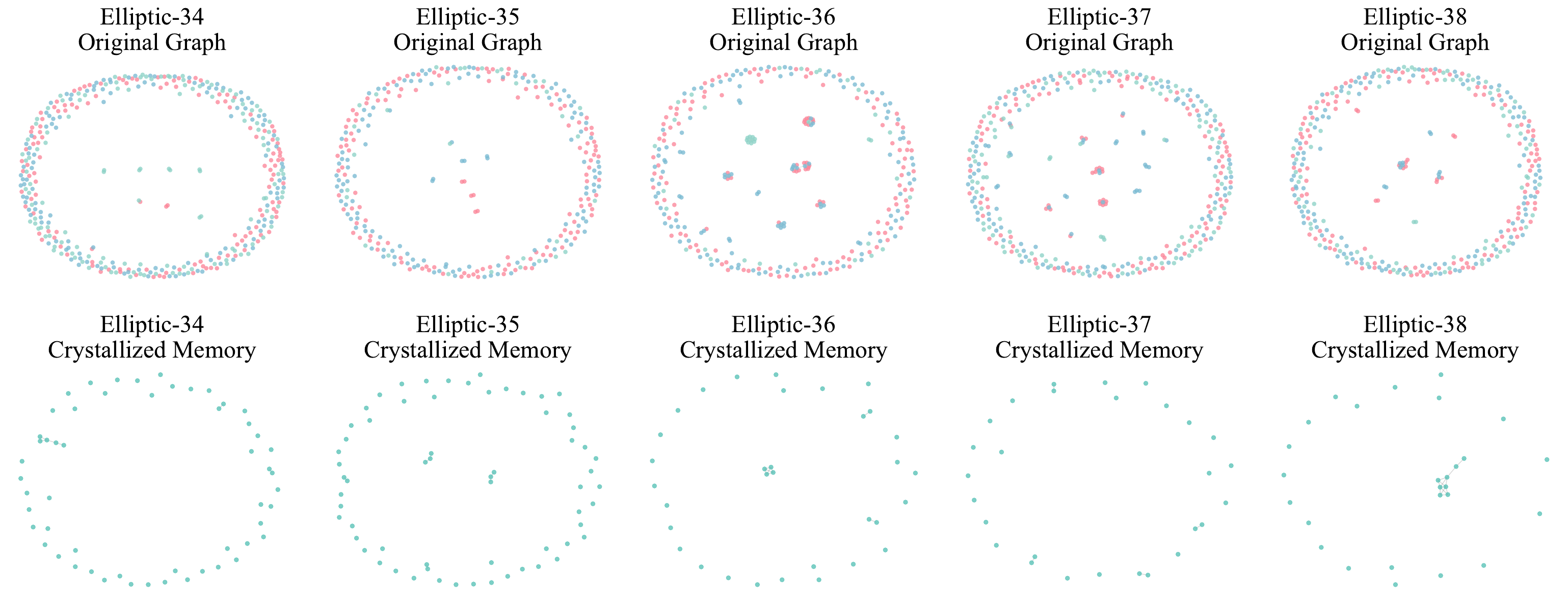}
    \caption{Structural visualization of the original graph and its crystallized memory on continuous-domain Elliptic graphs with GCN backbone (Part VI).}
    \label{fig:networkx-elliptic-gcn-6}
    \vspace{-1em}
\end{figure*}

\begin{figure*}[t]
    \centering
    \includegraphics[width=\linewidth]{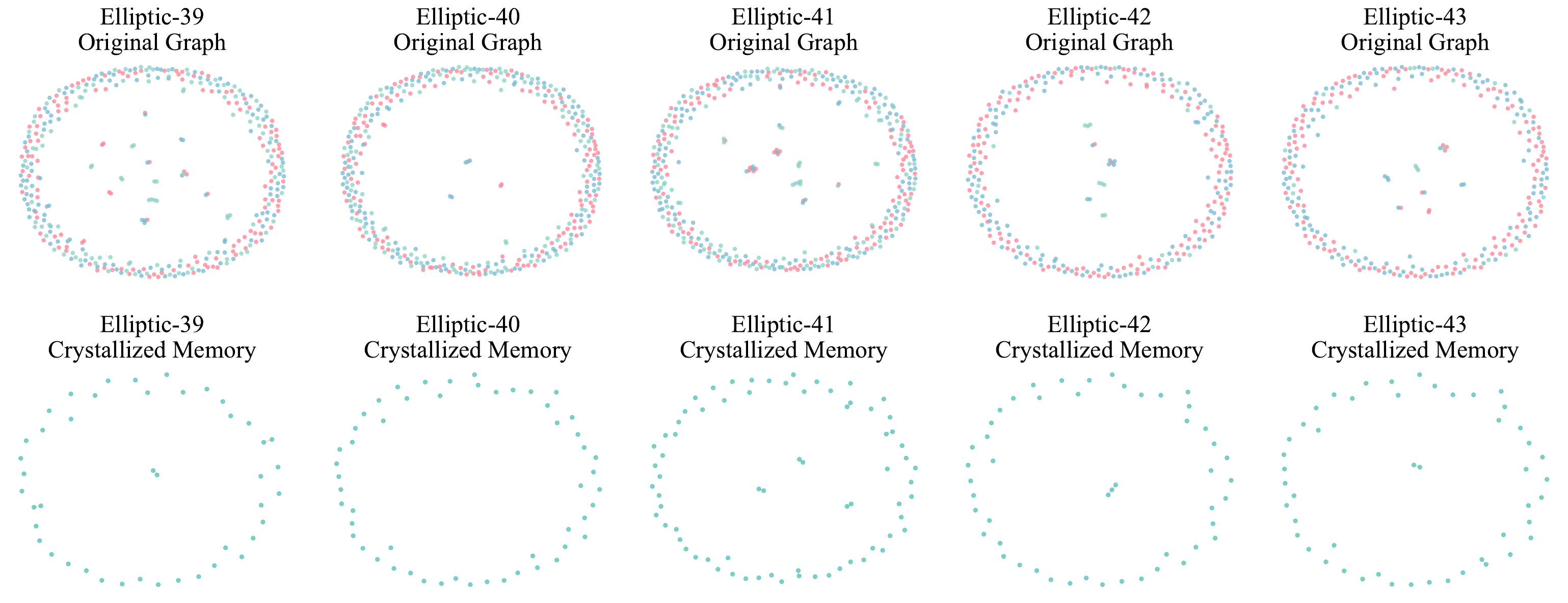}
    \caption{Structural visualization of the original graph and its crystallized memory on continuous-domain Elliptic graphs with GCN backbone (Part VII).}
    \label{fig:networkx-elliptic-gcn-7}
    \vspace{-1em}
\end{figure*}

\begin{figure*}[t]
    \centering
    \includegraphics[width=\linewidth]{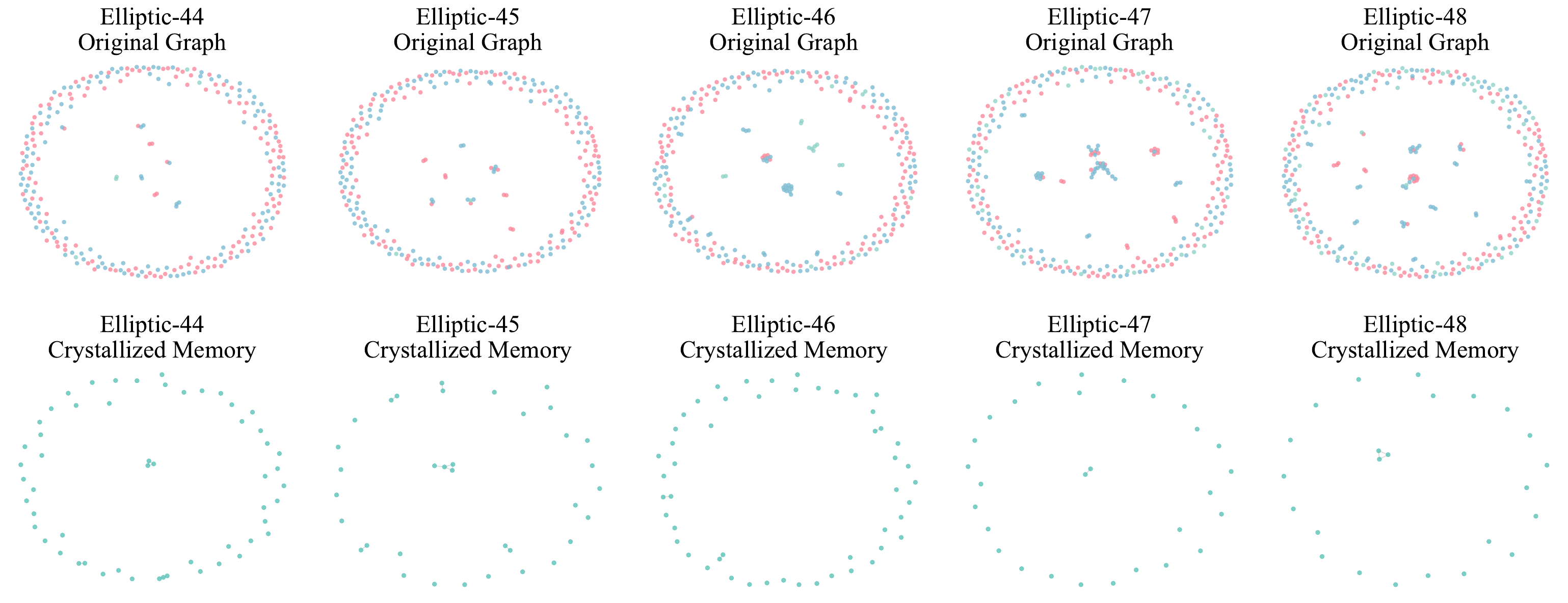}
    \caption{Structural visualization of the original graph and its crystallized memory on continuous-domain Elliptic graphs with GCN backbone (Part VIII).}
    \label{fig:networkx-elliptic-gcn-8}
    \vspace{-1em}
\end{figure*}

\begin{figure*}[t]
    \centering
    \includegraphics[width=\linewidth]{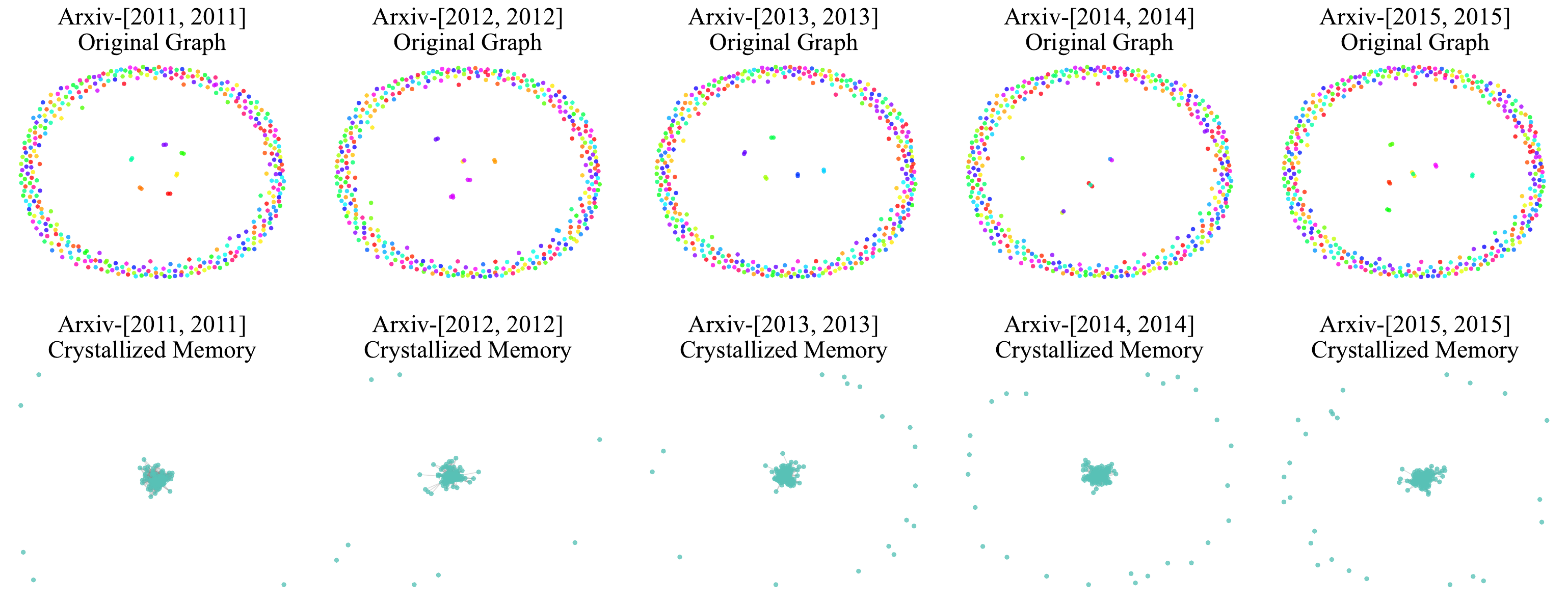}
    \caption{Structural visualization of the original graph and its crystallized memory on continuous-domain OGB-Arxiv graphs with GCN backbone (Part I).}
    \vspace{-1em}
\end{figure*}

\begin{figure*}[t]
    \centering
    \includegraphics[width=\linewidth]{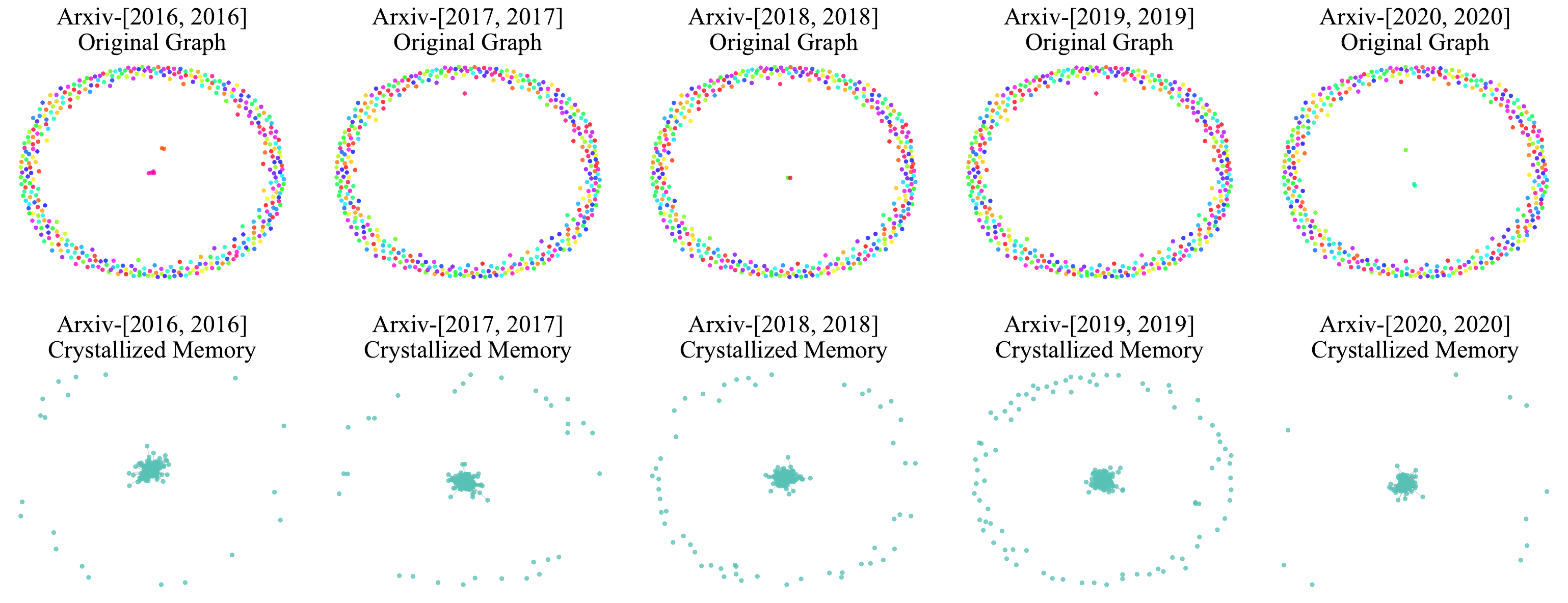}
    \caption{Structural visualization of the original graph and its crystallized memory on continuous-domain OGB-Arxiv graphs with GCN backbone (Part II).}
    \label{fig:networkx-end}
    \vspace{-1em}
\end{figure*}

\input{table/ap-effi}

\subsection{Additional Scaling Efficiency Analysis} \label{ap:scale}
To evaluate the efficiency of \ourmethod, we construct datasets by generating Erd\H{o}s--R\'enyi graphs with varying node and domain scales, based on the size characteristics of existing datasets.
In Figures~\ref{fig:eff2} and~\ref{fig:eff1} of the main paper, we report the runtime and peak GPU memory usage under increasing node and domain scales, where a red cross indicates an Out-Of-Memory (OOM) failure.
We provide detailed experimental settings below.

\textbf{Node Scaling Analysis.}
For the node scaling analysis, we fix the number of target domains to 10 and gradually increase the number of nodes per graph from 1k to 40k.
The node feature dimension is set to 1288, consistent with OGB-Arxiv.
The numerical results are reported in Table~\ref{tab:eff-1}.
This experiment analyzes how the efficiency of \ourmethod changes as each test sample increases in size.
Specifically, when the number of nodes exceeds 40k, GCAL fails due to OOM, while \ourmethod remains within GPU memory limits even at the largest tested scale.

\textbf{Domain Scaling Analysis.}
For the domain scaling analysis, we set the initial number of target domains to 10, with each graph containing 10k nodes and 128-dimensional node features (again aligned with OGB-Arxiv).
We then gradually increase the number of target domains by a scaling factor.
Detailed results are shown in Table~\ref{tab:eff-2}.
Compared to GCAL, which runs out of memory when the number of target domains reaches 50 or more, \ourmethod maintains stable runtime and GPU usage across all scales.

We also benchmark the memory crystallization step alone on a single graph up to \textbf{200k nodes} in Appendix~\ref{ap:scalability}, confirming that the closed-form inversion remains under one second even at this scale and is therefore not the bottleneck of the pipeline.

\textbf{Storage Cost of Crystallized Memories.} We additionally report the storage cost of the crystallized memories as the stream length increases in Table~\ref{tab:reb-1}.
Although the storage cost of crystallized memories grows with the number of domains, it remains small in absolute size. EMC maintains a low GPU memory footprint as the stream grows, requiring only 145.59 MB with 80 domains, while the crystallized memories occupy 39.07 MB. In comparison, GCAL exhibits substantially higher memory growth and runs out of GPU memory when the number of domains reaches 50.

\begin{table*}[t]
\centering
\caption{Scaling efficiency of GPU memory consumption and crystallized memory storage.}
\label{tab:reb-1}
\resizebox{\linewidth}{!}{
\begin{tabular}{l|cccccccc}
\toprule
\textbf{\#Domain} & 10 & 20 & 30 & 40 & 50 & 60 & 70 & 80 \\
\midrule
GCAL GPU (MiB)
& 6892.91 & 14239.83 & 21586.75 & 28933.67
& OOM & OOM & OOM & OOM \\
EMC GPU (MiB)
& 81.02 & 85.92 & 90.80 & 95.71
& 104.52 & 118.21 & 131.90 & 145.59 \\
Crystallized Memory (MiB)
& 4.88 & 9.76 & 14.65 & 19.53
& 24.41 & 29.30 & 34.18 & 39.07 \\
\bottomrule
\end{tabular}
}
\end{table*}

\textbf{Discussion on Memory Budgeting.}
A natural strategy is to cap the buffer at $K$ memories and manage it as in agent-memory mechanisms, for example keeping the most recent ones in a first-in-first-out manner, retaining those most representative by similarity, or maintaining a weighted moving average of past memories. Because crystallization is a closed-form operation on node features, the oldest memories can also be merged and re-crystallized into a single memory at negligible cost, yielding constant storage, which is not available to generative replay without retraining the generator.

\subsection{Additional Results of Parameter Study} \label{ap:sensi}
We provide additional results on hyperparameter sensitivity in Figure~\ref{fig:a} and Figure~\ref{fig:b} to further support our analysis.

\noindent\textbf{Sensitivity Analysis of $\tau$.}
We investigate the impact of $\tau$, the hyperparameter for graph sparsification, by varying it from 0.5 to 1.0.
As shown in Figure~\ref{fig:t}, \ourmethod exhibits robustness to a wide range of $\tau$ values, while proper tuning is still beneficial for optimal performance.

\begin{figure*}[t]
    \begin{minipage}{0.5\textwidth}
        \centering
        \includegraphics[width=\linewidth]{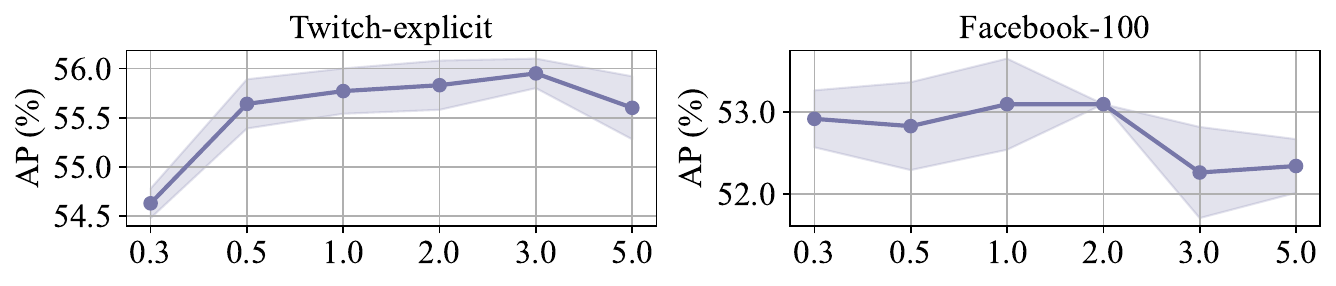}
        \caption{The sensitivity of hyperparameter $\alpha$.}
        \label{fig:a}
    \end{minipage}
    \begin{minipage}{0.5\textwidth}
        \centering
        \includegraphics[width=\linewidth]{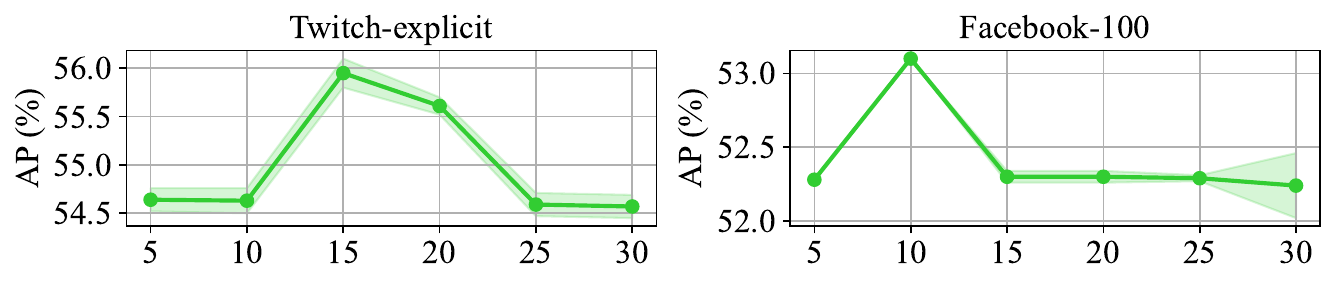}
        \caption{The sensitivity of hyperparameter $\beta$.}
        \label{fig:b}
    \end{minipage}
\end{figure*}

\begin{figure*}[t]
\centering
\includegraphics[width=0.5\linewidth]{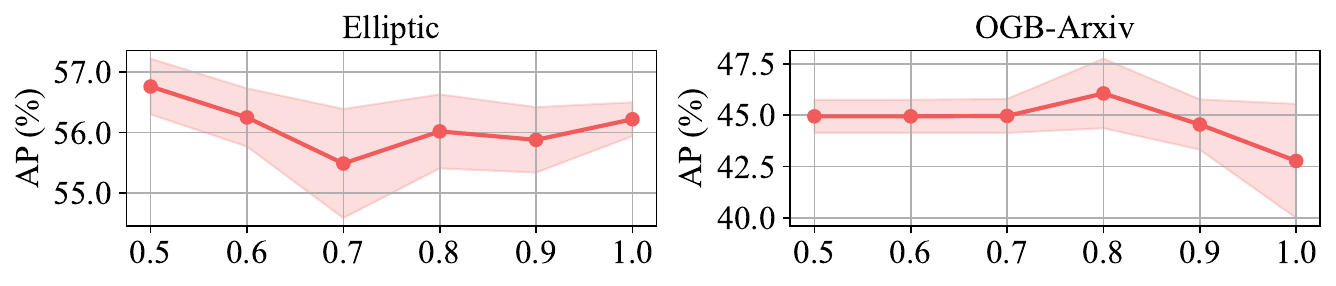}
\hspace{-0.2cm}
\includegraphics[width=0.5\linewidth]{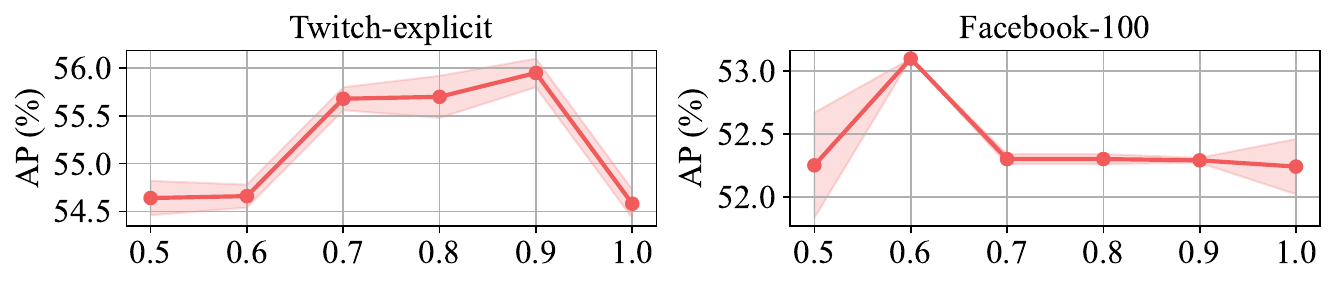}
\caption{The sensitivity of hyperparameter $\tau$.}
\label{fig:t}
\end{figure*}

\input{table/ap-lam}

\noindent\textbf{Analysis of Memory Size.}
To study the influence of the memory size on model performance, we define the relationship between the number of memory nodes $N'_{t}$ and the number of original sample nodes $N_{t}$ via a scaling factor $\lambda$, such that $N_{t} = \lambda^{-1} N'_{t}$.
In the main paper, we set $\lambda = 0.01$. We further analyze the sensitivity of $\lambda$ in Table~\ref{tab:sensi-N}.
The size of the memory has a minor effect on adaptation performance: the model maintains both generalizability and adaptability regardless of the scaling factor between $N'_{t}$ and $N_{t}$, demonstrating the robustness of \ourmethod.
Moreover, different values of $\lambda$ incur only negligible variation in runtime and GPU memory cost, indicating that constructing crystallized memories introduces minimal and stable computational overhead during adaptation.

\section{Broader Impacts}
\label{Impacts}
Distribution shifts arise pervasively in real-world graph applications, including social networks, recommendation systems, citation analysis, and financial transaction monitoring, where the underlying data evolves continually over time. By enabling graph models to adapt under such non-stationary shifts in a training-free and memory-efficient manner, \ourmethod contributes to deploying reliable graph learning systems at lower computational and energy cost, which we view as a positive societal outcome. We do not foresee specific negative societal impacts beyond those general to applied machine learning.

\textbf{Future Directions.}
A broader direction suggested by our results is to distinguish task-relevant structural evolution from redundant variations under continual distribution shifts. \ourmethod already follows this principle implicitly by crystallizing each incoming domain into a compact memory that preserves predictive information while suppressing redundant domain-specific content. Recent studies on graph OOD detection~\cite{hou2025structural,hou2026redundancy,hou2025test,hou2026graph} and anomaly detection~\cite{cai2025sdm,cai2026moegad,zhang2026escaping} provide complementary evidence that explicitly separating essential graph structure from redundancy can improve robustness under distribution shifts. Structural-entropy-based~\cite{hou2025structural,wu2022structural,wu2023sega} and redundancy-aware approaches~\cite{hou2026redundancy} aim to preserve informative structural patterns while filtering irrelevant topology, whereas test-time calibration methods~\cite{hou2025test,hou2026graph} dynamically summarize evolving distributional factors through adaptive dictionaries. These perspectives suggest a unified future direction for continual graph adaptation: first extract the task-relevant structural information of each incoming domain, then measure distribution drift within this essential representation, and finally use the estimated drift to control memory construction and adaptation strength. Such a design could allow graph models to respond more strongly to meaningful structural changes while remaining stable against redundant or transient fluctuations. Extending this framework to open-world streams with emerging semantic classes could further connect continual domain adaptation with novel class discovery \cite{hou2024nc2d}.

\textbf{Limitations.}
Despite its practical advantages, \ourmethod has two main limitations:
\begin{itemize}[leftmargin=1.5em]
    \item \textbf{Fixed hyperparameter configuration.} Our current experiments use a single fixed hyperparameter configuration for each dataset. The hyperparameters are not dynamically adjusted for each incoming domain, which may limit adaptation when the characteristics of distribution shifts vary substantially over time.
    
    \item \textbf{Lack of explicit drift measurement.} \ourmethod adapts sequentially under non-stationary distribution shifts, but does not explicitly estimate the magnitude of domain drift to determine the appropriate adaptation strength. Incorporating explicit drift estimation and drift-aware adaptation is a promising direction for future work.
\end{itemize}

%% file: chapter/add-proof.tex
The following result identifies the precise condition under which marginal alignment is equivalent to joint alignment, motivating the state-evolving design that follows.

\begin{proposition}
\label{ap-prop:4-3}
Given any pair of sequences of samples from the target domain, $\mathcal{T} = \{G^t\}_{t=1}^T$ and the corresponding crystallized memories $\tilde{\mathcal{T}} = \{\tilde{G}^t\}_{t=1}^T$, following the joint distributions $\mathbb{P}_{G^{1:T}}$ and $\mathbb{Q}_{\tilde{G}^{1:T}}$ respectively, with each $G^t$ and $\tilde{G}^t$ following the marginal distributions $\mathbb{P}_{G^t}$ and $\mathbb{Q}_{\tilde{G}^t}$. Then the following relation holds:
\vspace{-1em}
\begin{equation}
\mathcal{D}_{\text{KL}} (\mathbb{P}_{G^{1:T}} \,\|\, \mathbb{Q}_{\tilde{G}^{1:T}})
\equiv \sum_{t=1}^T  \mathcal{D}_{\text{KL}}(\mathbb{P}_{G^t}\,\|\,\mathbb{Q}_{\tilde G^t} )\quad \text{iff}\quad \mathbb{P}_{G^{1:T}} = \prod_{t=1}^T \mathbb{P}_{G^t},\quad \mathbb{Q}_{\tilde{G}^{1:T}} = \prod_{t=1}^T \mathbb{Q}_{\tilde{G}^t},
\end{equation}
where $\mathcal{D}_{\text{KL}}$ denotes the Kullback--Leibler divergence and $\equiv$ indicates that the two expressions are identically equal.
\end{proposition}

\begin{proof}
First, we prove sufficiency: if the snapshots are temporally independent, then the two expressions coincide. By the definition of KL divergence and its chain rule, the divergence between the joint distributions of $\mathcal{T}$ and $\tilde{\mathcal{T}}$ can be expanded as:
\begin{equation}
\begin{aligned}
&\mathcal{D}_{\text{KL}}(\mathbb{P}_{G^{1:T}}\|\,\mathbb{Q}_{\tilde G^{1:T}})\\
=& \sum_{G^{1:T}} \mathbb{P}_{G^{1:T}}\log\frac{\mathbb{P}_{G^{1:T}}}{\mathbb{Q}(\tilde G^{1:T})} \\
=& \sum_{G^{1:T}} \mathbb{P}_{G^{1:T}}\log\frac{\mathbb{P}(G^1)}{\mathbb{Q}(\tilde G^1)}
   + \sum_{G^{1:T}} \mathbb{P}_{G^{1:T}} \,\log\frac{\prod_{t=2}^T \mathbb{P}(G^t\mid G^{1:t-1})}
                                            {\prod_{t=2}^T \mathbb{Q}(\tilde G^t\mid \tilde G^{1:t-1})}\\
=& \sum_{G^1} \mathbb{P}(G^1)\,\log\frac{\mathbb{P}(G^1)}{\mathbb{Q}(\tilde G^1)}
+\sum_{t=2}^T \sum_{G^{1:t-1}} \mathbb{P}(G^{1:t-1})
      \Bigl[\sum_{g^t} \mathbb{P}(G^t\mid G^{1:t-1})
             \log\frac{\mathbb{P}(G^t\mid G^{1:t-1})}{\mathbb{Q}(\tilde G^t\mid \tilde G^{1:t-1})}\Bigr]\\
=& \mathcal{D}_{\text{KL}}(\mathbb{P}_{G^1}\|\,\mathbb{Q}_{\tilde G^1})
   + \sum_{t=2}^T \mathcal{D}_{\text{KL}}(\mathbb{P}_{G^t\mid G^{1:t-1}}\|\,\mathbb{Q}_{\tilde G^t\mid \tilde G^{1:t-1}}).
\end{aligned}
\end{equation}

Under the assumption of temporal independence, \ie, $\mathbb{P}_{G^{1:T}} = \prod_{t=1}^T \mathbb{P}_{G^t}$ and $\mathbb{Q}_{\tilde{G}^{1:T}} = \prod_{t=1}^T \mathbb{Q}_{\tilde{G}^t}$, each conditional distribution $\mathbb{P}_{G^t \mid G^{1:t-1}}$ reduces to the marginal $\mathbb{P}_{G^t}$, and likewise for $\mathbb{Q}$. As a result, the sum of conditional KL terms collapses to $\sum_{t=2}^T \mathcal{D}_{\text{KL}}(\mathbb{P}_{G^t}\,\|\,\mathbb{Q}_{\tilde{G}^t})$. Hence the joint KL over the graph--memory divergence collapses to:
\begin{equation}
\begin{aligned}
\mathcal{D}_{\text{KL}}(\mathbb{P}_{G^{1:T}}\|\,\mathbb{Q}_{\tilde G^{1:T}})
= \sum_{t=1}^T \mathcal{D}_{\text{KL}}(\mathbb{P}_{G^t}\|\,\mathbb{Q}_{\tilde G^t}).
\end{aligned}
\end{equation}

Conversely, if the equality between the joint KL and the sum of marginal KLs holds, then by the chain-rule expansion above, every conditional KL term $\mathcal{D}_{\text{KL}}(\mathbb{P}_{G^t\mid G^{1:t-1}}\|\,\mathbb{Q}_{\tilde G^t\mid \tilde G^{1:t-1}})$ must equal $\mathcal{D}_{\text{KL}}(\mathbb{P}_{G^t}\|\,\mathbb{Q}_{\tilde G^t})$ for all $t$. By the strict convexity of KL divergence, this forces $\mathbb{P}_{G^t\mid G^{1:t-1}}=\mathbb{P}_{G^t}$ and $\mathbb{Q}_{\tilde G^t\mid \tilde G^{1:t-1}}=\mathbb{Q}_{\tilde G^t}$ almost surely, i.e., temporal independence.
\end{proof}

\noindent \textbf{Remark.}
Proposition~\ref{ap-prop:4-3} is a \emph{biconditional}: it characterizes precisely when summing per-domain KL divergences equals the joint KL divergence, namely when and only when the domains (and their crystallized memories) are independent. Aligning the marginal distributions of individual graphs $G^{(t)}$ is therefore sufficient to ensure joint distribution alignment over the given sequence of test samples \emph{if and only if} the memories are independent.
Since real test-time domains are generally correlated under continual drift, this implies that aligning marginal distributions alone is insufficient, motivating us to model the relationships among continuously arriving test samples under covariate shift through the state evolving field over crystallized memories.

%% file: table/ap-dataset-info-tab.tex
\begin{table*}[t]
\caption{The statistics of datasets with distribution shifts.} 

\centering
\resizebox{0.75\textwidth}{!}{ 
\begin{tabular}{c|cccc}
\toprule
\textbf{Category} & \textbf{Datasets}  & \textbf{\#Nodes}  & \textbf{\#Edges}  & \textbf{\#Domains} \\
\midrule
\multirow{2}{*}{\makecell[c]{Regional\\Shifts}}  
& Twitch-explicit    & 1,912 - 9,498  & 31,299 - 153,138  &  7\\
& Facebook-100       & 769 - 41,554   & 33,312 - 2,724,458  & 12\\
\midrule

\multirow{2}{*}{\makecell[c]{Temporal\\Shifts\\}}  
& Elliptic           & 1,089 - 7,880       & 1,168 - 9,164          & 41\\ 
& OGB-Arxiv          & 4,427 - 39,711        & 1,225 - 38,735      & 11\\
\bottomrule
\end{tabular}
}
\label{table_dataset_intro}
\end{table*}

%% file: table/ap-effi.tex
\begin{table*}[t]
\caption{Node scaling efficiency evaluation with runtime (s) and GPU memory usage (MiB).}
  \centering
  \resizebox{\textwidth}{!}{
    \begin{tabular}{c|c|ccccccccc}
    \toprule
    \textbf{Methods} & \textbf{Metric / \#Node} & 1000 & 2000 & 5000 & 10000 & 15000 & 20000 & 25000 & 30000 & 40000 \\
    \midrule
    \multirow{2}[0]{*}{GCAL} & Runtime (s) & 25.41 & 33   & 97.77 & 391.87 & 918.07 & 1694.24 & OOM  & OOM  & OOM \\
         & GPU Memory (MiB)  & 91.11 & 304.42 & 1754.81 & 6892.91 & 15472.69 & 27462.27 & OOM  & OOM  & OOM \\
    \midrule
    \multirow{2}[0]{*}{\textbf{\ourmethod}} & Runtime (s) & \textbf{1.75} & \textbf{2.72} & \textbf{25.89} & \textbf{210.52} & \textbf{407.95} & \textbf{593.20} & \textbf{792.00} & \textbf{1207.47} & \textbf{2213.20} \\
         & GPU Memory (MiB)  & \textbf{20.72} & \textbf{25.06} & \textbf{42.2} & \textbf{81.02} & \textbf{131.6} & \textbf{200.07} & \textbf{ 278.16} & \textbf{366.92} & \textbf{597.48} \\
    \bottomrule
    \end{tabular}%
    }
  \label{tab:eff-1}%
\end{table*}%

\begin{table*}[t]
\caption{Domain scaling efficiency evaluation with runtime (s) and GPU memory usage (MiB).}
  \centering
  \resizebox{\textwidth}{!}{
    \begin{tabular}{c|c|cccccccc}
    \toprule
    \textbf{Methods} & \textbf{Metric / \#Domain} & 1$\times$  & 2$\times$  & 3$\times$  & 4$\times$  & 5$\times$  & 6$\times$  & 7$\times$  & 8$\times$\\
    \midrule
    \multirow{2}[0]{*}{GCAL} & Runtime (s) & 430.04 & 1056.97 & 1713.34 & 2460.46 & OOM  & OOM  & OOM  & OOM \\
         & GPU Memory (MiB)  & 6892.91 & 14239.83 & 21586.75 & 28933.67 & OOM  & OOM  & OOM  & OOM \\
    \midrule
    \multirow{2}[0]{*}{\textbf{\ourmethod}} & Runtime (s) & \textbf{215.35} & \textbf{431.27} & \textbf{646.81} & \textbf{863.62} & \textbf{1079.35} & \textbf{1296.11} & \textbf{1513.47} & \textbf{1730.93} \\
         & GPU Memory (MiB)  & \textbf{81.02} & \textbf{85.92} & \textbf{90.80} & \textbf{95.71} & \textbf{104.52} & \textbf{118.21} & \textbf{131.90} & \textbf{145.59} \\
    \bottomrule
    \end{tabular}%
    }
  \label{tab:eff-2}%

\end{table*}%

%% file: table/ap-lam.tex
\begin{table*}[!t]
\caption{Performance comparison of \ourmethod under different values of $\lambda$ across multiple metrics, with corresponding runtime (s) and GPU memory usage (MiB). `$\uparrow$' indicates higher is better, and `$\downarrow$' indicates lower is better.}

\centering

\begin{minipage}[t]{0.50\textwidth}
\centering

\resizebox{\linewidth}{!}{
\begin{tabular}{c|c|ccc|cc}
\toprule
\textbf{Dataset} & {$\lambda$} & AP ($\uparrow$) & TDG ($\uparrow$) & TDA ($\uparrow$) & Time ($\downarrow$) & GPU ($\downarrow$) \\
\midrule

\multirow{6}{*}{\textbf{Twitch}} 
    & 0.01 & 55.95±0.15 & 55.56±0.13 & 55.83±0.14 & 74.64 & 854.97 \\
         & 0.05 & 54.57±0.15 & 55.46±0.15 & 53.90±0.12 & 74.92 & 854.98 \\
         & 0.1  & 54.57±0.15 & 55.46±0.15 & 53.89±0.13 & 74.90 & 854.97 \\
         & 0.2  & 54.57±0.16 & 55.46±0.15 & 53.90±0.12 & 74.84 & 854.97 \\
         & 0.3  & 54.56±0.16 & 55.46±0.15 & 53.88±0.12 & 75.09 & 854.11 \\
         & 0.5  & 54.59±0.16 & 55.47±0.15 & 53.88±0.12 & 74.67 & 854.97 \\
\midrule
\multirow{6}{*}{\textbf{Facebook}} 
    & 0.01 & 53.10±0.00 & 53.49±0.27 & 52.30±0.04 & 18.84 & 3006.08 \\
         & 0.05 & 53.10±0.00 & 53.45±0.35 & 52.29±0.02 & 18.84 & 3006.08 \\
         & 0.1  & 53.10±0.00 & 53.45±0.35 & 52.30±0.03 & 18.92 & 3006.08 \\
         & 0.2  & 53.10±0.00 & 53.46±0.33 & 52.30±0.04 & 18.92 & 3006.08 \\
         & 0.3  & 53.10±0.00 & 53.36±0.53 & 52.30±0.03 & 18.98 & 3006.08 \\
         & 0.5  & 53.10±0.00 & 53.50±0.24 & 52.30±0.04 & 18.93 & 3006.08 \\

\bottomrule
\end{tabular}
}
\end{minipage}
\hfill
\begin{minipage}[t]{0.48\textwidth}
\centering
\resizebox{\linewidth}{!}{
\begin{tabular}{c|c|ccc|cc}
\toprule
\textbf{Dataset} & {$\lambda$} & AP ($\uparrow$) & TDG ($\uparrow$) & TDA ($\uparrow$) & Time ($\downarrow$) & GPU ($\downarrow$) \\
\midrule

\multirow{6}{*}{\textbf{Elliptic}} 
& 0.01 & 56.76±0.46 & 58.73±0.40 & 55.92±0.31 & 22.76 & 189.36 \\
         & 0.05 & 56.18±0.25 & 58.44±0.24 & 55.96±0.28 & 23.38 & 189.36 \\
         & 0.1  & 56.20±0.24 & 58.48±0.19 & 55.96±0.27 & 23.52 & 189.36 \\
         & 0.2  & 56.23±0.24 & 58.53±0.19 & 55.98±0.23 & 23.19 & 189.36 \\
         & 0.3  & 56.22±0.21 & 58.56±0.20 & 56.02±0.21 & 23.18 & 189.36 \\
         & 0.5  & 56.25±0.24 & 56.03±0.19 & 58.58±0.23 & 22.98 & 189.36 \\
\midrule

\multirow{6}{*}{\textbf{Arxiv}} 
& 0.01 & 46.07±1.70 & 47.11±1.80 & 43.52±1.25 & 24.84 & 391.02 \\
         & 0.05 & 46.06±1.54 & 47.38±1.54 & 47.38±1.54 & 25.54 & 391.05 \\
         & 0.1  & 46.00±1.59 & 47.29±1.62 & 47.29±1.62 & 25.73 & 390.98 \\
         & 0.2  & 46.02±1.63 & 47.20±1.70 & 43.50±1.23 & 25.62 & 391.05 \\
         & 0.3  & 43.91±1.65 & 47.16±1.73 & 43.50±1.23 & 26.12 & 391.05 \\
         & 0.5  & 43.85±1.70 & 47.11±1.79 & 43.48±1.26 & 25.70 & 391.05 \\

\bottomrule
\end{tabular}
}
\end{minipage}

\label{tab:sensi-N}
\end{table*}